\documentclass{article}

\usepackage[main, final]{neurips_2026}

\usepackage[utf8]{inputenc} 
\usepackage[T1]{fontenc}    
\usepackage{hyperref}       
\usepackage{url}            
\usepackage{booktabs}       
\usepackage{amsfonts}       
\usepackage{nicefrac}       
\usepackage{microtype}      
\usepackage[table]{xcolor}  

\usepackage{amsthm}
\usepackage{amssymb}
\usepackage{amsmath}
\usepackage[english]{babel}
\usepackage{graphicx}
\usepackage{booktabs} 
\usepackage{dsfont}
\usepackage{multirow}
\usepackage{multicol}

\usepackage{tikz}
\usetikzlibrary{arrows.meta, decorations.pathmorphing, calc, patterns, positioning}
\usepackage{subcaption}
\usepackage{bm}
\usepackage{graphicx}
\graphicspath{{../figures/}}

\newcommand{\1}{\mathds{1}}

\usepackage{comment}

\usepackage{cancel}

\RequirePackage{fancyhdr}
\RequirePackage{algorithm}
\RequirePackage{algorithmic}

\usepackage{array} 

\newtheorem{theorem}{Theorem}
\newtheorem{lemma}{Lemma}
\newtheorem{prop}{Proposition}

\newtheorem{assumption}{Assumption}

\newcommand{\y}{\boldsymbol{y}}
\newcommand{\K}{\boldsymbol{K}}
\newcommand{\p}{\boldsymbol{p}}

\newcommand{\z}{\boldsymbol{z}}

\let\r\undefined
\newcommand{\r}{\boldsymbol{r}}
\let\v\undefined
\newcommand{\v}{\boldsymbol{v}}
\let\k\undefined
\newcommand{\k}{\boldsymbol{k}}
\let\u\undefined
\newcommand{\u}{\boldsymbol{u}}

\newcommand{\X}{\boldsymbol{X}}

\newcommand{\W}{\boldsymbol{W}}
\newcommand{\A}{\boldsymbol{A}}

\newcommand{\M}{\boldsymbol{M}}

\newcommand{\V}{\boldsymbol{V}}

\newcommand{\q}{\boldsymbol{q}}
\let\a\undefined
\newcommand{\a}{\boldsymbol{a}}

\newcommand{\R}{\mathbb{R}}
\newcommand{\J}{\mathcal{J}}
\let\L\undefined
\newcommand{\L}{\mathcal{L}}
\newcommand{\Q}{\mathcal{Q}}

\newcommand{\inner}[1]{\left\langle #1\right\rangle}

\allowdisplaybreaks

\newcommand{\acc}{\mathrm{acc}}

\newcommand{\logits}{\mathrm{logits}}
\newcommand{\BLOCKN}{B_n}

\newcommand{\BLOCKK}{B_k}

\newtheorem{innercustomthm}{Theorem}
\newenvironment{customthm}[1]
  {\renewcommand\theinnercustomthm{#1}\innercustomthm}
  {\endinnercustomthm}

\newtheorem{innercustomlemma}{Lemma}
\newenvironment{customlemma}[1]
  {\renewcommand\theinnercustomlemma{#1}\innercustomlemma}
  {\endinnercustomlemma}

\title{DropKV: Decoupling Residual-Output Perturbation for Near-Optimal KV-Cache Eviction}

\author{%
  Aozhong Zhang$^1$ \; Selcuk Gurses$^1$ \; Yanxia Deng$^3$ \; Naigang Wang$^2$ \; Chi-Chun Liu$^2$ \\
  \bfseries Davis Wertheimer$^2$ \; Derrick Liu$^2$ \; Xin Li$^1$ \; Zi Yang$^1$ \; {Felix X.-F. Ye$^1$\thanks{Corresponding authors. The code is also available at \href{https://github.com/NVIDIA/kvpress/blob/main/kvpress/presses/dropkv_press.py}{NVIDIA/kvpress}}} \; {Penghang Yin$^1$\footnotemark[1]}\\[0.5em]
  \normalfont $^1$University at Albany, SUNY \quad $^2$IBM T. J. Watson Research Center \quad $^3$Northwestern University\\
  \texttt{\{azhang3, sgurses, xli48, zyang8, xye2, pyin\}@albany.edu}\\
  \texttt{\{nwang, cliu\}@us.ibm.com} \quad \texttt{\{Davis.Wertheimer, dliu\}@ibm.com}\\
  \texttt{yanxia.deng@kellogg.northwestern.edu}\\
  \texttt{https://github.com/AozhongZhang/DropKV}\\
}

\begin{document}

\maketitle

\begin{abstract}
Key-value (KV) cache management has become a critical bottleneck for scaling Large Language Models (LLMs) to long-context scenarios due to linear memory growth. While recent eviction methods have moved beyond simple attention-score heuristics to incorporate value representations, they typically face a fundamental trade-off: the underlying optimization for minimizing output perturbation is an NP-hard combinatorial problem, leading to mathematically suboptimal heuristics or implementations that are difficult to kernelize. In this work, we propose DropKV, a simple principled eviction framework based on \textbf{D}ecoupled \textbf{R}esidual-\textbf{O}utput \textbf{P}erturbation. By decoupling the joint eviction decision into independent per-token scoring, DropKV circumvents combinatorial intractability and admits a provable constant-factor approximation guarantee under a cone condition that we empirically verify on three long-context LLMs. To ensure practical viability, we provide a high-performance implementation using fused Triton kernels that avoid materializing the full attention matrix, delivering a 19.8$\times$ scoring-kernel speedup that translates to up to a 9.9\% matched-batch end-to-end prefill speedup. Extensive evaluations on the RULER and LongBench benchmarks demonstrate that DropKV consistently outperforms state-of-the-art baselines at equivalent cache budgets, achieving the lowest inference latency among eviction methods while also remaining faster than the dynamic cache baseline. 
\end{abstract}

\section{Introduction}
\label{sec:intro}

Large Language Models (LLMs)~\cite{achiam2023gpt,grattafiori2024llama,jiang2023mistral7b,gemmateam2025gemma3technicalreport,qwen2025qwen25technicalreport} are now deployed across code, scientific, and conversational tasks. As these models are increasingly used on long-form content such as entire codebases or lengthy transcripts, the underlying hardware efficiency becomes paramount. Autoregressive inference in transformer-based language models \cite{vaswani2017attention} relies heavily on key-value (KV) caching, which stores previously computed key and value projections to avoid redundant attention computation during the decoding phase. 
This mechanism effectively trades increased memory usage for significantly faster inference, making it a standard optimization in modern LLM inference systems. However, the KV cache memory requirement grows linearly with the  sequence length and can become a prohibitive burden in long-context scenarios. For example, caching 128K tokens in a Phi-3-mini~\cite{abdin2024phi3technicalreporthighly} model requires approximately 48 GB of memory in \texttt{FP16} precision for the KV cache alone, which dwarfs the roughly 7 GB needed to store the whole model weights. This rapid expansion of the memory footprint leads to substantial GPU memory saturation and increased memory I/O overhead~\cite{zhang2023h2o}, making KV cache a critical bottleneck for scalable long-context inference.

To address this challenge, a growing body of work has been proposed to reduce the KV cache memory footprint, including quantization~\cite{hooper2024kvquant,liu2024kivi}, low-rank compression~\cite{chang2024palu,sun2024shadowkv}, and eviction~\cite{cai2024pyramidkv,li2024snapkv,park2025keydiff,tang2024quest,zhang2023h2o}. Among these methods, cache eviction has emerged as a practical solution that selectively retains only a small subset of tokens during inference. Existing eviction methods roughly fall into two categories. Query-agnostic methods evaluate importance based on intrinsic properties of the cache, independent of the current decoding state, such as token position~\cite{xiao2023efficient}, the $\ell_2$-norms~\cite{devoto2024simple} or pairwise similarities~\cite{park2025keydiff} of the keys. 
These methods overlook the dynamic interaction between queries and cached keys. In contrast, query-aware methods leverage attention signals to identify critical KV pairs~\cite{li2024snapkv,zhang2023h2o}. While more adaptive, existing methods typically equate importance solely with attention weights. This perspective is incomplete, as it ignores the critical role of the associated values; even a token with a moderate attention score can significantly shift the final representation if its corresponding value vector has a large magnitude or unique direction. 

More recent methods have sought to transcend attention-score-only criteria by incorporating value representations into their selection logic~\cite{feng2025identify, geng2025accurate} in a more principled manner. Specifically, they attempt to quantify the perturbation induced by cache eviction to the final attention output, in an effort to identify the optimal eviction set that minimizes the deviation from the original output. However, this underlying optimization is inherently an NP-hard combinatorial problem, forcing existing research to settle for mathematically suboptimal heuristics. Furthermore, as policy designs become increasingly intricate, it remains unclear whether these approaches can be efficiently kernelized, limiting their viability in production inference.

\paragraph{Contributions.} 
In this work, we propose DropKV, a mathematically grounded and efficiently implemented KV-cache eviction framework. Specifically, our main contributions are as follows:
\begin{itemize}
    \item \textbf{Decoupled Perturbation Framework:} We introduce a novel scoring mechanism that approximates the optimal eviction set by decoupling the NP-hard combinatorial search into a tractable per-token perturbation minimization with linear computational complexity.
    
    \item \textbf{Theoretical Near-Optimality Guarantees:} We prove that our decoupled proxy matches the optimal solution of the NP-hard problem up to a constant multiplicative factor under the cone condition, which we empirically verify on three long-context LLMs: Llama-3.1-8B, Mistral-7B-512K, and Qwen2.5-7B-1M.

    \item \textbf{High-Performance Triton Implementation:} We develop a suite of fused Triton kernels that compute eviction scores on-the-fly. By avoiding the materialization of the attention matrix, our implementation reduces the scoring-kernel scratch from ${\sim}9$\,GB to ${\sim}17$\,MB at $n{=}128\mathrm{K}$, removing the prefill memory bottleneck and yielding end-to-end matched-batch prefill speedups.

    \item \textbf{State-of-the-Art Performance:} On RULER and LongBench at 5\% and 10\% cache budgets, DropKV achieves the highest accuracy and the lowest latency among query-aware and output-aware eviction baselines, while also remaining faster than the dynamic cache baseline.
\end{itemize}



\section{Background and Related Work}
\label{sec:background}
\subsection{Causal Attention and KV Cache}
Transformer-based language models \cite{vaswani2017attention} rely on the self-attention mechanism to model long-range interactions among tokens in a sequence. Given an input matrix of hidden representations $  \X \in \mathbb{R}^{n \times d}  $, the query, key, and value matrices are obtained via linear projections:
$
\boldsymbol{Q} = \X \W_Q \in \mathbb{R}^{n \times d}, \quad 
\K = \X \W_K \in \mathbb{R}^{n \times d}, \quad 
\V = \X \W_V \in \mathbb{R}^{n \times d},
$
where $\W_Q, \W_K, \W_V \in \mathbb{R}^{d \times d}$ are learnable projection matrices. 
The attention output is then computed as
\begin{equation}\label{eq:attn}
    \boldsymbol{A} := \mathrm{softmax}\left( \frac{\boldsymbol{Q}\boldsymbol{K}^\top + \M}{\sqrt{d}} \right) \boldsymbol{V}\in\R^{n\times d},
\end{equation}
where $\M\in\R^{n\times n}$ is a causal attention mask that prevents attending to future tokens. It is defined as $\M_{ij} = 0$ if $i\geq j$, and $-\infty$ otherwise. During autoregressive inference, the computation of self-attention proceeds in two phases: the prefill stage and the decoding stage.


In the \emph{prefill} stage, the model processes the $n$-token prompt in a single forward pass via~\eqref{eq:attn} and stores the resulting $\K,\V\in\R^{n\times d}$ in the KV cache. It computes logits at its final position to select token \(n+1\). In the \emph{decoding} stage, at step $t{+}1$ ($t > n$), the latest token's key $\k_{t}$ and value $\v_{t}$ are computed and appended to the cached $\K_{1:t-1},\V_{1:t-1}$ to form $\K_{1:t},\V_{1:t}$, and its attention output is computed from its query $\q_{t}$ as
\[
\a_{t}^\top
=
\mathrm{softmax}\!\left(
\q_{t}^\top\K_{1:t}^\top/\sqrt{d}
\right)\V_{1:t}.
\]
The model outputs vocabulary logits and uses them to select token $t + 1$.

\subsection{KV-Cache Eviction}

The challenge of efficient long-context inference has led to a variety of strategies for KV cache eviction. Early approaches relied on simple sliding windows \cite{beltagy2020longformer} to maintain recent context, though these often risked losing critical historical information. To address this, StreamingLLM~\cite{xiao2023efficient} introduced ``attention sinks'', demonstrating that preserving a small set of initial tokens is helpful for maintaining the functionality of the attention mechanism. KNorm~\cite{devoto2024simple} retains tokens with smaller key norms, while KeyDiff~\cite{park2025keydiff} prefers tokens whose keys are more diverse in the representation space. KVzip~\cite{kim2025kvzip} evaluates KV cache importance by prompting the model to reconstruct chunked contexts and aggregating the maximum attention weights received by the context keys. Methods like H2O~\cite{zhang2023h2o} and Scissorhands~\cite{liu2023scissorhands} moved toward importance-based selection by utilizing accumulated attention scores to identify ``heavy hitter'' entries. This line of research culminated in state-of-the-art frameworks like SnapKV~\cite{li2024snapkv}, which employs observation window-based accumulation and pooling operations to consolidate attention weights. More recently, the focus has expanded from token-level scoring to architectural optimization; for instance, AdaKV~\cite{feng2024ada} and related budget-allocation methods~\cite{cai2024pyramidkv,qin2025cake,xiao2024duoattention,yang2024pyramidinfer} dynamically redistribute cache capacity across different attention heads to better capture their unique characteristics. Despite these advancements, existing strategies remain largely empirical, relying almost exclusively on attention weights as a proxy for token importance.

Newer inquiries have begun to bridge this gap by incorporating value representations into the eviction logic~\cite{feng2025identify, geng2025accurate}. These methods can be viewed as \emph{output-aware} eviction strategies that attempt to indirectly minimize the attention-output perturbation induced by token removal. For example, CriticalKV~\cite{feng2025identify} seeks to minimize a loose upper-bound on the attention-output perturbation, while AnDPro~\cite{geng2025accurate} formulates the combinatorial perturbation minimization problem through a convex relaxation. DropKV instead scores tokens using the exact per-token output perturbation and further admits a constant-factor approximation guarantee for the original joint multi-token eviction problem under a cone condition (\S\ref{sec:theory}). See Appendix~\ref{app:prior-objectives} for detailed comparisons with existing methods.

A separate line of work targets the same memory/bandwidth bottleneck without permanently removing tokens: dynamic sparse attention keeps the full KV cache and attends to only a subset per step (Quest~\cite{tang2024quest}, SparQ~\cite{ribar2023sparq}, Loki~\cite{singhania2024loki}), while SVD-based compression and offload methods retain all token positions but compress or stream the cache between GPU and CPU (ShadowKV~\cite{sun2024shadowkv}, xKV~\cite{chang2025xkv}, InfiniGen~\cite{lee2024infinigen}). These approaches are orthogonal to the eviction track, in which DropKV operates: tokens are permanently removed and the resident cache footprint shrinks.

\section{Method}
\subsection{Attention-Output Perturbation}
Recall that in an autoregressive transformer with $n$ input tokens, the attention output $\a$ for a query vector $\q \in\R^d$ is computed as
$$
\p = \mathrm{softmax}\big( \K \q/\sqrt{d}\big)\in\R^n, \quad \a =  \V^\top \p = \sum_{i=1}^n p_i \v_i \in\R^d,
$$
where $\K, \V \in\R^{n\times d}$ are the accumulated key and value projection matrices in the KV cache, $\v_i^\top$ is the $i$-th row vector of $\V$, and $p_i\in(0,1)$ is the $i$-th attention weight after softmax, satisfying $\sum_{i=1}^n p_i = 1$. The following result quantifies the perturbation to the attention output, $\Delta(\mathcal{J})$, associated with query $\q$ induced by an index set $\J$ of evicted tokens. That is, the perturbation depends not only on the attention scores but also on the value residuals of the form $\a-\v_i$.

\begin{lemma}[Perturbation]\label{thm:exact}
    Suppose a set of tokens indexed by $\mathcal{J} \subset [n]$ is evicted from the full KV cache, and let $\a_{-\mathcal{J}}$ denote the resulting attention output. Then the post-eviction shift $\Delta(\mathcal{J}) := \a_{-\mathcal{J}} - \a$ associated with the query $\q$ satisfies
    $$
    \Delta(\mathcal{J}) = \frac{\sum_{i \in \mathcal{J}} p_i (\a-\v_i)}{1 - P_{\mathcal{J}}},
    $$
    where $P_\mathcal{J} := \sum_{i \in \mathcal{J}} p_i$ denotes the total attention weights associated with the evicted tokens. In particular, for $\mathcal{J}=\{j\}$, we have
    $\Delta(\{j\}) = \frac{p_j}{1-p_j}\,(\a-\v_j)$ for per-token eviction.
\end{lemma}

Crucially, all quantities necessary for evaluating $\Delta(\mathcal{J})$ can be precomputed \emph{prior to the execution of any eviction decision}. Equivalent formulations of $\Delta(\mathcal{J})$ have appeared in recent studies \cite{feng2025identify,geng2025accurate}. We include the full proof in Appendix~\ref{sec:proof} for self-containedness.


\subsection{Combinatorial Objective and a Decoupled Proxy}
To determine the optimal eviction set $\J^*$ of $k$ tokens, one seeks to solve the following combinatorial optimization problem:
\begingroup
\renewcommand{\theHequation}{dropkvExact}
\begin{equation}\tag{Exact}\label{eq:opt_exact}
\J^* = \arg\min_{\substack{\mathcal{J} \subset [n] \\ |\mathcal{J}| = k}} \| \Delta(\mathcal{J})\|_2,
\end{equation}
\endgroup
where $\| \Delta(\mathcal{J})\|_2 = \frac{\| \sum_{i \in \mathcal{J}} p_i (\a - \v_i) \|_2}{1 - P_{\J}}$.
However, this optimization is NP-hard (Appendix~\ref{app:nphard}): brute-force enumeration of the $\binom{n}{k}$ feasible subsets is infeasible even for moderate $n$ and $k$, motivating a tractable proxy.

To circumvent this intractability, we propose a simple surrogate: select the $k$ tokens that individually induce the smallest output perturbation. Formally, we perform a smallest topk selection, denoted by $\mathrm{topk}^{-}$, over the per-token perturbation loss:
\begingroup
\renewcommand{\theHequation}{dropkvProxy}
\begin{equation}\tag{Proxy}\label{eq:opt_proxy}
\hat{\J} 
= \arg\mathrm{topk}_{j \in [n]}^{-} \left(\| \Delta(\{j\}) \|_2, k\right),
\end{equation}
\endgroup
where $\| \Delta(\{j\})\|_2 = \frac{p_j}{1 - p_j} \bigl\| \a - \v_j \bigr\|_2$ is the perturbation norm from evicting only the $j$-th token. Since each per-token perturbation is independent, all $n$ scores cost $O(nd)$ in parallel, making~\eqref{eq:opt_proxy} dramatically more tractable than the combinatorial problem in~\eqref{eq:opt_exact}.

Our proposed proxy decouples the joint eviction decision into independent per-token scoring by simply omitting the interactions among evicted tokens. While this independence assumption might seem crude, we prove theoretically in \S~\ref{sec:theory} that it is surprisingly \emph{near-optimal} under conditions commonly observed in practice. Specifically, when the attention weights $p_j$ are small (Assumption~\ref{ass:small-p}) and the residual directions $\a - \v_j$ are approximately aligned (Assumption~\ref{ass:cone}), the proxy $\hat\J$ closely approximates the optimal eviction policy $\J^*$ in terms of perturbation loss (Theorem~\ref{thm:optimal}).



\subsection{Multi-Query Aggregation}
For each token $j$, we define the importance score based on the cumulative eviction loss over an observation window of queries $\Q = \{\q^{(t)}\}_{t=1}^w$, similar to SnapKV \cite{li2024snapkv}. Specifically, accounting for the dependence of the attention weight $p_j^{(t)}$ and output $\a^{(t)}$ on the query $\q^{(t)}$, we score the $j$-th token by total squared perturbation error across $\Q$: 
$$\mathrm{score}_j :=\sum_{t=1}^w \|\Delta^{(t)}(\{j\})\|_2^2 = \sum_{t=1}^w \left(p_j^{(t)}/(1 - p_j^{(t)})\right)^2\| \a^{(t)} - \v_j \|_2^2.
$$
The squared form is a kernel-friendly surrogate: it enables the fused-tile implementation described in \S~\ref{subsec:triton}, while at $w{=}1$ preserving the same $\mathrm{topk}^{-}$ ranking as the unsquared proxy~\eqref{eq:opt_proxy}. Theorem~\ref{thm:optimal} establishes the constant-factor bound for the single-query proxy~\eqref{eq:opt_proxy}; the multi-query case ($w>1$) follows the standard observation-window practice~\citep{feng2025identify,li2024snapkv}. Proposition~\ref{prop:multiquery} (Appendix~\ref{sec:proof}) extends it to the multi-query score. Finally, tokens are ranked by their scores, and the $k$ tokens with the smallest scores are selected for eviction:
\begingroup
\renewcommand{\theHequation}{dropkvScore}
\begin{equation}\tag{DropKV}\label{eq:DropKV}
    \mathrm{topk}_{j \in [n]}^{-} \left(\mathrm{score}_j, k\right). 
\end{equation}
\endgroup
The pseudocode for DropKV is summarized in Algorithm~\ref{alg:dropkv}.

\subsection{Triton Kernel Implementation}\label{subsec:triton}
A na\"ive PyTorch realization of Algorithm~\ref{alg:dropkv} materializes dense window-by-cache tensors per head: the attention weights have shape $w\times n$, and a direct broadcasted implementation of the value-residual term forms $w\times n\times d$ intermediates. This memory and bandwidth pattern grows as $O(wnd)$ and dominates the scoring path at long context. We replace it with three fused Triton kernels --- a FlashAttention-2-style forward pass, the scoring kernel, and a split-K merge --- that compute $\mathrm{score}_j$ in \eqref{eq:DropKV} by streaming over key/value tiles without materializing the attention-weight or value-distance tensors.

The scoring kernel is the central contribution. Using the row-wise Log-Sum-Exp $\mathrm{LSE}^{(t)} := \log\sum_{j=1}^n \exp(\K\q^{(t)}/\sqrt{d})_j$ produced by the forward pass~\citep{dao2023flashattention2}, it streams over key tiles, recomputes logits $\mathrm{logit}_j = \k_j^\top \q^{(t)}/\sqrt{d}$ on-the-fly, derives the attention weight $p_j^{(t)} = \exp(\mathrm{logit}_j - \mathrm{LSE}^{(t)})$, and accumulates $\mathrm{score}_j \mathrel{+}= \big(p_j^{(t)}/(1-p_j^{(t)})\big)^{2}\,\|\a^{(t)} - \v_j\|_2^{2}$ across query indices $t$. By precomputing $\|\a^{(t)}\|_2^{2}$ and $\|\v_j\|_2^{2}$ once and using the identity $\|\a-\v\|_2^{2} = \|\a\|_2^{2} + \|\v\|_2^{2} - 2\langle\a,\v\rangle$, each $\V$ tile is read only once per pass, substantially reducing I/O.

The auxiliary forward and split-K merge kernels follow standard FlashAttention-style online softmax; detailed pseudocode for the scoring kernel is in Appendix~\ref{sec:triton-pseudocode}. The fused path matches the PyTorch reference to within \texttt{BF16} round-off, with a scoring-step speedup over the na\"ive PyTorch baseline that grows from $1.8{\times}$ at $n=4\mathrm{K}$ to $19.8{\times}$ at $n=128\mathrm{K}$ ($w=8$, A100-80GB; see Fig.~\ref{fig:scoring} left).

\begin{algorithm}[h]
\caption{DropKV: Decoupled Residual-Output Perturbation for KV-Cache Eviction. }
\label{alg:dropkv}
\begin{algorithmic}[1]
\REQUIRE keys $\K\in\R^{n\times d}$, values $\V\in\R^{n\times d}$,
  query window $\Q=\{\q^{(t)}\}_{t=1}^{w}$, eviction size $k$, pooling kernel size $\kappa$
\ENSURE Eviction set $\widehat{\J}\subseteq[n]$ with $|\widehat{\J}|=k$
\STATE $\forall\, t\in[w]:\ \p^{(t)}\gets\mathrm{softmax}(\K\q^{(t)}/\sqrt{d}),\;
       \a^{(t)}\gets\V^\top\p^{(t)}$\hfill
\FOR{$j = 1, \dots, n$}
  \STATE $\mathrm{score}_j \gets \sum_{t=1}^{w}\big(p_j^{(t)}/(1-p_j^{(t)})\big)^{2}\,\|\a^{(t)}-\v_j\|_2^{2}$
\ENDFOR
\STATE $\mathrm{score} \gets \mathrm{MaxPool1D}(\mathrm{score},\ \kappa)$;\ \ $\mathrm{score}_j \gets +\infty$ for $j \in \{n{-}w{+}1,\dots,n\}$\hfill
\STATE $\widehat{\J} \gets$ indices of the $k$ smallest $\mathrm{score}_j$
\RETURN $\widehat{\J}$
\end{algorithmic}
\end{algorithm}




\section{Theoretical Guarantees}\label{sec:theory}


In this section, we establish formal guarantees for our proxy~\eqref{eq:opt_proxy}. To make the analysis tractable and realistic, we restrict attention to eviction over a \emph{candidate pool} $\mathcal{L} \subseteq [n]$ (e.g., the set of tokens with the lowest attention scores) and consider only subsets $\mathcal{J} \subseteq \mathcal{L}$ with exactly $|\mathcal{J}| = k$. We begin with two assumptions and report their numerical verification in Appendix~\ref{app:assumption-verification}.

\begin{assumption}[Small Eviction Probabilities]\label{ass:small-p}
There exists $\varepsilon \in (0,1)$ with $k\varepsilon<1$ such that $p_i \le \varepsilon$ for all $i \in \mathcal{L}$.
\end{assumption}

Under this condition, any set $\mathcal{J} \subseteq \mathcal{L}$ with $|\mathcal{J}| = k$ satisfies $P_{\mathcal{J}}= \sum_{i \in \mathcal{J}} p_i \le k \varepsilon < 1$. When $\varepsilon$ is small (as is typical for low-attention tokens selected for eviction), $P_{\mathcal{J}}$ remains negligible. Next, we impose a directional alignment condition on the residual vectors $\r_i := \a - \v_i \in \mathbb{R}^d$ for $i \in \mathcal{L}$.



\begin{assumption}[Cone Condition]\label{ass:cone}
There exist a unit vector $\u \in \mathbb{R}^d$ with $\|\u\|_2 = 1$ and a constant $c \in (0,1]$ such that
$\u^\top \r_i \ge c \|\r_i\|_2, \quad \forall \, i \in \mathcal{L}.$
\end{assumption}

Intuitively, the cone condition requires the perturbations induced by different token removals to be sufficiently aligned rather than adversarially canceling each other.
Our main result, Theorem~\ref{thm:optimal}, shows that, under these assumptions, restricting the top-k selection in \eqref{eq:opt_proxy} to the candidate pool $\mathcal{L}$ yields a near-optimal objective value in \eqref{eq:opt_exact}. It provides an approximation matching the optimal solution of the NP-hard problem \eqref{eq:opt_exact} up to a constant multiplicative factor. The full proof is deferred to Appendix~\ref{sec:proof}.

\begin{theorem}[Near-Optimality]\label{thm:optimal}
Let $\J^*=\arg\min\{F(\J) := \|\Delta(\J)\|_2:\J\subseteq \L,\ |\J|=k\}$ be the optimal eviction set and
$\hat{\J}=\arg\mathrm{topk}_{j\in\L}^{-} \left(\|\Delta(\{j\})\|_2, k\right)$ the proxy.
Under Assumptions~\ref{ass:small-p} and \ref{ass:cone}, we have
$$
F(\hat{\J})
\le \frac{1}{c(1-\varepsilon)(1-k\varepsilon)}\,F(\J^*).
$$
\end{theorem}
Proposition~\ref{prop:multiquery} in Appendix~\ref{sec:proof} extends Theorem~\ref{thm:optimal} to the multi-query score~\eqref{eq:DropKV}, under the same assumptions imposed on each query.

\paragraph{Numerical Verification of the Assumptions.}
Appendix~\ref{app:assumption-verification} reports the small-weight statistics for Assumption~\ref{ass:small-p}. We verify Assumption~\ref{ass:cone} on real prefill passes at the RULER deployment length $n=128\mathrm{K}$ across our three target long-context models: Llama-3.1-8B-Instruct \cite{grattafiori2024llama}, MegaBeam-Mistral-7B-512K \cite{wu2025scaling}, and Qwen2.5-7B-Instruct-1M \cite{qwen2.5}.
For each (layer, head, query-position) triple, we compute the cone constant:
$$
c^\star \;:=\; \max_{\|\u\|=1}\,\min_{i\in\mathcal{L}}\,\frac{\u^\top\r_i}{\|\r_i\|_2},
$$
where $\mathcal{L}$ is the candidate eviction pool; Assumption~\ref{ass:cone} holds at a given triple iff $c^\star>0$. As shown in Fig.~\ref{fig:cone}~(b)\&(c), $c^\star>0$ on $91.2\%$, $91.1\%$, and $76.5\%$ of triples for Llama, Mistral, and Qwen2.5, respectively, with median $c^\star = +0.107, +0.094, +0.074$. Imposing both preconditions of Theorem~\ref{thm:optimal} simultaneously ($c^\star>0$ \emph{and} $1-k\varepsilon>0$) leaves a finite multiplicative factor on $89.1\%$, $87.5\%$, and $73.4\%$ of triples (Table~\ref{tab:assumption-bound-factor}); conditional on these preconditions, the realized factor $1/[c(1-\varepsilon)(1-k\varepsilon)]$ has median $10.4$, $11.8$, and $11.3$ across the three models (full distribution in the same table). A direct brute-force study on small subsets of $\mathcal{L}$ shows the empirical approximation ratio is in fact much smaller than this worst-case bound (median $F(\hat{\mathcal{J}})/F(\mathcal{J}^*)\le 1.16$ and 95th percentile $\le 1.43$ across all three models; Appendix~\ref{app:assumption-verification}, Table~\ref{tab:exact-optimal-gap}). Per-layer prevalence stays high from Layer~1 onward on Llama and Mistral; Llama Layer~0 ($\approx 16\%$) and Qwen2.5 Layers~0--3 (below $10\%$) exhibit pronounced early-layer dips (Fig.~\ref{fig:cone}~b).

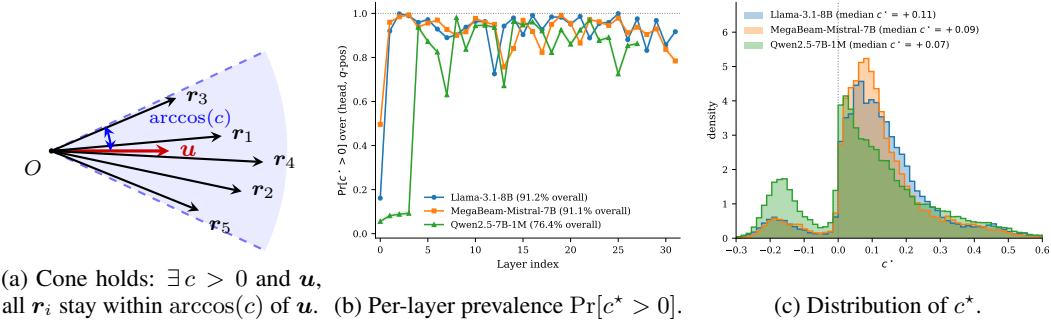
\begin{figure}[h]
\centering
\begin{subfigure}[b]{0.30\textwidth}
\centering
\begin{tikzpicture}[scale=0.78, >=stealth]
  \fill[blue!8] (0,0) -- (25:4) arc[start angle=25, end angle=-25, radius=4] -- cycle;
  \draw[blue!55, dashed, thick] (0,0) -- (25:4);
  \draw[blue!55, dashed, thick] (0,0) -- (-25:4);
  \draw[->, red!80!black, very thick] (0,0) -- (2.0,0) node[right, font=\small] {$\u$};
  \draw[->, thick] (0,0) -- (5:2.9) node[right, font=\footnotesize] {$\r_1$};
  \draw[->, thick] (0,0) -- (-12:3.3) node[right, font=\footnotesize] {$\r_2$};
  \draw[->, thick] (0,0) -- (23:2.3) node[right, font=\footnotesize] {$\r_3$};
  \draw[->, thick] (0,0) -- (-3:3.6) node[right, font=\footnotesize] {$\r_4$};
  \draw[->, thick] (0,0) -- (-22:2.7) node[below right, font=\footnotesize] {$\r_5$};
  \draw[<->, blue, thick] (1.0,0) arc[start angle=0, end angle=25, radius=1.0];
  \node[font=\scriptsize, blue, anchor=west] at (1.5, 0.55) {$\arccos(c)$};
  \fill[black] (0,0) circle (0.045);
  \node[below left, font=\footnotesize] at (0,0) {$O$};
\end{tikzpicture}
\subcaption{Cone holds: $\exists \, c>0$ and $\u$, all $\r_i$ stay within $\arccos(c)$ of $\u$.}
\label{fig:cone-a}
\end{subfigure}\hfill
\begin{subfigure}[b]{0.34\textwidth}
\centering
\includegraphics[width=\textwidth]{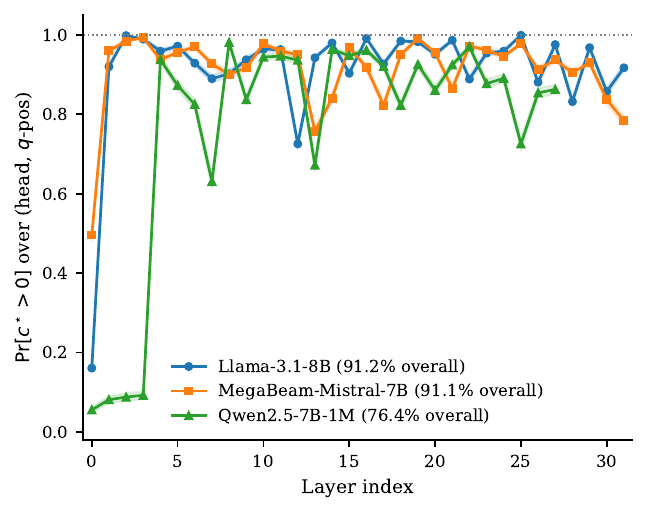}
\subcaption{Per-layer prevalence $\Pr[c^\star>0]$.}
\label{fig:cone-b}
\end{subfigure}\hfill
\begin{subfigure}[b]{0.34\textwidth}
\centering
\includegraphics[width=\textwidth]{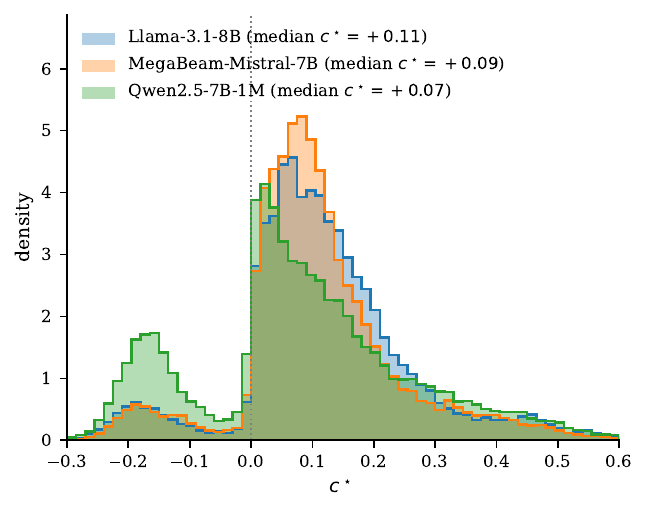}
\subcaption{Distribution of $c^\star$.}
\label{fig:cone-c}
\end{subfigure}
\caption{Cone condition: definition and empirical verification at $n{=}128\mathrm{K}$ on three long-context models. \textbf{(a)} Residuals $\r_i:=\a-\v_i$ inside a cone of half-angle $\arccos(c)$ about a single direction $\u$, so no positive combination cancels to zero. \textbf{(b)} Layerwise prevalence of $c^\star>0$ across (head, query-position) triples, with binomial standard-error bands. \textbf{(c)} Distribution of $c^\star$ per model; mass to the right of the dotted line at $c^\star{=}0$ is the cone-holds region.}
\label{fig:cone}
\end{figure}

\textbf{Robustness Analysis.} When the cone condition does not strictly hold, a relaxed approximate cone condition (Assumption~\ref{ass:approx_cone} in Appendix~\ref{subsec:robust_analysis}) bounds the weighted mass of outlier residuals by some $\mu\ge 0$. Under Assumption~\ref{ass:small-p} and this relaxed condition, the proxy satisfies the additive bound
$$
F(\hat{\J})\;\le\;\frac{1}{c(1-\varepsilon)(1-k\varepsilon)}\,\Big(F(\J^*) + (1+c)\mu\Big),
$$
which recovers Theorem~\ref{thm:optimal} when $\mu=0$. The full statement and proof are given in Appendix~\ref{subsec:robust_analysis}.

\section{Experiments}
\label{sec:experiments}
We evaluate our method on multiple long-context language models, benchmarks, and competitive baselines to examine its effectiveness across different architectures and context regimes.

{\bf Models and Baselines.}
We evaluate on three representative decoder-only LLMs: Llama-3.1-8B-Instruct~\cite{grattafiori2024llama}, MegaBeam-Mistral-7B-512K~\cite{wu2025scaling} (Mistral-7B-512K), and Qwen2.5-7B-Instruct-1M~\cite{qwen2.5}. Mistral-7B-512K is a long-context extension of Mistral-7B-Instruct-v0.2 that supports context lengths up to 512K tokens, and Qwen2.5-7B-Instruct-1M is the long-context version of Qwen2.5 with a maximum context length of 1M tokens. We compare against a broad set of competitive baselines, including SnapKV~\cite{li2024snapkv}, PyramidKV~\cite{cai2024pyramidkv}, StreamingLLM~\cite{xiao2023efficient}, KeyDiff~\cite{park2025keydiff}, CAKE~\cite{qin2025cake}, CriticalKV~\cite{feng2025identify}, and AnDPro~\cite{geng2025accurate}. For each method, we follow the default settings in the original paper whenever possible and only adjust the cache budget to enable matched-budget comparison.\footnote{Implementations of several baselines are adapted from the public codebase: \url{https://github.com/NVIDIA/kvpress}.}

{\bf Benchmarks.}
We use two widely adopted long-context benchmarks: RULER~\cite{hsieh2024ruler} and LongBench~\cite{bai2024longbench}. RULER is a synthetic benchmark with configurable sequence lengths and task complexity for evaluating long-context capabilities beyond simple in-context recall. LongBench is a multitask benchmark covering single-document QA, multi-document QA, summarization, few-shot learning, synthetic tasks, and code completion, with automatic evaluation for scalable assessment. We use the 16 English and code tasks from LongBench. The full task lists are provided in Appendix~\ref{sec:ruler_longbench_info}.

{\bf Implementation details.}
We compare all methods and report results at \textbf{5\%}, \textbf{10\%} and \textbf{30\%} KV-cache budgets. Following SnapKV~\cite{li2024snapkv}, we use the last 8 query tokens and set the pooling kernel size to 11 when computing the eviction scores. All experiments are conducted on an NVIDIA A100-80GB GPU and evaluated within a unified lm-evaluation-harness~\cite{eval-harness} pipeline to ensure a consistent and fair comparison across methods.
\vspace{-10pt}
 
\begin{table}[H]
 \centering
 \renewcommand{\arraystretch}{0.9}
 \caption{RULER results at 128K sequence length and 5\% KV-cache budget. On the Avg.\ column, the best result is bolded and the second-best is underlined; the DropKV row is shaded gray for emphasis.
 }
 \label{tab:ruler_128k_0.05}
 \resizebox{\textwidth}{!}{
 \begin{tabular}{l|c|rrrrrrrrrrr|r}
 \toprule
 \textbf{Method} & Budget & N-S1 & N-S2 & N-S3 & N-MK1 & N-MK2 & N-MQ & N-MV & VT & FWE & QA-1 & QA-2 & Avg. \\ 
 \midrule
 \multicolumn{14}{c}{\textbf{LLaMA-3.1-8B-Instruct}} \\ 
 \midrule
 Full KV & 100\% & 100 & 99.2 & 99.6 & 97.6 & 88.2 & 98.7 & 95.85 & 84 & 61.93 & 64.17 & 44 & 84.84 \\
 \midrule
 SnapKV & \multirow{8}{*}{5\%} & 100 & 99 & 71.4 & 97.6 & 76.2 & 98.6 & 95.75 & 76.92 & 52.53 & 64.03 & 43.2 & 79.57 \\
 PyramidKV & & 100 & 99.4 & 62.4 & 97.4 & 74.2 & 98.6 & 94.7 & 78.36 & 54.86 & 63.9 & 45.6 & 79.04 \\
 StreamingLLM & & 4.4 & 4.2 & 4.8 & 5.6 & 5 & 5.05 & 5 & 4.9 & 73.4 & 37.48 & 37.4 & 17.02 \\
 KeyDiff & & 100 & 98.4 & 16.6 & 81.6 & 0 & 78.3 & 75.55 & 80.2 & 26.6 & 42.58 & 36.8 & 57.88 \\
 CAKE & &  100 & 99 & 78.4 & 97.4 & 74.2 & 98.05 & 92.95 & 75.28 & 44.73 & 63.63 & 43.2 & 78.80 \\
 CriticalKV & & 100 & 99 & 69.2 & 97.6 & 75.6 & 98.55 & 95.2 & 77.4 & 46.67 & 63.9 & 42.8 & 78.72 \\
 AnDPro & & 100 & 99 & 68.4 & 97.4 & 78.4 & 98.75 & 96 & 76.52 & 54.53 & 63.76 & 43.4 & \underline{79.65} \\
 \rowcolor{gray!30} DropKV & & 100 & 98.8 & 89.6 & 97.6 & 86.2 & 98.25 & 96.9 & 77.32 & 49.73 & 64.17 & 43.4 & \textbf{82.00} \\
 \midrule
 \multicolumn{14}{c}{\textbf{Qwen2.5-7B-Instruct-1M}} \\ 
 \midrule
 Full KV & 100\% & 100 & 100 & 99.6 & 99.8 & 98.2 & 99.9 & 78.9 & 95.6 & 81.13 & 57 & 47 & 87.01 \\
 \midrule
 SnapKV &\multirow{8}{*}{5\%} & 100 & 100 & 16.4 & 99.6 & 86.2 & 99.55 & 77.1 & 96.2 & 66.4 & 56.05 & 46 & 76.68 \\
 PyramidKV & & 100 & 100 & 4 & 99.6 & 78.6 & 99.55 & 77.25 & 96.16 & 62.87 & 55.78 & 44.4 & 74.38 \\
 StreamingLLM & & 4.8 & 4.6 & 4.8 & 5.2 & 4.8 & 4.65 & 4.8 & 7.35 & 81.6 & 28.1 & 35.6 & 16.94 \\
 KeyDiff &  & 100 & 98.2 & 63.8 & 64 & 0.4 & 53.8 & 51.7 & 94.64 & 55.93 & 25.88 & 22.6 & 57.36 \\
 CAKE & & 100 & 100 & 25.6 & 99.6 & 25.8 & 99.6 & 76.45 & 96.2 & 65.67 & 56.47 & 46.6 & 72.00 \\
 CriticalKV & & 100 & 100 & 22.6 & 99.8 & 88.2 & 99.55 & 77.15 & 96.48 & 64.93 & 56.38 & 46.6  & 77.43 \\
 AnDPro & & 99.8 & 99.8 & 45.8 & 99.6 & 90 & 99.6 & 77.25 & 96.4 & 66.93 & 57.73 & 46 & \underline{79.90} \\
 \rowcolor{gray!30} DropKV & & 100 & 99.6 & 81.6 & 99.2 & 97.4 & 99.7 & 77.15 & 96.64 & 67.33 & 57.07 & 45.8 & \textbf{83.77} \\
 \midrule
 \multicolumn{14}{c}{\textbf{Mistral-7B-512K}} \\ 
 \midrule
 Full KV & 100\%  & 100 & 100 & 100 & 100 & 99.6 & 99.2 & 75.4 & 96.48 & 90.8 & 68.78 & 49 & 89.02 \\
 \midrule
SnapKV & \multirow{8}{*}{5\%} & 100 & 95.6 & 42.2 & 98 & 47.8 & 97.5 & 77.3 & 93.64 & 87.33 & 68.18 & 47.4 & 77.72 \\
PyramidKV & & 100 & 95.8 & 36 & 98.4 & 46.2 & 98.1 & 78.6 & 94.56 & 87.93 & 68.25 & 47.4 & 77.39 \\
StreamingLLM & & 5.6 & 6.4 & 7.8 & 7.2 & 6 & 6.05 & 6.2 & 7.68 & 91.2 & 33.23 & 32.2 & 19.05 \\
KeyDiff & & 100 & 64.8 & 20 & 46.6 & 1.2 & 27.35 & 31.95 & 93.36 & 73 & 34.8 & 19.6 & 46.61 \\
CAKE & & 100 & 97.2 & 56.8 & 97.6 & 29 & 97.85 & 78 & 93.24 & 84.93 & 68.12 & 47.8 & 77.32 \\
CriticalKV & & 100 & 98 & 42.2 & 99.4 & 47.4 & 97.55 & 78.05 & 93.48 & 86.87 & 68.32 & 47.4 & 78.06 \\
AnDPro & & 100 & 98.4 & 56.2 & 99.2 & 62.2 & 98.35 & 80.2 & 94.12 & 90.06 & 68.97 & 47.4 & \underline{81.37} \\
\rowcolor{gray!30} DropKV & & 100 & 96.4 & 76.4 & 99 & 92.2 & 97.05 & 79.95 & 93.48 & 88.93 & 68.63 & 47.4 & \textbf{85.40} \\
 \bottomrule
 \end{tabular}

 }
\end{table}
\vspace{-12pt}
\begin{figure}[h!]
    \centering
    \includegraphics[width=0.81\textwidth]{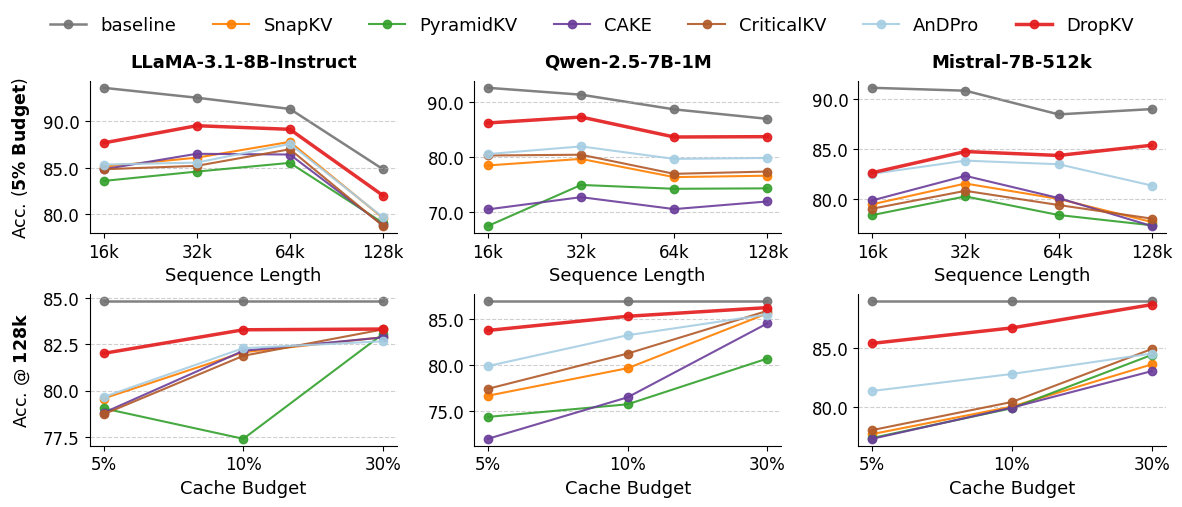}
    \vspace{-10pt}
    \caption{Average accuracy on RULER. \textbf{Top:} 5\% budget at varying sequence lengths. \textbf{Bottom:} varying cache budgets at $128\mathrm{K}$ sequence length. Our method consistently outperforms baselines and shows more graceful degradation at longer contexts.}
    \label{fig:ruler_all_results}
\end{figure}
\vspace{-5pt}
\subsection{RULER}
We evaluate all methods on the RULER across sequence lengths from 16K to 128K. Fig.~\ref{fig:ruler_all_results} shows accuracy across sequence lengths at 5\% budget (top) and across cache budgets at 128K (bottom). Table~\ref{tab:ruler_128k_0.05} reports the detailed 128K results at 5\% budget. Overall, our method consistently achieves the best performance across all three model families and remains the closest to the Full-KV baseline.

At as low as 5\% budget, our method achieves the highest average accuracy on all models: 82.00 on Llama-3.1-8B-Instruct, 83.77 on Qwen2.5-7B-1M, and 85.40 on Mistral-7B-512K. Compared with the best competing baseline, this corresponds to gains of 2.35, 3.87, and 4.03 points, respectively. Fig.~\ref{fig:ruler_all_results} further shows that the advantage of our method is consistent across context lengths. In particular, under high compression, competing methods exhibit a much sharper degradation as the sequence length increases, whereas our method degrades more gracefully and consistently outperforms prior approaches. This trend is especially clear at 128K, where our method preserves accuracy significantly better across all three models. For brevity, Table~\ref{tab:ruler_128k_0.05} presents the representative 128K results at 5\% budget. Additional RULER results are provided in Appendix~\ref{sec:additional_result}, including full per-task breakdowns as well as evaluations across different cache budget and sequence lengths; See Table~\ref{tab:ruler_128k_appendix}--\ref{tab:ruler_16k_appendix}.

\subsection{LongBench}
Table~\ref{tab:longbench_0.05} reports the LongBench results at 5\% KV-cache budget. Our method achieves the highest average score on all three model families, reaching 45.51 on Llama-3.1-8B-Instruct, 42.66 on Qwen2.5-7B-1M, and 36.22 on Mistral-7B-512K. Compared with the strongest baseline, this gives improvements of 0.87, 0.94, and 0.68 points, respectively.

These results show that our method remains effective across a diverse set of LongBench tasks under aggressive compression ratio. In addition to the gains in the overall average, our method also performs competitively on multiple task categories. In the main paper, we only report the results at 5\% budget in Table~\ref{tab:longbench_0.05}; additional results are provided in Appendix~\ref{sec:additional_result} (Table~\ref{tab:longbench_0.9} \& \ref{tab:longbench_0.7}).
 
\begin{table}[!h]
 \renewcommand{\arraystretch}{0.9}
 \centering
 \caption{LongBench results at 5\% KV-cache budget.} 
 \label{tab:longbench_0.05}
 \resizebox{\textwidth}{!}{
 \begin{tabular}{l|cccccccccccccccc|r}
 \toprule
 
 & \multicolumn{3}{c}{\textbf{Single Doc. QA}} & \multicolumn{3}{c}{\textbf{Multi Doc. QA}} & \multicolumn{3}{c}{\textbf{Summarization}} & \multicolumn{2}{c}{\textbf{Synthetic}} & \multicolumn{3}{c}{\textbf{Fewshot Learning}} & \multicolumn{2}{c}{\textbf{Code}} & \\ 
 \cmidrule(lr){2-4} \cmidrule(lr){5-7} \cmidrule(lr){8-10} \cmidrule(lr){11-12} \cmidrule(lr){13-15} \cmidrule(lr){16-17}
 Method & \rotatebox{45}{NrtvQA} & \rotatebox{45}{Qasper} & \rotatebox{45}{MF-en} & \rotatebox{45}{HotpotQA} & \rotatebox{45}{2WikiMQA} & \rotatebox{45}{Musique} & \rotatebox{45}{GovReport} & \rotatebox{45}{QMSum} & \rotatebox{45}{MultiNews} & \rotatebox{45}{PCount} & \rotatebox{45}{PR-en} & \rotatebox{45}{TREC} & \rotatebox{45}{TriviaQA} & \rotatebox{45}{SAMSum} & \rotatebox{45}{Lcc} & \rotatebox{45}{RB-P} & \textbf{Avg.} \\ 
 \midrule
  
 \multicolumn{18}{c}{\textbf{LLaMA-3.1-8B-Instruct}} \\ 
 \midrule
 FullKV & 30.51 & 45.59 & 56.34 & 58.22 & 50.02 & 31.75 & 34.3 & 25.33 & 26.76 & 9.5 & 100 & 73.5 & 92.29 & 43.96 & 64.77 & 55.7 & 49.91 \\
 \midrule
 SnapKV & 30.34 & 31.51 & 54.45 & 56.95 & 47.34 & 29.8 & 22.64 & 24.62 & 16.03 & 9.5 & 100 & 39 & 92.03 & 42.18 & 58.82 & 54.97 & 44.39 \\
 PyramidKV & 31.08 & 31.42 & 54.37 & 56.92 & 47.34 & 30.62 & 22.4 & 24.27 & 16.33 & 9.5 & 100 & 39.5 & 91.92 & 42.46 & 59.44 & 54.27 & 44.49 \\
 StreamingLLM & 25.91 & 21.94 & 33.16 & 49.06 & 41.35 & 23.91 & 21.67 & 21.5 & 16.86 & 9.5 & 97 & 26.5 & 84.55 & 41.9 & 58.02 & 49.83 & 38.92 \\
 KeyDiff & 27.93 & 27.17 & 38.52 & 51.92 & 36.8 & 26.54 & 23.02 & 22.91 & 18.31 & 9.5 & 63.5 & 30.5 & 78.52 & 22.36 & 32.21 & 31.65 & 33.84 \\
 CAKE & 30.41 & 33.54 & 54.79 & 57.09 & 48.29 & 31.07 & 23.54 & 24.36 & 17.33 & 9.5 & 100 & 37 & 91.05 & 42.54 & 58.58 & 53.7 & 44.55 \\
 CriticalKV & 30.03 & 31.07 & 54.55 & 57.33 & 47.15 & 30.09 & 23.16 & 23.92 & 18.46 & 9.5 & 100 & 41 & 92.03 & 42.14 & 59.25 & 54.54 & \underline{44.64} \\
 AnDPro & 30.97 & 32.06 & 54.44 & 57.62 & 46.44 & 31.53 & 22.92 & 24.25 & 17.98 & 9.25 & 100 & 38 & 91.65 & 41.86 & 58.58 & 54.9 & 44.53 \\
 \rowcolor{gray!30} DropKV & 30.48 & 32.49 & 55.34 & 57.29 & 46.84 & 31.25 & 23.71 & 24.94 & 18.7 & 9.5 & 100 & 48.5 & 92.27 & 43.06 & 60.29 & 53.47 & \textbf{45.51} \\
 \midrule
 \multicolumn{18}{c}{\textbf{Qwen2.5-7B-Instruct-1M}} \\ 
 \midrule
 FullKV & 28.88 & 48.39 & 51.57 & 60.27 & 53.19 & 35.63 & 33.33 & 23.22 & 24.53 & 6.5 & 100 & 66.5 & 85.28 & 45.29 & 45.61 & 35.31 & 46.47\\
 \midrule
 SnapKV & 27.76 & 35.9 & 48.74 & 59.15 & 51.51 & 34.19 & 23.42 & 22.32 & 13.13 & 6.5 & 100 & 33 & 84.66 & 44.2 & 38.75 & 34.47 & 41.11 \\
 PyramidKV & 27.19 & 34.45 & 48.49 & 57.01 & 51.57 & 32.37 & 22.67 & 21.5 & 13.09 & 5.5 & 99 & 33.5 & 84.93 & 43.55 & 38.46 & 33.07 & 40.40 \\
 StreamingLLM & 21.48 & 28.86 & 27.5 & 47.08 & 42.58 & 22.99 & 21.86 & 18.51 & 13.97 & 6.5 & 16.5 & 27 & 80.62 & 42.26 & 37.84 & 32.34 & 30.49 \\
 KeyDiff & 20.48 & 24.5 & 21.94 & 38.44 & 25.67 & 17.35 & 21.19 & 17.17 & 12.11 & 6.5 & 13 & 19.5 & 42.65 & 13.11 & 20.75 & 18.96 & 20.83 \\
 CAKE & 28.75 & 38.51 & 49.22 & 59.59 & 52.14 & 33.85 & 24.12 & 22.59 & 14.08 & 7 & 100 & 36 & 84.71 & 44.35 & 37.53 & 35.06 & \underline{41.72} \\
 CriticalKV & 28.31 & 38.22 & 48.83 & 57.63 & 51.99 & 34.11 & 24.4 & 22.35 & 14.83 & 7 & 100 & 36.5 & 84.61 & 43.83 & 38.59 & 34.73 & 41.52\\
 AnDPro & 28.61 & 36.11 & 49.88 & 59.44 & 50.87 & 34.56 & 23.61 & 22.24 & 14.22 & 6.5 & 100 & 37.5 & 84.71 & 44.21 & 38.08 & 34.76 & 41.58 \\
 \rowcolor{gray!30} DropKV & 28.28 & 39.09 & 50.54 & 58.87 & 51.6 & 33.93 & 25.25 & 21.42 & 15.36 & 7 & 100 & 48.5 & 85.06 & 44.85 & 38.95 & 33.88 & \textbf{42.66} \\
 \midrule
 \multicolumn{18}{c}{\textbf{Mistral-7B-512K}} \\ 
 \midrule
 FullKV & 22.04 & 16.47 & 42.75 & 24.06 & 20.72 & 13.04 & 33.35 & 25.25 & 26.97 & 0 & 100 & 74.5 & 89.55 & 41.69 & 55.54 & 56.39 & 40.15 \\
 \midrule
 SnapKV & 21.39 & 10.98 & 41.92 & 18.68 & 16.25 & 9.21 & 23.38 & 24.53 & 17.4 & 0 & 100 & 38.5 & 90.3 & 39.23 & 50.64 & 53.6 & 34.75 \\
 PyramidKV & 20.4 & 11.25 & 42.22 & 19.62 & 15.67 & 8.3 & 23.34 & 24.74 & 17.18 & 0 & 100 & 38.5 & 90.13 & 39.09 & 50.67 & 52.81 & 34.62 \\
 StreamingLLM & 18.25 & 9.28 & 25.16 & 14.07 & 13.96 & 5.92 & 18.79 & 21.41 & 16.65 & 0.5 & 20.5 & 22.5 & 86.81 & 38.74 & 47.26 & 49.87 & 25.60 \\
 KeyDiff & 19.49 & 10.43 & 22.08 & 18.99 & 14.74 & 7.25 & 22.09 & 21.76 & 19.13 & 0.5 & 6.5 & 16 & 58.84 & 17.53 & 24.25 & 23.84 & 18.96 \\
 CAKE & 21.3 & 11.7 & 43.07 & 18.54 & 14.95 & 11.11 & 24.3 & 24.67 & 19.67 & 0 & 100 & 46 & 89.39 & 39.62 & 50.1 & 54.25 & \underline{35.54} \\
 CriticalKV & 21.19 & 13.72 & 42.78 & 17.72 & 15.83 & 9.89 & 24.15 & 24.18 & 20.12 & 0 & 100 & 43 & 90.08 & 40.29 & 51.77 & 53.88 & \underline{35.54} \\
 AnDPro & 19.93 & 12.96 & 40.49 & 19.01 & 14.17 & 9.74 & 23.57 & 24.51 & 19.11 & 0.5 & 100 & 44.5 & 90.3 & 39.23 & 49.36 & 53.23 & 35.04 \\
 \rowcolor{gray!30} DropKV & 21.52 & 11.64 & 41 & 17.6 & 14.39 & 10.58 & 25.23 & 24.98 & 20.13 & 0.5 & 100 & 57 & 90.3 & 39.99 & 52.76 & 51.9 & \textbf{36.22} \\
 \bottomrule
 \end{tabular}

 }
\end{table}
\vspace{-5pt}

\subsection{Extended Results with Adaptive Head Budget} 
To evaluate Dropkv under adaptive compression constraints, Table~\ref{tab:ruler_adadropkv_ruler_longbench} presents results on RULER (at 64K sequence length) and LongBench across different KV-cache budgets ($5\%, 10\%$, and $30\%$). We implement AdaDropKV, which adopts the head budget allocation strategy of AdaKV while utilizing DropKV's token-scoring mechanism. AdaDropKV consistently outperforms competing adaptive baselines (such as CAKE, AdaKV, and AnDPro) across both model architectures. Additional results are provided in Appendix~\ref{sec:additional_result}, including full per-task breakdowns as well as evaluations across different cache budget and sequence lengths; See Table~\ref{tab:ruler_adadropkv_appendix} \& ~\ref{tab:longbench_adadropkv}.
\begin{table}[htbp]
 \centering
 \renewcommand{\arraystretch}{0.7}
 \caption{Average scores on RULER (64K sequence length) and LongBench across different KV-cache budgets ($5\%, 10\%,$ and $30\%$). The best result is highlighted in bold.
 }
 \label{tab:ruler_adadropkv_ruler_longbench}
 \begin{tabular}{l|ccc|ccc}
 \toprule
 \multirow{2}{*}{Method} & \multicolumn{3}{c}{RULER} & \multicolumn{3}{|c}{LongBench} \\
 \cmidrule(lr){2-4} \cmidrule(lr){5-7}
 & 5\% & 10\% & 30\% & 5\% & 10\% & 30\% \\
 \midrule
 \multicolumn{7}{c}{\textbf{LLaMA-3.1-8B-Instruct}} \\ 
 \midrule
 Full KV & 91.35 & 91.35 & 91.35 & 49.91 & 49.91 & 49.91 \\
 \midrule
 CAKE & 86.43 & 89.67 & 90.93 & 44.55 & 46.74 & 48.82 \\
 AdaKV & 88.19 & 90.03 & 91.03 & 44.76 & 46.95 & 49.05 \\
 AnDPro & 87.57 & 90.01 & 91.17 & 44.53 & 46.95 & 49.06 \\
 AdaDropKV & \textbf{90.08} & \textbf{90.54} & \textbf{91.30} & \textbf{46.45} & \textbf{48.43} & \textbf{49.47} \\
 \midrule
 \multicolumn{7}{c}{\textbf{Qwen2.5-7B-Instruct-1M}} \\ 
 \midrule
 FullKV & 88.73 & 88.73 & 88.73 & 46.47 & 46.47 & 46.47 \\
 \midrule
 CAKE & 70.61 & 75.91 & 84.21 & 41.72 & 43.69 & 45.79 \\
 AdaKV & 76.45 & 79.60 & 85.40 & 41.38 & 43.38 & 45.62 \\
 AnDPro & 79.73 & 82.99 & 86.29 & 41.58 & 43.94 & 45.85 \\
 AdaDropKV & \textbf{84.25} & \textbf{85.79} & \textbf{87.76} & \textbf{43.02} & \textbf{44.80} & \textbf{46.00} \\
 \bottomrule
 \end{tabular}

\end{table}

\subsection{Memory and Latency}
We evaluate the Triton path against a matched PyTorch reference on the three long-context LLMs on an A100-80GB.

\textbf{Isolated scoring} (Fig.~\ref{fig:scoring} left). At $n{=}128\mathrm{K}$, $B{=}1$, $w{=}8$, and $5\%$ budget the fused Triton kernel is $19.8{\times}$ faster than the PyTorch reference while holding scoring-kernel scratch flat at ${\sim}17$\,MB vs. ${\sim}9$\,GB (a ${\sim}540{\times}$ scratch reduction; total peak allocation drops from $10.6$\,GB to $1.6$\,GB on the heaviest of the three target models).

\textbf{End-to-end prefill at matched batch size} (Fig.~\ref{fig:scoring} right). At the lower of the two methods' max batch per cell, matched-batch prefill speedup $\delta:=(t_{\mathrm{py}}{-}t_{\mathrm{tri}})/t_{\mathrm{tri}}$ is positive in eight of nine $(n,\text{model})$ cells of $n{\in}\{4\mathrm{K},32\mathrm{K},128\mathrm{K}\}$, ranging $+1.9\%$ to $+9.9\%$; the lone exception is Qwen at $n{=}32\mathrm{K}$ ($\delta{=}{-}0.1\%$). The dynamic cache (no eviction) admits a strictly smaller batch than DropKV at every $n$ and model, so a matched-batch comparison with dynamic is not defined; we instead compare each method at its own capacity ceiling below.

\textbf{Serving throughput vs.\ dynamic cache.} Decoding throughput at each method's maximum feasible batch is $1.75$--$8.22{\times}$ higher than the dynamic cache at $5\%$ budget (Table~\ref{tab:serving-throughput} in Appendix~\ref{sec:additional_result}). This gain comes from the KV-bandwidth reduction shared by any eviction method at this budget: per-step decode latency is bandwidth-bound by the (post-eviction) KV size and is therefore comparable across eviction methods. Triton's contribution is concentrated in prefill (the $19.8{\times}$ scoring speedup and ${\sim}540{\times}$ scratch reduction reported above), not per-step decode; Triton vs.\ PyTorch differences in $B$ within Table~\ref{tab:serving-throughput} reflect the ${\sim}9$\,GB prefill-scratch reduction crossing the batch-fit boundary in some cells (e.g., Mistral $n{=}32\mathrm{K}$: $B{=}12$ vs.\ $11$).

\begin{figure}[h]
\centering
\newlength{\figthreerowht}%
\settoheight{\figthreerowht}{\includegraphics[width=0.72\textwidth]{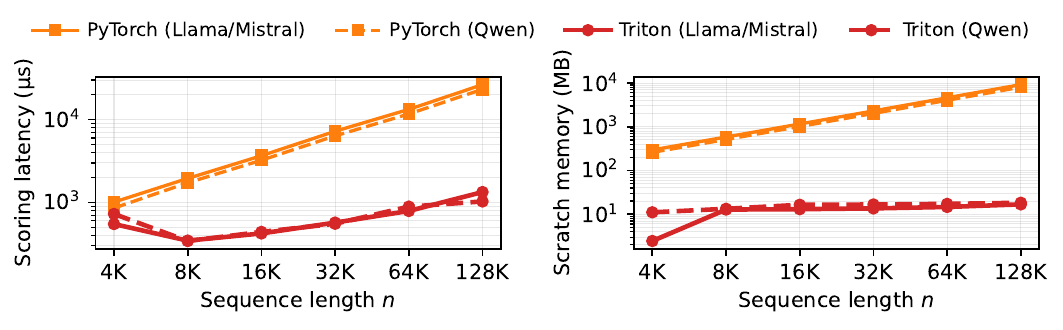}}%
\begin{minipage}[t][\figthreerowht]{0.72\textwidth}
  \centering
  \vspace{0pt}
  \includegraphics[width=\textwidth]{fig3_scoring_variants.pdf}%
  \label{fig:scoring-a}
\end{minipage}\hfill
\begin{minipage}[t][\figthreerowht]{0.26\textwidth}
  \centering
  \vspace{10pt}
  \footnotesize
  \setlength{\tabcolsep}{3pt}
  \renewcommand{\arraystretch}{1.15}
  \begin{tabular}{lccc}
  \toprule
          & $4\mathrm{K}$ & $32\mathrm{K}$ & $128\mathrm{K}$ \\
  \midrule
  Llama   & $+7.2$ & $+9.9$ & $+2.5$ \\
  Mistral & $+7.0$ & $+4.9$ & $+2.7$ \\
  Qwen    & $+1.9$ & $-0.1$ & $+3.3$ \\
  \bottomrule
  \end{tabular}%
  \label{fig:scoring-b}
\end{minipage}
\vspace{-5pt}
\caption{DropKV scoring at $B{=}1$, $w{=}8$, and $5\%$ budget. \textbf{Left:} Triton vs.\ PyTorch on scoring latency and peak scratch memory. At $n{=}128\mathrm{K}$ Triton is ${\sim}20{\times}$ faster with scratch flat at ${\sim}17$\,MB vs.\ ${\sim}9$\,GB. \textbf{Right:} Matched-batch prefill relative speedup percentage is positive in 8 of 9 cells; Qwen at $n{=}32\mathrm{K}$ is essentially neutral ($-0.1\%$).}
\label{fig:scoring}
\end{figure}

\textbf{End-to-end full generation latency.} Table~\ref{tab:latency_bs} reports full-generation latency and time-to-first-token (TTFT) on Mistral-7B-512K with $32\mathrm{K}$ seq. length, 500 decoding steps, 5\% budget. DropKV achieves the lowest full-generation latency across all batch sizes; the dynamic cache runs out of memory beyond $B{=}4$, while all eviction methods scale to $B{=}12$. Among eviction methods, DropKV consistently delivers the lowest total generation time and the lowest TTFT at every batch size; its TTFT also matches the dynamic baseline within $0.6\%$ where Dynamic is feasible, indicating that the prefill eviction overhead remains minimal while the compressed cache improves decoding efficiency.

\vspace{-8pt}
\begin{table}[htbp]
  \centering
  \renewcommand{\arraystretch}{0.88}
  \caption{End-to-End and Time to First Token latencies (in seconds) across batch sizes $B\in\{1,2,4,8,12\}$. All evaluations were conducted using the Mistral-7B-512K model with 32K sequence length. $\downarrow$ indicates lower values are better. ``OOM'' denotes Out-of-Memory.}
  \label{tab:latency_bs}
  \resizebox{\textwidth}{!}{
  \begin{tabular}{l ccccc|ccccc}
    \toprule
    \multirow{2}{*}{Method} & \multicolumn{5}{c}{End-to-End Latency (s) $\downarrow$} & \multicolumn{5}{|c}{Time to First Token (s) $\downarrow$} \\
    \cmidrule(lr){2-6} \cmidrule(lr){7-11}
    & 1 & 2 & 4 & 8 & 12 & 1 & 2 & 4 & 8 & 12 \\
    \midrule
    Dynamic     & 19.08 & 30.91 & 54.72 & OOM   & OOM   & 3.69          & \textbf{7.36} & 14.83          & OOM            & OOM \\
    SnapKV      & 16.44 & 20.22 & 27.71 & 43.08 & 59.10 & 3.95          & 7.47          & 14.92          & 30.36          & 46.38 \\
    CriticalKV  & 16.72 & 21.10 & 29.55 & 46.36 & 64.75 & 4.24          & 8.27          & 16.68          & 33.62          & 51.96 \\
    AnDPro      & 17.79 & 23.06 & 33.22 & 53.91 & 76.25 & 5.31          & 10.30         & 20.45          & 41.00          & 63.43 \\
    \textbf{DropKV} & \textbf{16.21} & \textbf{20.13} & \textbf{27.50} & \textbf{42.71} & \textbf{59.01} & \textbf{3.68} & 7.40          & \textbf{14.75} & \textbf{29.94} & \textbf{46.29} \\
    \bottomrule
  \end{tabular}
  }
\end{table}


\section{Discussion and Limitation}
\label{sec:discussion}
In this work, we introduced DropKV, a framework that bridges the gap between theoretical optimality and practical scalability in KV cache eviction. By decoupling the joint perturbation problem, we demonstrated that near-optimal cache eviction can be achieved with minimal computational overhead. Our empirical results suggest that the cone condition, defined as the alignment of value residuals, holds at the majority of (layer, head, query-position) triples in our three target long-context LLMs, providing a mathematical justification for why independent per-token scoring remains effective. 
In future work, we plan to extend DropKV to a Triton decode-time kernel, since our current implementation accelerates prefill scoring only.

\begin{ack}
This work was in part supported by NSF grant DMS-2208126, two UAlbany-IBM CEAIS seed grants, and an IBM PhD Fellowship Award.

\end{ack}
\bibliographystyle{plain}
\bibliography{reference}

\newpage

\appendix

\normalsize
\section*{Appendix}

\section{Additional Experimental Results}\label{sec:additional_result}

\subsection{Extended RULER Results}\label{app:ruler-extended}

Tables~\ref{tab:ruler_128k_appendix}--\ref{tab:ruler_16k_appendix} extend the RULER evaluation in \S\ref{sec:experiments} to additional sequence lengths ($n\in\{16\mathrm{K}, 32\mathrm{K}, 64\mathrm{K}\}$ at $5\%$, $10\%$, and $30\%$ budgets) and to looser cache budgets at $n{=}128\mathrm{K}$ ($10\%$ and $30\%$), complementing the headline $5\%$ result in Table~\ref{tab:ruler_128k_0.05}.

\begin{table}[H]
 \centering
 \caption{RULER results at 128K sequence length, 10\% and 30\% KV-cache budgets. The best result is highlighted in gray, and the second-best is underlined.
 }
 \label{tab:ruler_128k_appendix}
 \resizebox{\textwidth}{!}{
 \begin{tabular}{l|c|rrrrrrrrrrr|r}
 \toprule
 \textbf{Method} & Budget & N-S1 & N-S2 & N-S3 & N-MK1 & N-MK2 & N-MQ & N-MV & VT & FWE & QA-1 & QA-2 & Avg. \\ 
 \midrule
 \multicolumn{14}{c}{\textbf{LLaMA-3.1-8B-Instruct}} \\ 
 \midrule
 FullKV & 100\% & 100 & 99.2 & 99.6 & 97.6 & 88.2 & 98.7 & 95.85 & 84 & 61.93 & 64.17 & 44 & 84.84 \\
 \midrule
 SnapKV & \multirow{8}{*}{10\%} & 100 & 99 & 90.8 & 97.4 & 78.2 & 98.45 & 96.4 & 78.92 & 56.47 & 63.23 & 43.8 & 82.06\\
 PyramidKV & & 100 & 100 & 55.47 & 97.66 & 71.88 & 98.83 & 93.75 & 76.09 & 53.91 & 60.81 & 42.97 & 77.40 \\
 StreamingLLM & & 11.72 & 8.59 & 14.84 & 11.72 & 12.50 & 10.55 & 8.79 & 10.94 & 75.78 & 34.83 & 32.82 & 21.19 \\
 KeyDiff & & 100 & 100 & 57.03 & 94.53 & 1.56 & 95.51 & 94.92 & 78.44 & 45.31 & 43.68 & 35.16 & 67.83 \\
 CAKE & & 100 & 99 & 94.2 & 97.6 & 80 & 98.7 & 95.35 & 77.08 & 54.4 & 63.57 & 43.6 & 82.14 \\
 CriticalKV & & 100 & 98.8 & 90.4 & 97.4 & 81.8 & 98.5 & 95.55 & 78.64 & 52.73 & 63.3 & 43.4 & 81.87 \\
 AnDPro & & 100 & 99 & 89.2 & 97.4 & 82.6 & 99 & 95.6 & 78.96 & 56.27 & 63.3 & 43.8 & \underline{82.28} \\
 \rowcolor{gray!30} DropKV & & 100 & 98.6 & 97.6 & 97.6 & 87.6 & 98.8 & 97.1 & 79.24 & 51.87 & 64.16 & 43.4 & \textbf{83.27} \\
 \midrule
 SnapKV & \multirow{8}{*}{30\%} & 100 & 100 & 99.22 & 97.66 & 84.38 & 99.22 & 94.09 & 78.44 & 57.9 & 59.24 & 41.41 & 82.87 \\
 PyramidKV & & 100 & 100 & 99.22 & 97.66 & 85.94 & 99.22 & 94.34 & 79.22 & 56.25 & 59.51 & 42.19 & 83.05 \\
 StreamingLLM & & 32.03 & 28.13 & 34.38 & 35.16 & 26.56 & 31.84 & 29.1 & 32.81 & 80.73 & 67.32 & 36.72 & 39.53 \\
 KeyDiff & & 100 & 100 & 99.22 & 98.44 & 17.19 & 99.22 & 100 & 77.97 & 58.59 & 50 & 37.5 & 76.19 \\
 CAKE & & 100 & 100 & 99.22 & 97.66 & 85.16 & 99.22 & 94.92 & 78.44 & 60.24 & 59.24 & 41.41 & 83.23 \\
 CriticalKV & & 100 & 100 & 99.22 & 97.66 & 85.94 & 99.22 & 95.51 & 77.26 & 60.68 & 59.25 & 41.41 & \underline{83.29} \\
 AnDPro & & 100 & 100 & 97.66 & 97.66 & 84.38 & 99.02 & 93.55 & 76.56 & 58.33 & 60.55 & 41.41 & 82.65 \\
 \rowcolor{gray!30} DropKV & & 100 & 100 & 99.22 & 97.66 & 85.94 & 99.22 & 96.09 & 77.5 & 59.38 & 60.03 & 41.41 & \textbf{83.31} \\
 \midrule
 \multicolumn{14}{c}{\textbf{Qwen2.5-7B-Instruct-1M}} \\ 
 \midrule
 FullKV & 100\% & 100 & 100 & 99.6 & 99.8 & 98.2 & 99.9 & 78.9 & 95.6 & 81.13 & 57 & 47 & 87.01 \\
 \midrule
 SnapKV &\multirow{8}{*}{10\%} & 100 & 100 & 42.6 & 99.8 & 87.8 & 99.6 & 76.55 & 96.32 & 70.53 & 56.67 & 46.6 & 79.68 \\
 PyramidKV & & 100 & 100 & 11 & 99.8 & 82.8 & 99.7 & 76.95 & 96.24 & 66.13 & 55.58 & 45 & 75.75 \\
 StreamingLLM & & 7.81 & 8.59 & 12.50 & 8.59  & 13.28 & 9.96 & 8.98 & 15.16 & 84.11 & 23.50 & 37.50 & 20.91 \\
 KeyDiff &  & 100 & 99.22 & 92.97 & 85.16 & 2.34  & 81.84 & 73.24 & 95.47 & 58.85 & 25.59 & 31.25 & 67.81 \\
 CAKE & & 100 & 100 & 60.4 & 99.6 & 37.2 & 99.75 & 76.95 & 95.88 & 69.8 & 56.2 & 45.8 & 76.51 \\
 CriticalKV & & 100 & 100 & 58.2 & 99.6 & 92 & 99.6 & 76.65 & 96.12 & 69.13 & 56.73 & 45.8 & 81.26 \\
 AnDPro & & 99.8 & 100 & 74 & 99.8 & 94.6 & 99.75 & 76.9 & 96.48 & 71.4 & 57.4 & 45.6 & \underline{83.25} \\
 \rowcolor{gray!30} DropKV & & 100 & 99.4 & 97 & 99.2 & 97.8 & 99.7 & 77.1 & 96.48 & 69.6 & 56.13 & 46 & \textbf{85.31} \\
 \midrule
 SnapKV & \multirow{8}{*}{30\%} & 100 & 100 & 92.97 & 100 & 96.09 & 99.81 & 75.98 & 96.09 & 76.82 & 56.38 & 47.66 & 85.62 \\
 PyramidKV & & 100 & 100 & 50 & 100 & 89.84 & 100 & 76.56 & 95.31 & 72.14 & 56.38 & 47.66 & 80.72 \\
 StreamingLLM & & 30.47 & 32.81 & 35.16 & 30.47 & 32.03 & 30.66 & 28.52 & 47.66 & 84.9 & 59.44 & 38.28 & 40.95 \\
 KeyDiff & & 100 & 100 & 100 & 98.44 & 16.41 & 100 & 81.45 & 95.16 & 67.97 & 35.48 & 36.72 & 75.60 \\
 CAKE & & 100 & 100 & 99.22 & 100 & 79.69 & 99.81 & 75.98 & 95.94 & 77.34 & 56.64 & 45.31 & 84.54 \\
 CriticalKV & & 100 & 100 & 98.44 & 100 & 97.66 & 99.8 & 75.39 & 96.41 & 74.74 & 54.82 & 47.66 & \underline{85.90} \\
 AnDPro & & 100 & 100 & 89.06 & 100 & 97.66 & 99.61 & 75.97 & 95.78 & 78.39 & 55.6 & 48.44 & 85.50 \\
 \rowcolor{gray!30} DropKV & & 100 & 100 & 99.22 & 100 & 99.22 & 99.61 & 75.78 & 96.09 & 73.96 & 57.16 & 47.66 & \textbf{86.25} \\
 \midrule
 \multicolumn{14}{c}{\textbf{Mistral-7B-512K}} \\ 
 \midrule
 FullKV & 100\%  & 100 & 100 & 100 & 100 & 99.6 & 99.2 & 75.4 & 96.48 & 90.8 & 68.78 & 49 & 89.02 \\
 \midrule
 SnapKV & \multirow{8}{*}{10\%} & 100 & 96 & 72.8 & 98.6 & 38.8 & 98.5 & 79.7 & 94.8 & 86.27 & 68.12 & 47 & 80.05 \\
 PyramidKV & & 100 & 97.4 & 66.4 & 99.2 & 40.2 & 98.9 & 78.7 & 95.16 & 87.53 & 68.12 & 47.4 & 79.91 \\
 StreamingLLM & & 7.03 & 14.84 & 14.06 & 11.72 & 10.16 & 11.13 & 12.31 & 12.50 & 90.89 & 28.71 & 28.91 & 22.02 \\
 KeyDiff & & 100 & 89.84 & 64.84 & 75.78 & 3.91 & 73.44 & 57.42 & 94.06 & 78.13  & 41.02 & 24.22 & 63.88 \\
 CAKE & & 100 & 97.2 & 84.2 & 98.8 & 29 & 98.85 & 78.5 & 93.88 & 83.8 & 67.78 & 47.6 & 79.96 \\
 CriticalKV & & 100 & 98 & 73 & 99.2 & 41.4 & 98.45 & 78.55 & 94.16 & 87.4 & 67.73 & 47 & 80.44 \\
 AnDPro & & 100 & 98 & 85.4 & 99.6 & 48.4 & 99.1 & 79.95 & 94.44 & 90.47 & 67.78 & 47.8 & \underline{82.81} \\
 \rowcolor{gray!30} DropKV & & 100 & 97 & 88.4 & 100 & 93.2 & 98.2 & 79.05 & 94.24 & 88.2 & 67.77 & 47.8 & \textbf{86.71}\\
 \midrule
 SnapKV & \multirow{8}{*}{30\%} & 100 & 100 & 98.44 & 100 & 42.19 & 99.41 & 75.98 & 95.62 & 87.24 & 71.74 & 49.22 & 83.62 \\
 PyramidKV & & 100 & 100 & 96.88 & 100 & 55.47 & 99.61 & 75.78 & 95.78 & 86.46 & 70.18 & 48.44 & 84.42 \\
 StreamingLLM & & 30.47 & 27.34 & 31.25 & 26.56 & 31.25 & 27.34 & 31.05 & 40.78 & 90.63 & 73.83 & 36.72 & 40.66 \\
 KeyDiff & & 100 & 98.44 & 98.44 & 94.53 & 21.88 & 98.24 & 72.85 & 95.62 & 80.99 & 56.12 & 37.5 & 77.69 \\
 CAKE & & 100 & 99.22 & 100 & 100 & 36.72 & 99.61 & 76.17 & 95.62 & 85.16 & 71.22 & 49.22 & 82.99 \\
 CriticalKV & & 100 & 100 & 100 & 100 & 56.25 & 99.41 & 75.78 & 94.53 & 88.8 & 70.96 & 48.44 & \underline{84.92} \\
 AnDPro & & 100 & 99.22 & 99.22 & 100 & 48.44 & 99.41 & 76.95 & 95.62 & 90.89 & 70.96 & 49.22 & 84.54 \\
 \rowcolor{gray!30} DropKV & & 100 & 97.66 & 100 & 99.22 & 97.65 & 99.02 & 78.32 & 95.62 & 87.76 & 70.96 & 49.22 & \textbf{88.68} \\
 \bottomrule
 \end{tabular}
 }
\end{table}
 
\begin{table}[H]
 \centering
 \renewcommand{\arraystretch}{0.88}
 \caption{RULER results at 64K sequence length and 5\%, 10\% and 30\% KV-cache budget. The best result is highlighted in gray, and the second-best is underlined.
 }
 \label{tab:ruler_64k_appendix}
 \resizebox{\textwidth}{!}{
 \begin{tabular}{l|c|rrrrrrrrrrr|r}
 \toprule
 \textbf{Method} & Budget & N-S1 & N-S2 & N-S3 & N-MK1 & N-MK2 & N-MQ & N-MV & VT & FWE & QA-1 & QA-2 & Avg. \\ 
 \midrule
 \multicolumn{14}{c}{\textbf{LLaMA-3.1-8B-Instruct}} \\ 
 \midrule
 FullKV & 100\% & 100 & 100 & 100 & 100 & 100 & 100 & 99.02 & 98.91 & 85.68 & 71.22 & 50 & 91.35 \\
 \midrule
 SnapKV & \multirow{8}{*}{5\%} & 100 & 99.40 & 84.6 & 99 & 94.4 & 99.5 & 98.7 & 96.24 & 74.87 & 69.9 & 49.2 & \underline{87.80} \\
 PyramidKV & & 100 & 100.00 & 57.8 & 99.6 & 94.8 & 99.5 & 99.1 & 97.24 & 73.13 & 69.9 & 49.8 & 85.53 \\
 StreamingLLM & & 5.80 & 7.40 & 4.80 & 7.60 & 5.40 & 4.25 & 5.05 & 9.12 & 94.87 & 35.80 & 40.77 & 20.08 \\
 KeyDiff & & 100 & 99.40 & 42.40 & 89.80 & 1.40 & 87.35 & 90.55 & 95.20 & 54.53 & 38 & 45.85 & 67.68 \\
 CAKE & & 100 & 99.00 & 82.4 & 98.6 & 90.6 & 98.45 & 96.4 & 96.4 & 71.27 & 68.57 & 49 & 86.43 \\
 CriticalKV & & 100 & 100.00 & 76.2 & 99.2 & 95.4 & 99.55 & 99.15 & 96.24 & 72.73 & 69.3 & 49.2 & 87.00 \\
 AnDPro & & 100 & 100.00 & 76.8 & 99.8 & 97.6 & 99.85 & 98.7 & 96.92 & 74.53 & 69.67 & 49.4 & 87.57 \\
 \rowcolor{gray!30} DropKV & & 100 & 100.00 & 96.8 & 99.6 & 99 & 99.45 & 98.35 & 96.2 & 71.67 & 69.2 & 50.4 & \textbf{89.15}\\
 \midrule
 SnapKV & \multirow{8}{*}{10\%} & 100 & 100 & 99.22 & 99.22 & 98.44 & 99.61 & 99.22 & 97.03 & 80.02 & 70.44 & 46.09 & 89.94 \\
 PyramidKV & & 100 & 100 & 91.41 & 100 & 96.88 & 99.61 & 99.81 & 97.5 & 78.91 & 72.79 & 50.78 & 89.79 \\
 StreamingLLM & & 13.28 & 12.5 & 7.81 & 10.16 & 17.19 & 8.98 & 10.55 & 15.16 & 96.35 & 44.08 & 37.5 & 24.87 \\
 KeyDiff & & 100 & 100 & 95.31 & 97.66 & 3.91 & 99.22 & 98.63 & 95.47 & 64.58 & 50.59 & 43.75 & 77.19 \\
 CAKE & & 100 & 99.22 & 99.22 & 99.22 & 96.09 & 99.81 & 98.05 & 97.03 & 77.6 & 70.96 & 49.22 & 89.67 \\
 CriticalKV & & 100 & 100 & 96.09 & 100 & 99.22 & 99.81 & 99.61 & 96.88 & 80.21 & 70.44 & 47.66 & 89.99 \\
 AnDPro & & 100 & 100 & 96.09 & 100 & 98.44 & 100 & 99.41 & 97.18 & 78.85 & 71.74 & 48.44 & \underline{90.01} \\
 \rowcolor{gray!30} DropKV & & 100 & 100 & 100 & 100 & 100 & 100 & 99.22 & 97.19 & 77.87 & 72.01 & 46.88 & \textbf{90.29} \\
 \midrule
 SnapKV & \multirow{8}{*}{30\%} & 100 & 100 & 100 & 100 & 100 & 100 & 99.41 & 98.13 & 85.94 & 70.44 & 48.44 & 91.12 \\
 PyramidKV & & 100 & 100 & 100 & 100 & 99.22 & 100 & 98.83 & 98.44 & 84.38 & 71.48 & 48.44 & 90.98 \\
 StreamingLLM & & 34.38 & 35.17 & 31.25 & 33.59 & 34.38 & 30.47 & 31.45 & 36.72 & 95.05 & 39.13 & 38.28 & 39.99 \\
 KeyDiff & & 100 & 100 & 100 & 100 & 29.69 & 100 & 99.61 & 97.34 & 86.2 & 53.52 & 43.75 & 82.74 \\
 CAKE & & 100 & 100 & 99.22 & 100 & 99.22 & 99.81 & 99.61 & 98.12 & 86.2 & 70.44 & 47.66 & 90.93 \\
 CriticalKV & & 100 & 100 & 100 & 100 & 100 & 100 & 98.63 & 97.97 & 85.42 & 71.22 & 46.88 & 90.92 \\
 AnDPro & & 100 & 100 & 100 & 100 & 100 & 100 & 99.81 & 98.12 & 86.58 & 70.7 & 47.66 & \underline{91.17} \\
 \cellcolor{gray!30}DropKV & \cellcolor{gray!30} & \cellcolor{gray!30} 100 & \cellcolor{gray!30} 100 & \cellcolor{gray!30} 100 & \cellcolor{gray!30} 100 & \cellcolor{gray!30} 100 & \cellcolor{gray!30} 100 & \cellcolor{gray!30} 99.02 & \cellcolor{gray!30} 98.28 & \cellcolor{gray!30} 86.46 & \cellcolor{gray!30} 71.75 & \cellcolor{gray!30} 48.09 & \cellcolor{gray!30} \textbf{91.24} \\
 \midrule
 \multicolumn{14}{c}{\textbf{Qwen2.5-7B-Instruct-1M}} \\ 
 \midrule
 FullKV & 100\% & 100 & 100 & 100 & 100 & 99.22 & 100 & 80.86 & 94.53 & 83.33 & 60.29 & 57.81 & 88.73 \\
 \midrule
 SnapKV & \multirow{8}{*}{5\%} & 100 & 100 & 10.2 & 100 & 83.4 & 99.9 & 77.8 & 94.84 & 66.87 & 55.38 & 52.2 & 76.42 \\
 PyramidKV & & 100 & 100 & 1 & 99.8 & 73 & 99.8 & 77.5 & 92.4 & 66.67 & 55.57 & 51.6 & 74.30 \\
 StreamingLLM & & 5.80 & 5.60 & 4.40 & 3.80 & 4.80 & 4.85 & 4.75 & 9.32 & 81.13 & 32.40 & 29.35 & 16.93 \\
 KeyDiff & & 100 & 99.80 & 91 & 78.20 & 0.80 & 65.45 & 63.45 & 92.68 & 66.13 & 23.40 & 21.75 & 63.88 \\
 CAKE & & 100 & 100 & 14 & 99.6 & 19 & 99.45 & 75.7 & 94.72 & 67.27 & 55 & 52 & 70.61 \\
 CriticalKV & & 100 & 100 & 12.6 & 100 & 88.2 & 99.9 & 77.55 & 94.84 & 66.73 & 55.38 & 52 & 77.02 \\
 AnDPro & & 100 & 99.8 & 39.2 & 100 & 89.4 & 99.8 & 78.7 & 95.36 & 67 & 55.33 & 52.4 & \underline{79.73} \\
 \rowcolor{gray!30} DropKV & & 100 & 99.6 & 74.4 & 99.6 & 98.4 & 99.85 & 78.95 & 95.36 & 67.13 & 55.73 & 51.8 & \textbf{83.71} \\
 \midrule
 SnapKV & \multirow{8}{*}{10\%} & 100 & 100 & 28.13 & 100 & 88.28 & 100 & 79.1 & 93.75 & 67.19 & 58.2 & 57.03 & 79.24 \\
 PyramidKV & & 100 & 100 & 3.91 & 100 & 80.47 & 100 & 79.1 & 94.06 & 66.67 & 58.14 & 56.25 & 76.24 \\
 StreamingLLM & & 9.38 & 7.81 & 5.47 & 7.03 & 9.38 & 8.4 & 10.55 & 19.37 & 83.07 & 26.89 & 35.16 & 20.23 \\
 KeyDiff & & 100 & 100 & 100 & 93.75 & 1.56 & 91.8 & 77.73 & 95.47 & 66.67 & 23.76 & 31.25 & 71.09 \\
 CAKE & & 100 & 100 & 48.44 & 100 & 32.81 & 100 & 77.92 & 94.69 & 67.19 & 57.68 & 56.25 & 75.91 \\
 CriticalKV & & 100 & 100 & 42.97 & 100 & 91.41 & 100 & 77.93 & 94.06 & 67.19 & 56.84 & 58.59 & 80.82 \\
 AnDPro & & 100 & 99.22 & 64.84 & 100 & 94.53 & 99.81 & 78.52 & 95.16 & 67.19 & 57.42 & 56.25 & \underline{82.99} \\
 \rowcolor{gray!30} DropKV & & 100 & 99.22 & 85.94 & 100 & 98.44 & 100 & 79.1 & 94.69 & 68.49 & 58.72 & 56.25 & \textbf{85.53} \\
 \midrule
 SnapKV & \multirow{8}{*}{30\%} & 100 & 99.22 & 83.59 & 100 & 96.09 & 100 & 78.32 & 94.06 & 72.14 & 58.2 & 57.03 & 85.33 \\
 PyramidKV & & 100 & 100 & 39.84 & 100 & 92.97 & 99.81 & 80.27 & 94.69 & 69.79 & 58.46 & 55.47 & 81.03 \\
 StreamingLLM & & 29.69 & 25 & 24.22 & 30.47 & 30.47 & 25.78 & 34.77 & 47.03 & 85.68 & 26.3 & 35.16 & 35.87 \\
 KeyDiff & & 100 & 99.22 & 100 & 100 & 16.41 & 100 & 82.62 & 94.84 & 69.53 & 34.96 & 46.09 & 76.70 \\
 CAKE & & 100 & 100 & 95.31 & 100 & 70.31 & 99.61 & 79.69 & 93.75 & 72.14 & 57.68 & 57.81 & 84.21 \\
 CriticalKV & & 100 & 100 & 96.09 & 100 & 97.66 & 99.81 & 78.91 & 94.69 & 74.48 & 57.88 & 57.03 & \underline{86.96} \\
 AnDPro & & 100 & 100 & 91.41 & 100 & 96.09 & 99.61 & 79.69 & 94.53 & 73.44 & 58.2 & 56.25 & 86.29 \\
 \rowcolor{gray!30} DropKV & & 100 & 100 & 100 & 100 & 100 & 99.41 & 79.1 & 94.69 & 75.26 & 58.72 & 57.03 & \textbf{87.66} \\
 \midrule
 \multicolumn{14}{c}{\textbf{Mistral-7B-512K}} \\ 
 \midrule
 FullKV & 100\% & 100 & 100 & 100 & 99.22 & 99.22 & 99.81 & 75 & 96.41 & 79.43 & 70.44 & 53.91 & 88.49 \\
 \midrule
 SnapKV & \multirow{8}{*}{5\%} & 100 & 99.4 & 48 & 99.2 & 67.2 & 98.45 & 75.1 & 96.56 & 74.67 & 70.97 & 50.8 & 80.03 \\
 PyramidKV & & 100 & 99.6 & 30.4 & 99.6 & 66.4 & 98.65 & 76.85 & 96.2 & 73.53 & 70.83 & 50.6 & 78.42 \\
 StreamingLLM & & 6.60 & 4.40 & 6.80 & 7.20 & 5.80 & 8.75 & 6.15 & 7.60 & 92.60 & 29.80 & 35.83 & 19.23 \\
 KeyDiff & & 100 & 78.80 & 23.40 & 58.80 & 0.60 & 34.35 & 35.10 & 95.76 & 69.93 & 22.40 & 33.20 & 50.21 \\
 CAKE & & 100 & 99.6 & 66.2 & 99 & 50.8 & 99 & 75.95 & 95.32 & 74.87 & 70.17 & 50.4 & 80.12 \\
 CriticalKV & & 100 & 99.6 & 44 & 98.8 & 65.6 & 98.45 & 75.9 & 95.6 & 73.67 & 70.63 & 51.4 & 79.42 \\
 AnDPro & & 100 & 99.4 & 64.2 & 99.8 & 82.4 & 99.2 & 78.15 & 96.16 & 74.73 & 71.23 & 53.2 & \underline{83.50} \\
 \rowcolor{gray!30} DropKV & & 100 & 99.2 & 69.2 & 99.4 & 93.8 & 97.75 & 77.4 & 96.28 & 73.93 & 69.77 & 51.4 & 84.38\\
 \midrule
 SnapKV & \multirow{8}{*}{10\%} & 100 & 100 & 77.34 & 99.22 & 63.28 & 99.02 & 77.15 & 96.41 & 74.48 & 71.22 & 50 & 82.56 \\
 PyramidKV & & 100 & 100 & 71.09 & 100 & 64.84 & 99.81 & 77.15 & 96.41 & 73.44 & 73.5 & 51.56 & 82.53 \\
 StreamingLLM & & 6.25 & 11.72 & 11.72 & 9.38 & 10.16 & 13.28 & 12.31 & 12.81 & 94.27 & 31.71 & 32.03 & 22.33 \\
 KeyDiff & & 100 & 91.41 & 75 & 82.81 & 4.69 & 73.83 & 58.59 & 96.09 & 70.05 & 38.02 & 31.25 & 65.61 \\
 CAKE & & 100 & 100 & 94.53 & 99.22 & 57.03 & 99.61 & 76.38 & 96.56 & 75 & 71.48 & 50.78 & 83.69 \\
 CriticalKV & & 100 & 98.44 & 78.91 & 98.44 & 63.28 & 99.41 & 77.54 & 95.16 & 73.7 & 72.01 & 51.56 & 82.59 \\
 AnDPro & & 100 & 100 & 91.41 & 99.22 & 76.56 & 99.81 & 76.17 & 96.25 & 74.22 & 72 & 52.34 & \underline{85.27} \\
 \rowcolor{gray!30} DropKV & & 100 & 99.22 & 89.94 & 99.22 & 92.97 & 98.83 & 78.13 & 96.88 & 73.7 & 72.01 & 53.13 & \textbf{86.73} \\
 \midrule
 SnapKV & \multirow{8}{*}{30\%} & 100 & 100 & 100 & 99.22 & 63.91 & 100 & 75 & 96.09 & 76.04 & 71.22 & 53.13 & 84.96 \\
 PyramidKV & & 100 & 100 & 98.44 & 99.22 & 66.41 & 100 & 74.81 & 96.25 & 74.22 & 72.2 & 53.91 & 85.04 \\
 StreamingLLM & & 28.91 & 35.16 & 26.56 & 26.56 & 26.56 & 31.64 & 28.52 & 35.31 & 93.49 & 31.58 & 30.47 & 35.89 \\
 KeyDiff & & 100 & 100 & 100 & 96.09 & 21.09 & 96.48 & 72.07 & 96.56 & 73.18 & 53.39 & 38.28 & 77.01 \\
 CAKE & & 100 & 100 & 100 & 99.22 & 48.44 & 99.81 & 75 & 96.41 & 76.82 & 70.7 & 51.56 & 83.45 \\
 CriticalKV & & 100 & 100 & 100 & 99.22 & 64.06 & 99.81 & 75 & 94.53 & 77.08 & 71.48 & 53.13 & \underline{84.94} \\
 AnDPro & & 100 & 100 & 100 & 99.22 & 61.72 & 99.61 & 75.78 & 96.41 & 76.04 & 71.22 & 53.91 & 84.90 \\
 \rowcolor{gray!30} DropKV & & 100 & 100 & 100 & 99.22 & 95.31 & 100 & 75.39 & 96.09 & 76.3 & 71.22 & 53.13 & \textbf{87.88} \\
 \bottomrule
 \end{tabular}

 }
\end{table}
 
\begin{table}[H]
 \centering
 \renewcommand{\arraystretch}{0.88}
 \caption{RULER results at 32K sequence length and 5\%, 10\% and 30\% KV-cache budget. The best result is highlighted in gray, and the second-best is underlined.
 }
 \label{tab:ruler_32k_appendix}
 \resizebox{\textwidth}{!}{
 \begin{tabular}{l|c|rrrrrrrrrrr|r}
 \toprule
 \textbf{Method} & Budget & N-S1 & N-S2 & N-S3 & N-MK1 & N-MK2 & N-MQ & N-MV & VT & FWE & QA-1 & QA-2 & Avg. \\ 
 \midrule
 \multicolumn{14}{c}{\textbf{LLaMA-3.1-8B-Instruct}} \\ 
 \midrule
 Full KV & 100\% & 100 & 100 & 100 & 100 & 100 & 99.61 & 98.24 & 99.84 & 95.57 & 71.03 & 53.91 & 92.56 \\
 \midrule
 SnapKV & \multirow{8}{*}{5\%} & 100 & 99.80 & 48.6 & 99.4 & 97.4 & 99.75 & 98.35 & 98.84 & 81.27 & 68.15 & 55.4 & 86.09 \\
 PyramidKV & & 100 & 99.80 & 29 & 100 & 97.4 & 99.75 & 98.2 & 99.24 & 82.6 & 68.85 & 55.6 & 84.59 \\
 StreamingLLM & & 5 & 6 & 4.40 & 5.40 & 6.20 & 5.15 & 5.15 & 8.96 & 92.80 & 43.20 & 42.68 & 20.45 \\
 KeyDiff & & 100 & 100 & 61.60 & 95 & 2.80 & 91.25 & 91.95 & 98.96 & 65.33 & 42 & 45.50 & 72.22 \\
 CAKE & & 100 & 99.60 & 53 & 99.4 & 97.2 & 99.6 & 98.15 & 99.2 & 80.2 & 69.88 & 55.4 & \underline{86.51} \\
 CriticalKV & & 100 & 99.80 & 43 & 100 & 97.4 & 99.75 & 97.75 & 99.04 & 76.93 & 68.68 & 55 & 85.21 \\
 AnDPro & & 100 & 100.00 & 44.4 & 100 & 99.4 & 99.9 & 98.2 & 99.04 & 76.6 & 68.82 & 54.8 & 85.56 \\
 \rowcolor{gray!30} DropKV & & 100 & 100.00 & 84.8 & 100 & 99.6 & 99.6 & 96.8 & 98.72 & 79.33 & 69.25 & 56.8 & \textbf{89.54} \\
 \midrule
 SnapKV & \multirow{8}{*}{10\%} & 100 & 100 & 91.75 & 100 & 100 & 99.61 & 97.83 & 99.38 & 86.98 & 70.68 & 53.13 & \underline{90.85} \\
 PyramidKV & & 100 & 100 & 71.88 & 100 & 100 & 99.81 & 99.02 & 99.53 & 86.72 & 71.29 & 52.34 & 89.14 \\
 StreamingLLM & & 14.84 & 10.16 & 13.28 & 13.28 & 9.38 & 10.35 & 10.16 & 17.66 & 95.83 & 44.6 & 42.97 & 25.68 \\
 KeyDiff & & 100 & 100 & 94.53 & 100 & 8.59 & 99.61 & 96.29 & 99.53 & 72.92 & 47.59 & 50 & 79.01 \\
 CAKE & & 100 & 99.22 & 92.18 & 98.44 & 100 & 99.81 & 98.61 & 99.53 & 84.38 & 69.6 & 54.68 & 90.59 \\
 CriticalKV & & 100 & 100 & 89.06 & 100 & 100 & 99.61 & 98.24 & 99.38 & 81.77 & 70.38 & 51.56 & 90.00 \\
 AnDPro & & 100 & 100 & 89.84 & 100 & 100 & 99.81 & 98.61 & 99.68 & 85.15 & 69.6 & 53.91 & 90.60 \\
 \rowcolor{gray!30} DropKV & & 100 & 100 & 99.22 & 100 & 100 & 99.81 & 97.66 & 99.38 & 82.29 & 69.6 & 53.91 & \textbf{91.08} \\
 \midrule
 SnapKV & \multirow{8}{*}{30\%} & 100 & 100 & 100 & 100 & 100 & 99.61 & 98.44 & 99.53 & 91.41 & 70.64 & 53.13 & \underline{92.07} \\
 PyramidKV & & 100 & 100 & 100 & 100 & 100 & 99.61 & 98.44 & 99.53 & 90.63 & 70.51 & 53.13 & 91.99 \\
 StreamingLLM & & 30.47 & 37.5 & 35.16 & 35.16 & 34.38 & 31.64 & 33.59 & 41.72 & 96.09 & 41.99 & 45.31 & 42.09 \\
 KeyDiff & & 100 & 100 & 100 & 100 & 38.28 & 99.22 & 98.05 & 99.53 & 87.76 & 61.85 & 54.69 & 85.04 \\
 CAKE & & 100 & 99.22 & 100 & 98.44 & 100 & 99.81 & 99.22 & 99.53 & 91.67 & 69.2 & 52.34 & 91.77 \\
 CriticalKV & & 100 & 100 & 100 & 100 & 100 & 99.41 & 98.44 & 99.53 & 90.63 & 69.88 & 53.91 & 91.98 \\
 AnDPro & & 100 & 100 & 100 & 100 & 100 & 99.61 & 98.05 & 99.53 & 90.67 & 70 & 53.91 & 91.98 \\
 \rowcolor{gray!30} DropKV & & 100 & 100 & 100 & 100 & 100 & 99.41 & 98.63 & 99.53 & 90.76 & 70 & 54.69 & \textbf{92.09} \\
 \midrule
 \multicolumn{14}{c}{\textbf{Qwen2.5-7B-Instruct-1M}} \\ 
 \midrule
 Full KV & 100\% & 100 & 100 & 100 & 100 & 100 & 100 & 94.34 & 97.34 & 95.57 & 62.63 & 55.47 & 91.40 \\
 \midrule
 SnapKV & \multirow{8}{*}{5\%} & 100 & 100 & 2.8 & 100 & 82 & 99.9 & 86.45 & 97.92 & 89.33 & 62.85 & 55.8 & 79.73 \\
 PyramidKV & & 100 & 100 & 0.6 & 100 & 62.2 & 98.85 & 83.1 & 87.76 & 76.67 & 61.65 & 54.2 & 75.00 \\
 StreamingLLM & & 5 & 4.80 & 5 & 5.40 & 3 & 4.55 & 5.65 & 8.96 & 87.80 & 38.60 & 32.67 & 18.31 \\
 KeyDiff & & 100 & 99.80 & 98.40 & 86.40 & 0.20 & 73.15 & 68.40 & 94.60 & 86.07 & 24 & 30.53 & 69.23 \\
 CAKE & & 100 & 99.8 & 3.6 & 99.4 & 15.2 & 97.65 & 79.3 & 98 & 89.13 & 62.73 & 55.8 & 72.78 \\
 CriticalKV & & 100 & 100 & 4 & 100 & 88.6 & 99.9 & 86.6 & 97.92 & 90 & 62.78 & 55.6 & 80.49 \\
 AnDPro & & 99.8 & 99.8 & 20.4 & 100 & 87 & 99.95 & 89.6 & 98.44 & 88.2 & 63.05 & 55.8 & \underline{82.00} \\
 \rowcolor{gray!30} DropKV & & 100 & 99.6 & 63.8 & 100 & 99 & 99.8 & 90 & 98.2 & 90.8 & 62.95 & 56.6 & \textbf{87.34} \\ 
 \midrule
 SnapKV & \multirow{8}{*}{10\%} & 100 & 100 & 14.06 & 100 & 92.19 & 100 & 87.11 & 97.81 & 91.93 & 62.89 & 53.91 & 81.81 \\
 PyramidKV & & 100 & 100 & 1.56 & 100 & 82.03 & 100 & 88.09 & 94.84 & 88.8 & 62.63 & 55.47 & 79.40 \\
 StreamingLLM & & 14.06 & 14.06 & 10.94 & 9.38 & 9.38 & 10.55 & 12.31 & 15.58 & 90.63 & 31.58 & 38.28 & 23.34 \\
 KeyDiff & & 100 & 99.22 & 30.47 & 98.44 & 3.13 & 95.7 & 84.96 & 97.81 & 90.63 & 37.57 & 32.03 & 70.00 \\
 CAKE & & 100 & 100 & 20.31 & 100 & 28.91 & 100 & 87.5 & 97.97 & 92.19 & 62.37 & 55.47 & 76.79 \\
 CriticalKV & & 100 & 100 & 23.44 & 100 & 95.31 & 100 & 87.11 & 98.13 & 93.49 & 62.83 & 54.69 & 83.18 \\
 AnDPro & & 100 & 100 & 64.84 & 100 & 96.88 & 100 & 89.84 & 98.59 & 90.89 & 64.97 & 54.68 & \underline{87.34} \\
 \rowcolor{gray!30} DropKV & & 100 & 100 & 85.94 & 100 & 100 & 100 & 90.82 & 98.44 & 93.75 & 63.93 & 56.25 & \textbf{89.92} \\
 \midrule
 SnapKV & \multirow{8}{*}{30\%} & 100 & 100 & 71.09 & 100 & 97.66 & 100 & 88.09 & 97.5 & 95.57 & 62.37 & 55.47 & 87.98 \\
 PyramidKV & & 100 & 100 & 15.63 & 100 & 92.19 & 100 & 91.99 & 97.81 & 93.75 & 62.63 & 56.25 & 82.75 \\
 StreamingLLM & & 25 & 35.94 & 31.25 & 39.06 & 36.72 & 38.09 & 35.35 & 43.13 & 91.15 & 27.6 & 42.19 & 40.50 \\
 KeyDiff & & 100 & 100 & 100 & 100 & 26.56 & 99.81 & 92.38 & 96.88 & 96.35 & 53.52 & 49.22 & 83.16 \\
 CAKE & & 100 & 100 & 89.84 & 100 & 67.97 & 100 & 90.82 & 97.19 & 94.79 & 62.37 & 55.47 & 87.13 \\
 CriticalKV & & 100 & 100 & 87.5 & 100 & 97.66 & 100 & 90.82 & 97.97 & 95.57 & 64.19 & 55.47 & 89.93 \\
 AnDPro & & 100 & 100 & 95.31 & 100 & 100 & 100 & 89.45 & 97.97 & 93.23 & 62.13 & 55.25 & \underline{90.30} \\
 \rowcolor{gray!30} DropKV & & 100 & 100 & 100 & 100 & 100 & 100 & 91.21 & 97.81 & 95.31 & 62.63 & 55.47 & \textbf{91.13} \\
 \midrule
 \multicolumn{14}{c}{\textbf{Mistral-7B-512K}} \\ 
 \midrule
 Full KV & 100\%  & 100 & 100 & 100 & 100 & 100 & 99.81 & 78.32 & 98.44 & 96.09 & 77.54 & 49.22 & 90.86 \\
 \midrule
 SnapKV & \multirow{8}{*}{5\%} & 100 & 100 & 25 & 99.6 & 87.6 & 99.05 & 77.05 & 97.68 & 86.4 & 72.9 & 52 & 81.57 \\
 PyramidKV & & 100 & 100 & 13.2 & 99.4 & 86.4 & 99.15 & 78.05 & 96.72 & 86.87 & 72.3 & 51 & 80.28 \\
 StreamingLLM & & 4.60 & 6.60 & 5.60 & 6.40 & 5.40 & 5.80 & 7 & 6.56 & 83.40 & 34.40 & 37.32 & 18.46 \\
 KeyDiff & & 100 & 75.40 & 10 & 49.60 & 0.80 & 28.50 & 31.45 & 97.16 & 74 & 23 & 38.77 & 48.06 \\
 CAKE & & 100 & 100 & 47.7 & 99.4 & 72.8 & 99.1 & 76.8 & 97.2 & 87.27 & 73.43 & 52 & 82.34 \\
 CriticalKV & & 100 & 100 & 20.6 & 99.4 & 86 & 99.05 & 77.3 & 97.44 & 85.4 & 72.9 & 51.2 & 80.84 \\
 AnDPro & & 100 & 99.6 & 36.6 & 99.6 & 95.4 & 98.45 & 81.05 & 97.48 & 86.88 & 73.77 & 51.4 & \underline{83.86} \\
 \rowcolor{gray!30} DropKV & & 100 & 99.8 & 46.8 & 99.6 & 99.4 & 97.68 & 79.45 & 97.72 & 87.6 & 74.17 & 50.2 & \textbf{84.77}\\
 \midrule
 SnapKV & \multirow{8}{*}{10\%} & 100 & 100 & 57.81 & 100 & 89.84 & 99.61 & 79.3 & 98.13 & 89.06 & 75.13 & 47.66 & 85.14 \\
 PyramidKV & & 100 & 100 & 50 & 100 & 87.5 & 99.61 & 80.08 & 98.28 & 87.76 & 76.38 & 47.66 & 84.30 \\
 StreamingLLM & & 6.25 & 15.63 & 14.84 & 14.84 & 11.72 & 13.28 & 14.06 & 14.22 & 92.19 & 29.82 & 32.81 & 23.61 \\
 KeyDiff & & 100 & 92.97 & 55.47 & 78.91 & 3.13 & 72.66 & 55.47 & 97.66 & 80.73 & 43.1 & 25.78 & 64.17 \\
 CAKE & & 100 & 100 & 90.63 & 100 & 77.34 & 99.81 & 78.91 & 97.66 & 91.15 & 76.69 & 48.44 & \underline{87.33} \\
 CriticalKV & & 100 & 100 & 65.63 & 100 & 87.5 & 99.61 & 80.27 & 97.66 & 88.8 & 76.17 & 47.66 & 85.75 \\
 AnDPro & & 100 & 100 & 73.44 & 100 & 92.18 & 99.61 & 81.05 & 97.66 & 91.41 & 74.61 & 47.66 & 87.03 \\
 \rowcolor{gray!30} DropKV & & 100 & 100 & 79.69 & 100 & 99.22 & 98.83 & 81.25 & 97.81 & 91.41 & 76.43 & 46.88 & \textbf{88.32} \\
 \midrule
 SnapKV & \multirow{8}{*}{30\%} & 100 & 100 & 98.44 & 100 & 85.94 & 99.81 & 79.88 & 98.44 & 92.45 & 76.95 & 48.44 & 89.12 \\
 PyramidKV & & 100 & 100 & 97.66 & 100 & 90.63 & 99.81 & 79.88 & 98.91 & 91.15 & 77.47 & 49.22 & 89.52 \\
 StreamingLLM & & 23.44 & 35.16 & 33.59 & 35.16 & 33.59 & 30.67 & 33.01 & 40.78 & 89.06 & 30.86 & 35.94 & 38.30 \\
 KeyDiff & & 100 & 99.22 & 99.22 & 96.09 & 20.31 & 96.29 & 72.85 & 97.66 & 88.8 & 63.28 & 38.28 & 79.27 \\
 CAKE & & 100 & 100 & 100 & 100 & 82.81 & 99.81 & 78.13 & 98.44 & 94.53 & 77.54 & 49.22 & 89.13 \\
 CriticalKV & & 100 & 100 & 99.22 & 100 & 89.6 & 99.81 & 80.27 & 98.13 & 94.27 & 76.95 & 49.22 & \underline{89.77} \\
 AnDPro & & 100 & 100 & 99.22 & 100 & 89.06 & 99.81 & 80.27 & 98.44 & 94.01 & 76.95 & 49.22 & 89.73 \\
 \rowcolor{gray!30} DropKV & & 100 & 100 & 100 & 100 & 100 & 99.61 & 79.88 & 98.44 & 94.27 & 77.28 & 48.44 & \textbf{90.72} \\
 \bottomrule
 \end{tabular}

 }
\end{table}
 
\begin{table}[H]
 \centering
 \renewcommand{\arraystretch}{0.88}
 \caption{RULER results at 16K sequence length and 5\%, 10\% and 30\% KV-cache budget. The best result is highlighted in gray, and the second-best is underlined.
 }
 \label{tab:ruler_16k_appendix}
 \resizebox{\textwidth}{!}{
 \begin{tabular}{l|c|rrrrrrrrrrr|r}
 \toprule
 \textbf{Method} & Budget & N-S1 & N-S2 & N-S3 & N-MK1 & N-MK2 & N-MQ & N-MV & VT & FWE & QA-1 & QA-2 & Avg. \\ 
 \midrule
 \multicolumn{14}{c}{\textbf{LLaMA-3.1-8B-Instruct}} \\ 
 \midrule
 Full KV & 100\% & 100 & 100 & 100 & 100 & 100 & 100 & 98.05 & 99.84 & 96.35 & 76.3 & 59.38 & 93.63 \\
 \midrule
 SnapKV & \multirow{8}{*}{5\%} & 100 & 100.00 & 19.92 & 100 & 98.44 & 99.9 & 97.56 & 99.06 & 91.28 & 71.62 & 58.59 & 85.12 \\
 PyramidKV & & 100 & 100.00 & 7.81 & 100 & 96.88 & 99.61 & 98.63 & 99.14 & 91.67 & 67.45 & 58.2 & 83.58 \\
 StreamingLLM & & 3.80 & 6.20 & 5 & 4.80 & 5.40 & 4 & 5.95 & 8.04 & 91.73 & 43.40 & 44.12 & 20.22 \\
 KeyDiff & & 100 & 100 & 49.60 & 93.60 & 3 & 89.15 & 89.20 & 98.88 & 83.67 & 44.80 & 45.25 & 72.47 \\
 CAKE & & 100 & 100.00 & 21.88 & 99.61 & 97.27 & 99.71 & 97.66 & 99.22 & 91.15 & 67.77 & 58.98 & 84.84 \\
 CriticalKV & & 100 & 100.00 & 19.53 & 100 & 100 & 99.81 & 97.36 & 99.38 & 90.1 & 68.88 & 57.81 & 84.81 \\
 AnDPro & & 100 & 100.00 & 24.61 & 100 & 98.82 & 100 & 98.44 & 98.98 & 91.28 & 68.95 & 57.81 & \underline{85.35} \\
 \rowcolor{gray!30} DropKV & & 100 & 100.00 & 49.61 & 100 & 99.61 & 99.02 & 97.36 & 98.83 & 93.1 & 69.14 & 57.81 & \textbf{87.68} \\
 \midrule
 SnapKV & \multirow{8}{*}{10\%} & 100 & 100 & 71.09 & 100 & 99.22 & 100 & 98.44 & 99.53 & 94.53 & 76.04 & 59.38 & 90.75 \\
 PyramidKV & & 100 & 100 & 38.28 & 100 & 100 & 100 & 99.02 & 99.84 & 95.05 & 74.74 & 59.38 & 87.85 \\
 StreamingLLM & & 6.25 & 11.72 & 14.87 & 6.25 & 11.72 & 8.39 & 9.77 & 14.49 & 92.19 & 45.64 & 43.75 & 24.09 \\
 KeyDiff & & 100 & 100 & 96.09 & 100 & 10.94 & 100 & 96.88 & 99.53 & 85.16 & 51.37 & 49.22 & 80.84 \\
 CAKE & & 100 & 100 & 71.09 & 100 & 99.22 & 99.81 & 99.02 & 99.69 & 92.71 & 75.26 & 59.38 & 90.56 \\
 CriticalKV & & 100 & 100 & 61.72 & 100 & 100 & 100 & 98.24 & 99.69 & 94.53 & 75.26 & 60.15 & 89.96 \\
 AnDPro & & 100 & 100 & 71.88 & 100 & 100 & 100 & 98.63 & 99.69 & 94.53 & 75.85 & 59.38 & \underline{90.91} \\
 \rowcolor{gray!30} DropKV & & 100 & 100 & 81.25 & 100 & 100 & 99.61 & 98.05 & 99.69 & 94.27 & 75.26 & 59.38 & \textbf{91.59} \\
 \midrule
 SnapKV & \multirow{8}{*}{30\%} & 100 & 100 & 100 & 100 & 100 & 100 & 98.41 & 99.84 & 96.41 & 74.48 & 60.14 & 93.57 \\
 PyramidKV & & 100 & 100 & 98.44 & 100 & 100 & 100 & 98.63 & 98.44 & 96.35 & 73.11 & 60.16 & 93.19 \\
 StreamingLLM & & 22.66 & 32.03 & 39.84 & 30.47 & 33.59 & 30.27 & 33.79 & 40.78 & 91.93 & 74.74 & 40.78 & 42.81 \\
 KeyDiff & & 100 & 100 & 100 & 100 & 46.09 & 100 & 98.05 & 99.84 & 94.1 & 59.11 & 57.03 & 86.75 \\
 CAKE & & 100 & 99.22 & 100 & 100 & 99.22 & 100 & 99.22 & 99.84 & 96.35 & 73.96 & 58.59 & 93.31 \\
 CriticalKV & & 100 & 100 & 100 & 100 & 100 & 100 & 98.02 & 99.84 & 96.88 & 74.74 & 59.38 & \underline{93.53} \\
 AnDPro & & 100 & 100 & 100 & 100 & 100 & 100 & 97.74 & 99.84 & 96.03 & 75 & 60.23 & \underline{93.53} \\
 \rowcolor{gray!30} DropKV & & 100 & 100 & 100 & 100 & 100 & 100 & 97.85 & 99.84 & 96.09 & 75.59 & 60.16 & \textbf{93.59} \\
 \midrule
 \multicolumn{14}{c}{\textbf{Qwen2.5-7B-Instruct-1M}} \\ 
 \midrule
 Full KV & 100\% & 100 & 100 & 100 & 100 & 99.22 & 100 & 92.38 & 99.53 & 95.31 & 66.34 & 66.41 & 92.65 \\
 \midrule
 SnapKV & \multirow{8}{*}{5\%} & 100 & 99.61 & 1.17 & 100 & 70.31 & 99.71 & 84.08 & 99.53 & 91.93 & 60.55 & 57.03 & 78.54 \\
 PyramidKV & & 100 & 99.61 & 0 & 100 & 36.72 & 87.99 & 59.77 & 68.52 & 79.95 & 58.59 & 50.78 & 67.54 \\
 StreamingLLM & & 5.60 & 5.40 & 4 & 7 & 4.20 & 5.20 & 4.90 & 7.84 & 85.53 & 40.80 & 33.28 & 18.52 \\
 KeyDiff & & 100 & 99.60 & 93.40 & 70.60 & 0.80 & 58.70 & 56.65 & 96.60 & 83.40 & 18.80 & 24.32 & 63.90 \\
 CAKE & & 100 & 99.61 & 1.5 & 97.66 & 8.59 & 94.04 & 65.53 & 99.3 & 92.58 & 59.9 & 57.42 & 70.56\\
 CriticalKV & & 100 & 99.61 & 2.34 & 100 & 87.5 & 99.51 & 85.45 & 99.45 & 92.06 & 60.55 & 57.42 & 80.35 \\
 AnDPro & & 100 & 99.61 & 11.72 & 100 & 76.18 & 99.9 & 88.67 & 99.22 & 92.58 & 61.78 & 57.03 & \underline{80.61} \\
 \rowcolor{gray!30} DropKV & & 100 & 99.22 & 51.17 & 100 & 100 & 99.81 & 88.28 & 99.45 & 91.41 & 62.89 & 56.64 & \textbf{86.26} \\
 \midrule
 SnapKV & \multirow{8}{*}{10\%} & 100 & 100 & 7.8 & 100 & 85.94 & 99.41 & 87.7 & 99.69 & 92.45 & 66.41 & 62.5 & 81.99 \\
 PyramidKV & & 100 & 100 & 1.56 & 100 & 62.5 & 98.05 & 83.2 & 92.5 & 88.02 & 68.23 & 63.28 & 77.94 \\
 StreamingLLM & & 10.94 & 17.97 & 9.38 & 12.5 & 7.03 & 10.35 & 11.13 & 14.69 & 88.28 & 30.66 & 37.5 & 22.77 \\
 KeyDiff & & 100 & 100 & 100 & 95.31 & 3.91 & 88.48 & 80.47 & 97.97 & 84.9 & 26.63 & 31.25 & 73.54 \\
 CAKE & & 100 & 100 & 12.5 & 100 & 21.09 & 98.83 & 83.59 & 99.84 & 91.93 & 66.67 & 64.06 & 76.23 \\
 CriticalKV & & 100 & 100 & 14.06 & 100 & 91.41 & 100 & 88.87 & 100 & 94.01 & 68.49 & 66.41 & 83.93 \\
 AnDPro & & 100 & 100 & 38.28 & 100 & 85.16 & 100 & 90.04 & 99.84 & 94.01 & 67.19 & 65.63 & \underline{85.47} \\
 \rowcolor{gray!30} DropKV & & 100 & 100 & 88.28 & 100 & 100 & 100 & 89.26 & 99.84 & 93.75 & 68.75 & 64.84 & \textbf{91.34} \\
 \midrule
 SnapKV & \multirow{8}{*}{30\%} & 100 & 100 & 52.34 & 100 & 96.88 & 100 & 88.09 & 99.53 & 96.35 & 66.67 & 65.63 & 87.77 \\
 PyramidKV & & 100 & 100 & 10.16 & 100 & 85.16 & 99.81 & 89.65 & 99.69 & 94.27 & 66.15 & 66.41 & 82.85 \\
 StreamingLLM & & 27.34 & 35.94 & 34.38 & 35.16 & 22.66 & 33.98 & 32.81 & 40.94 & 90.89 & 67.45 & 42.19 & 42.16 \\
 KeyDiff & & 100 & 100 & 100 & 100 & 21.09 & 100 & 92.19 & 99.38 & 90.89 & 46.75 & 50.78 & 81.92 \\
 CAKE & & 100 & 100 & 67.19 & 100 & 55.47 & 100 & 88.87 & 99.84 & 96.09 & 64.78 & 66.41 & 85.33 \\
 CriticalKV & & 100 & 100 & 75 & 100 & 98.44 & 100 & 89.06 & 99.69 & 96.62 & 66.15 & 67.19 & 90.20 \\
 AnDPro & & 100 & 100 & 83.59 & 100 & 96.09 & 100 & 88.28 & 100 & 96.88 & 66.12 & 66.41 & \underline{90.67} \\
 \rowcolor{gray!30} DropKV & & 100 & 100 & 100 & 100 & 100 & 100 & 89.06 & 100 & 96.09 & 66.15 & 64.06 & \textbf{92.31} \\
 \midrule
 \multicolumn{14}{c}{\textbf{Mistral-7B-512K}} \\ 
 \midrule
 Full KV & 100\% & 100 & 100 & 100 & 100 & 100 & 99.61 & 85.35 & 98.28 & 87.24 & 75.13 & 57.03 & 91.15 \\
 \midrule
 SnapKV & \multirow{8}{*}{5\%} & 100 & 100 & 14.45 & 99.22 & 87.11 & 98.14 & 72.27 & 95.94 & 84.51 & 69.6 & 53.13 & 79.49 \\
 PyramidKV & & 100 & 100 & 6.64 & 99.22 & 87.89 & 97.95 & 74.9 & 93.91 & 83.59 & 68 & 50.39 & 78.41 \\
 StreamingLLM & & 4.20 & 5.20 & 5.80 & 4.20 & 4.60 & 4.65 & 2.80 & 6.72 & 84.93 & 33.80 & 39.48 & 17.85 \\
 KeyDiff & & 100 & 75.80 & 2.60 & 43.60 & 0.20 & 24.80 & 25.95 & 96.16 & 70.80 & 21.20 & 31.63 & 44.79 \\
 CAKE & & 99.61 & 99.61 & 26.95 & 98.44 & 80.08 & 98.54 & 74.61 & 95.55 & 84.38 & 68.16 & 52.73 & 79.88 \\
 CriticalKV & & 100 & 99.61 & 12.89 & 99.22 & 85.55 & 97.75 & 75.29 & 95.78 & 82.81 & 67.87 & 52.73 & 79.05 \\
 AnDPro & & 100 & 100 & 24.61 & 99.61 & 96.88 & 98.61 & 85.06 & 95.08 & 86.46 & 69.25 & 52.34 & \underline{82.54} \\
 \rowcolor{gray!30} DropKV & & 100 & 99.61 & 31.64 & 98.83 & 99.61 & 95.51 & 83.69 & 95.86 & 84.25 & 67.38 & 52.73 & \textbf{82.65} \\
 \midrule
 SnapKV & \multirow{8}{*}{10\%} & 100 & 100 & 48.44 & 100 & 90.63 & 99.61 & 83.4 & 98.13 & 84.38 & 75.13 & 53.13 & 84.80 \\
 PyramidKV & & 100 & 100 & 36.72 & 100 & 89.84 & 99.41 & 83.79 & 95.78 & 83.85 & 75.07 & 53.13 & 83.42 \\
 StreamingLLM & & 12.5 & 16.41 & 13.28 & 13.28 & 11.72 & 12.89 & 8.01 & 17.5 & 87.76 & 37.31 & 35.16 & 24.17 \\
 KeyDiff & & 100 & 85.94 & 34.38 & 71.09 & 0.78 & 60.35 & 45.12 & 98.44 & 73.7 & 42.58 & 26.56 & 58.09 \\
 CAKE & & 100 & 100 & 72.66 & 98.44 & 82.03 & 99.81 & 83.98 & 98.12 & 84.64 & 74.09 & 54.69 & 86.22 \\
 CriticalKV & & 100 & 100 & 43.75 & 100 & 92.19 & 99.02 & 85.35 & 97.66 & 83.59 & 73.83 & 53.91 & 84.48 \\
 AnDPro & & 100 & 100 & 62.5 & 100 & 97.66 & 99.61 & 87.11 & 97.5 & 86.2 & 73.56 & 53.12 & \underline{87.02} \\
 \rowcolor{gray!30} DropKV & & 100 & 100 & 64.84 & 100 & 100 & 99.22 & 86.13 & 98.28 & 84.12 & 75.46 & 53.91 & \textbf{87.45} \\
 \midrule
 SnapKV & \multirow{8}{*}{30\%} & 100 & 100 & 95.31 & 100 & 91.41 & 100 & 85.55 & 98.13 & 84.64 & 74.35 & 55.47 & 89.53 \\
 PyramidKV & & 100 & 100 & 91.41 & 100 & 92.19 & 100 & 85.16 & 98.28 & 83.85 & 75.39 & 56.25 & 89.32 \\
 StreamingLLM & & 31.25 & 25 & 31.25 & 32.81 & 34.38 & 33.2 & 25.2 & 44.06 & 86.98 & 31.84 & 40.63 & 37.87 \\
 KeyDiff & & 100 & 100 & 97.66 & 92.97 & 21.09 & 94.53 & 76.56 & 97.97 & 85.42 & 57.55 & 41.41 & 78.65 \\
 CAKE & & 100 & 100 & 99.22 & 100 & 89.06 & 100 & 83.98 & 98.12 & 85.16 & 74.35 & 56.25 & 89.65 \\
 CriticalKV & & 100 & 100 & 95.31 & 100 & 93.75 & 100 & 85.74 & 97.34 & 85.16 & 74.35 & 56.25 & 89.81 \\
 AnDPro & & 100 & 100 & 97.22 & 100 & 96.09 & 99.61 & 85.55 & 97.34 & 85.94 & 74.34 & 56.25 & \underline{90.21} \\
 \rowcolor{gray!30} DropKV & & 100 & 100 & 96.88 & 100 & 100 & 100 & 87.11 & 98.28 & 85.16 & 75.72 & 55.47 & \textbf{90.78} \\
 \bottomrule
 \end{tabular}

 }
\end{table}

\subsection{Extended LongBench Results}\label{app:longbench-extended}

Tables~\ref{tab:longbench_0.9} and~\ref{tab:longbench_0.7} extend the LongBench evaluation in \S\ref{sec:experiments} to $10\%$ and $30\%$ cache budgets.

\begin{table}[!h]
 \centering
 \caption{LongBench results at 10\% cache budget.} 
 \label{tab:longbench_0.9}
 \resizebox{\textwidth}{!}{
 \begin{tabular}{l|cccccccccccccccc|r}
 \toprule
 & \multicolumn{3}{c}{\textbf{Single Doc. QA}} & \multicolumn{3}{c}{\textbf{Multi Doc. QA}} & \multicolumn{3}{c}{\textbf{Summarization}} & \multicolumn{2}{c}{\textbf{Synthetic}} & \multicolumn{3}{c}{\textbf{Fewshot Learning}} & \multicolumn{2}{c}{\textbf{Code}} & \\ 
 \cmidrule(lr){2-4} \cmidrule(lr){5-7} \cmidrule(lr){8-10} \cmidrule(lr){11-12} \cmidrule(lr){13-15} \cmidrule(lr){16-17}
 Method & \rotatebox{45}{NrtvQA} & \rotatebox{45}{Qasper} & \rotatebox{45}{MF-en} & \rotatebox{45}{HotpotQA} & \rotatebox{45}{2WikiMQA} & \rotatebox{45}{Musique} & \rotatebox{45}{GovReport} & \rotatebox{45}{QMSum} & \rotatebox{45}{MultiNews} & \rotatebox{45}{PCount} & \rotatebox{45}{PR-en} & \rotatebox{45}{TREC} & \rotatebox{45}{TriviaQA} & \rotatebox{45}{SAMSum} & \rotatebox{45}{Lcc} & \rotatebox{45}{RB-P} & \textbf{Avg.} \\ 
 \midrule
  
 \multicolumn{18}{c}{\textbf{LLaMA-3.1-8B-Instruct}} \\ 
 \midrule
 FullKV & 30.51 & 45.59 & 56.34 & 58.22 & 50.02 & 31.75 & 34.3 & 25.33 & 26.76 & 9.5 & 100 & 73.5 & 92.29 & 43.96 & 64.77 & 55.7 & 49.91 \\
 \midrule
 SnapKV & 30.11 & 37.66 & 55.52 & 57.27 & 48.37 & 31.51 & 25.77 & 24.24 & 19.55 & 9.5 & 100 & 49.5 & 91.54 & 41.68 & 62.02 & 54.64 & 46.18 \\
 PyramidKV & 30.33 & 38.03 & 55.11 & 57.17 & 48.65 & 31.26 & 25.28 & 24.73 & 19.79 & 9.25 & 100 & 43 & 91.68 & 42.04 & 61.98 & 54.61 & 45.81 \\
 StreamingLLM & 27.82 & 24.8 & 33.13 & 50.57 & 42.3 & 24.17 & 24.06 & 21.69 & 19.06 & 9.5 & 92.5 & 42.5 & 87 & 41.32 & 61.45 & 51.98 & 40.87 \\
 KeyDiff & 30.41 & 33.2 & 44.1 & 53.59 & 41.45 & 28.98 & 25.33 & 23.22 & 20.94 & 9.5 & 74.5 & 44 & 88.38 & 30.44 & 36.82 & 35.82 & 38.79 \\
 CAKE & 29.6 & 40.32 & 56.08 & 57.69 & 49.45 & 31.18 & 26.11 & 24.56 & 20.03 & 9.5 & 100 & 50 & 92.53 & 43.72 & 62.05 & 54.98 & 46.74 \\
 CriticalKV & 30.39 & 38.44 & 55.87 & 57.22 & 48.46 & 31.85 & 26.29 & 24.9 & 20.96 & 9.25 & 100 & 55 & 91.6 & 42.94 & 63.15 & 54.42 & 46.92 \\
 AnDPro & 30.9 & 37.39 & 56.18 & 57.9 & 48.13 & 32.07 & 25.46 & 24.38 & 20.48 & 9.25 & 100 & 56 & 91.97 & 43.77 & 61.74 & 55.64 & \underline{46.95} \\
 \rowcolor{gray!30} DropKV & 31.34 & 40.66 & 57.3 & 57.9 & 48.49 & 31.65 & 26.77 & 24.37 & 21 & 9.5 & 100 & 59 & 91.75 & 43.61 & 62.29 & 54.38 & \textbf{47.50} \\
 \midrule
 \multicolumn{18}{c}{\textbf{Qwen2.5-7B-Instruct-1M}} \\ 
 \midrule
 FullKV & 28.88 & 48.39 & 51.57 & 60.27 & 53.19 & 35.63 & 33.33 & 23.22 & 24.53 & 6.5 & 100 & 66.5 & 85.28 & 45.29 & 45.61 & 35.31 & 46.47\\
 \midrule
 SnapKV & 28.02 & 42.45 & 50.51 & 60.51 & 52.39 & 35.34 & 26.94 & 22.9 & 16.65 & 6.5 & 100 & 44.67 & 85.91 & 44.39 & 42.28 & 35.6 & 43.44 \\
 PyramidKV & 27.25 & 40.68 & 49.42 & 57.89 & 51.61 & 34.66 & 24.32 & 21.95 & 16.57 & 7 & 100 & 35.67 & 84.73 & 44.38 & 41.51 & 32.58 & 41.89 \\
 StreamingLLM & 23.09 & 32.37 & 28.91 & 46.45 & 45.03 & 24.28 & 24.61 & 19.19 & 16.52 & 6 & 24.5 & 38.5 & 82.74 & 43.18 & 39.9 & 34.78 & 33.13 \\
 KeyDiff & 24.73 & 27.54 & 31.99 & 44.68 & 34 & 23.41 & 23.55 & 20.23 & 14.54 & 6 & 53.5 & 32.54 & 67.15 & 23.45 & 23.63 & 20.36 & 29.46 \\
 CAKE & 28.85 & 41.88 & 50.94 & 60.09 & 52.09 & 34.75 & 27.41 & 22.97 & 17.23 & 7 & 100 & 48 & 85.03 & 44.31 & 43.09 & 35.35 & 43.69 \\
 CriticalKV & 27.93 & 42.81 & 50.71 & 60.55 & 52.34 & 35.1 & 27.67 & 22.69 & 17.52 & 7 & 100 & 48.5 & 85.39 & 44.22 & 42.46 & 35.64 & 43.78 \\
 AnDPro & 27.98 & 42.07 & 50.68 & 59.84 & 51.75 & 36.02 & 26.97 & 22.86 & 16.56 & 7 & 100 & 53 & 85.34 & 45.1 & 42.65 & 35.23 & \underline{43.94} \\
 \rowcolor{gray!30} DropKV & 28.48 & 43.43 & 51 & 60.08 & 52.84 & 34.46 & 28.23 & 22.15 & 18.14 & 6.5 & 100 & 59.5 & 85.15 & 44.62 & 43.01 & 35.73 & \textbf{44.58} \\
 \midrule
 \multicolumn{18}{c}{\textbf{Mistral-7B-512k}} \\ 
 \midrule
 FullKV & 22.04 & 16.47 & 42.75 & 24.06 & 20.72 & 13.04 & 33.35 & 25.25 & 26.97 & 0 & 100 & 74.5 & 89.55 & 41.69 & 55.54 & 56.39 & 40.15 \\
 \midrule
 SnapKV & 21.18 & 12.48 & 40.73 & 19.75 & 16.58 & 10.24 & 26.92 & 24.97 & 21.49 & 0 & 100 & 58.5 & 90.3 & 40.4 & 53.55 & 54.6 & 36.98 \\
 PyramidKV & 21.44 & 13.73 & 40.92 & 17.38 & 16.41 & 9.38 & 25.8 & 24.77 & 21.44 & 0 & 100 & 55 & 90.3 & 40.71 & 53.29 & 53.75 & 36.52 \\
 StreamingLLM & 19.85 & 9.35 & 26.55 & 14.24 & 15.21 & 6.83 & 22.28 & 22.11 & 19.18 & 0.5 & 32 & 47 & 88.72 & 39.82 & 51.23 & 51.66 & 29.16 \\
 KeyDiff & 20.64 & 12.21 & 26.77 & 16.57 & 14 & 7.14 & 24.66 & 23.21 & 21.78 & 0 & 16 & 26.5 & 74.6 & 28.22 & 29.6 & 30.2 & 23.26 \\
 CAKE & 21.83 & 13.39 & 40.09 & 20.51 & 17.09 & 11.52 & 27.51 & 25.21 & 22.11 & 0 & 100 & 60.5 & 90.13 & 40.38 & 52.47 & 55.35 & \underline{37.38} \\
 CriticalKV & 22.9 & 12.9 & 40.14 & 20.21 & 16.17 & 10.7 & 27.58 & 25.18 & 22.47 & 0 & 100 & 60 & 90.3 & 40.08 & 53.19 & 54.75 & 37.29 \\
 AnDPro & 21.37 & 12.46 & 41.25 & 20.06 & 16.26 & 9.84 & 26.81 & 25.07 & 21.37 & 0 & 100 & 63.5 & 90.3 & 40.57 & 53.26 & 54.24 & 37.27 \\
 \rowcolor{gray!30} DropKV & 21.62 & 13.49 & 41.16 & 21.75 & 16.02 & 10.5 & 28.11 & 25.11 & 22.85 & 0 & 100 & 69 & 90.3 & 40.55 & 53.44 & 54.11 & \textbf{38.00} \\
 \bottomrule
 \end{tabular}

 }
\end{table}
 
\begin{table}[H]
 \centering
 \caption{Additional LongBench results at 30\% budget. The best average result is highlighted in gray, and the second-best is underlined.}
 \label{tab:longbench_0.7}
 \resizebox{\textwidth}{!}{
 \begin{tabular}{l|cccccccccccccccc|r}
 \toprule
 & \multicolumn{3}{c}{\textbf{Single Doc. QA}} & \multicolumn{3}{c}{\textbf{Multi Doc. QA}} & \multicolumn{3}{c}{\textbf{Summarization}} & \multicolumn{2}{c}{\textbf{Synthetic}} & \multicolumn{3}{c}{\textbf{Fewshot Learning}} & \multicolumn{2}{c}{\textbf{Code}} & \\ 
 \cmidrule(lr){2-4} \cmidrule(lr){5-7} \cmidrule(lr){8-10} \cmidrule(lr){11-12} \cmidrule(lr){13-15} \cmidrule(lr){16-17}
 Method & \rotatebox{45}{NrtvQA} & \rotatebox{45}{Qasper} & \rotatebox{45}{MF-en} & \rotatebox{45}{HotpotQA} & \rotatebox{45}{2WikiMQA} & \rotatebox{45}{Musique} & \rotatebox{45}{GovReport} & \rotatebox{45}{QMSum} & \rotatebox{45}{MultiNews} & \rotatebox{45}{PCount} & \rotatebox{45}{PR-en} & \rotatebox{45}{TREC} & \rotatebox{45}{TriviaQA} & \rotatebox{45}{SAMSum} & \rotatebox{45}{Lcc} & \rotatebox{45}{RB-P} & \textbf{Avg.} \\ 
 \midrule
  
 \multicolumn{18}{c}{\textbf{LLaMA-3.1-8B-Instruct}} \\ 
 \midrule
 FullKV & 30.51 & 45.59 & 56.34 & 58.22 & 50.02 & 31.75 & 34.3 & 25.33 & 26.76 & 9.5 & 100 & 73.5 & 92.29 & 43.96 & 64.77 & 55.7 & 49.91 \\
 \midrule
 SnapKV & 30.28 & 45.93 & 56.13 & 58.05 & 50.38 & 32.07 & 30.58 & 24.96 & 23.74 & 9.5 & 100 & 65.5 & 91.87 & 42.71 & 64.83 & 55.95 & 48.91 \\
 PyramidKV & 30.58 & 44.36 & 55.8 & 57.98 & 49.83 & 31.7 & 30.06 & 24.99 & 23.53 & 9.5 & 100 & 62.5 & 91.88 & 43.42 & 64.37 & 55.68 & 48.51 \\
 StreamingLLM & 28.33 & 31.55 & 37.18 & 54.34 & 48.13 & 26.98 & 28.64 & 22.44 & 23.25 & 9.5 & 83.5 & 59.5 & 91.54 & 43.42 & 63.18 & 54.13 & 44.10 \\
 KeyDiff & 30.23 & 40.88 & 54.35 & 57.59 & 48.15 & 30.29 & 29.87 & 24.58 & 24.14 & 9.17 & 99.5 & 61.5 & 90.74 & 41.1 & 49.15 & 48.17 & 46.21 \\
 CAKE & 29.65 & 44.72 & 56.87 & 57.98 & 49.5 & 32.1 & 30.29 & 24.99 & 23.69 & 9.5 & 100 & 66.5 & 92.26 & 43.44 & 64.4 & 55.21 & 48.82 \\
 CriticalKV & 30.61 & 45.54 & 56.57 & 57.93 & 49.89 & 32.04 & 30.92 & 24.95 & 24.12 & 9.5 & 100 & 68.5 & 91.87 & 42.83 & 64.63 & 55.12 & \underline{49.06} \\
 AnDPro & 30.78 & 44.93 & 56.67 & 57.95 & 49.58 & 31.95 & 30.56 & 25.39 & 23.72 & 9.5 & 100 & 67.5 & 92.33 & 43.51 & 64.62 & 55.95 & \underline{49.06} \\
 \rowcolor{gray!30} DropKV & 30.29 & 45.22 & 57.22 & 58.02 & 50.06 & 31.42 & 31.07 & 24.93 & 24.13 & 9.5 & 100 & 68.5 & 91.73 & 44.18 & 64.46 & 55.57 & \textbf{49.14} \\
 \midrule
 \multicolumn{18}{c}{\textbf{Qwen2.5-7B-Instruct-1M}} \\ 
 \midrule
 FullKV & 28.88 & 48.39 & 51.57 & 60.27 & 53.19 & 35.63 & 33.33 & 23.22 & 24.53 & 6.5 & 100 & 66.5 & 85.28 & 45.29 & 45.61 & 35.31 & 46.47\\
 \midrule
 SnapKV & 28.44 & 46.63 & 51.16 & 59.98 & 53.11 & 36.02 & 31.17 & 23.35 & 20.53 & 7.5 & 100 & 61.5 & 85.03 & 45.04 & 45.77 & 35.74 & 45.69 \\
 PyramidKV & 29.12 & 42.89 & 51.19 & 59.8 & 52.79 & 35.42 & 27.35 & 22.43 & 19.05 & 7 & 100 & 48 & 84.67 & 44.74 & 43.81 & 34.52 & 43.92 \\
 StreamingLLM & 25.91 & 37.66 & 32.18 & 50.75 & 48.04 & 27.35 & 29.46 & 19.84 & 20.97 & 5.5 & 46 & 54.5 & 85.29 & 44.41 & 44.31 & 36.62 & 38.05 \\
 KeyDiff & 27.13 & 37.4 & 44.17 & 55.13 & 46.07 & 30.67 & 29.02 & 21.94 & 19.79 & 5 & 100 & 61 & 82.8 & 40.72 & 31.11 & 25.05 & 41.06 \\
 CAKE & 28.81 & 46.36 & 51.38 & 59.78 & 53.34 & 35.62 & 31.22 & 23.18 & 20.95 & 7.5 & 100 & 63 & 84.85 & 45.48 & 45.74 & 35.49 & 45.79\\
 CriticalKV & 28.52 & 46.75 & 51.72 & 59.95 & 52.9 & 35.65 & 31.37 & 22.84 & 20.94 & 6.5 & 100 & 64.5 & 84.99 & 45.26 & 45.32 & 35.9 & 45.82 \\
 AnDPro & 27.77 & 46.58 & 50.75 & 60.3 & 53.41 & 35.55 & 31.37 & 23.14 & 20.6 & 7 & 100 & 66 & 84.99 & 44.92 & 45.4 & 35.84 & \underline{45.85} \\
 \rowcolor{gray!30} DropKV & 28.28 & 47.4 & 51.18 & 60.13 & 52.84 & 35.12 & 31.55 & 22.86 & 21.55 & 7 & 100 & 65 & 84.99 & 45.13 & 45.52 & 35.99 & \textbf{45.91} \\
 \midrule
 \multicolumn{18}{c}{\textbf{Mistral-7B-512k}} \\ 
 \midrule
 FullKV & 22.04 & 16.47 & 42.75 & 24.06 & 20.72 & 13.04 & 33.35 & 25.25 & 26.97 & 0 & 100 & 74.5 & 89.55 & 41.69 & 55.54 & 56.39 & 40.15 \\
 \midrule
 SnapKV & 21.88 & 16.3 & 42.1 & 21.77 & 19.02 & 11.4 & 30.74 & 25.33 & 24.5 & 0 & 100 & 69 & 89.71 & 40.78 & 54.93 & 56.14 & 38.98 \\
 PyramidKV & 21.92 & 17.28 & 41.48 & 18.51 & 15.3 & 10.68 & 30.05 & 25.63 & 24.08 & 0 & 100 & 66.5 & 90.05 & 41.1 & 54 & 54.5 & 38.19 \\
 StreamingLLM & 21.53 & 12.1 & 27.22 & 18.28 & 16.83 & 8.58 & 27.68 & 22.44 & 23.4 & 0 & 63 & 60.5 & 89.38 & 39.97 & 53.96 & 54.68 & 33.72 \\
 KeyDiff & 22.82 & 15.87 & 36.11 & 18.54 & 15.44 & 9.29 & 28.75 & 24.98 & 24.81 & 0 & 71.5 & 56 & 85.25 & 38.3 & 41.82 & 47.91 & 33.59 \\
 CAKE &  21.46 & 15.56 & 41.85 & 20.99 & 19.32 & 11.69 & 31.28 & 25.05 & 24.91 & 0 & 100 & 68.5 & 89.71 & 41.83 & 54.56 & 56.51 & 38.95 \\
 CriticalKV & 22.22 & 15.62 & 41.47 & 21.9 & 19.95 & 11.38 & 30.98 & 25.41 & 25.02 & 0 & 100 & 72.5 & 89.96 & 41.4 & 55.36 & 56.01 & \underline{39.32} \\
 AnDPro & 21.18 & 15.97 & 41.65 & 22.16 & 18.72 & 11.67 & 30.59 & 25.16 & 24.84 & 0 & 100 & 73.5 & 89.79 & 42.28 & 54.94 & 56.1 & 39.28 \\
 \rowcolor{gray!30} DropKV & 22.71 & 15.8 & 41.1 & 22.34 & 19.82 & 11.82 & 30.98 & 25.5 & 25.28 & 0 & 100 & 73.5 & 89.71 & 42.48 & 55.08 & 55.78 & \textbf{39.49} \\
 \bottomrule
 \end{tabular}

 }
\end{table}

\subsection{Extended Results with adaptive head budget}\label{app:adadropkv-extended}
Tables~\ref{tab:ruler_adadropkv_appendix} and \ref{tab:longbench_adadropkv} extend the evaluation of AdaDropKV to broader sequence lengths and expanded KV-cache budgets ($5\%$, $10\%$, and $30\%$) on RULER and LongBench, respectively. These extended results further demonstrate the versatility and robustness of combining DropKV's output-perturbation scoring with adaptive head budget allocation under diverse long-context workloads.
 
\begin{table}[H]
 \centering
 \renewcommand{\arraystretch}{0.88}
 \caption{RULER results at 128K and 64K sequence length (S. L.) and 5\%, 10\% and 30\% KV-cache budget. The best result is highlighted in gray, and the second-best is underlined.
 }
 \label{tab:ruler_adadropkv_appendix}
 \resizebox{\textwidth}{!}{
 \begin{tabular}{l|c|c|rrrrrrrrrrr|r}
 \toprule
 S. L.& \textbf{Method} & Budget & N-S1 & N-S2 & N-S3 & N-MK1 & N-MK2 & N-MQ & N-MV & VT & FWE & QA-1 & QA-2 & Avg. \\ 
 \midrule
 \multicolumn{15}{c}{\textbf{LLaMA-3.1-8B-Instruct}} \\ 
 \midrule
 \multirow{15}{*}{128k} & Full KV & 100\% & 100 & 99.2 & 99.6 & 97.6 & 88.2 & 98.7 & 95.85 & 84 & 61.93 & 64.17 & 44 & 84.84 \\
 \cmidrule(lr){2-15}
 & CriticalKV & \multirow{4}{*}{5\%} & 100 & 99 & 69.2 & 97.6 & 75.6 & 98.55 & 95.2 & 77.4 & 46.67 & 63.9 & 42.8 & 78.72 \\
 & AnDPro & & 100 & 99 & 68.4 & 97.4 & 78.4 & 98.75 & 96 & 76.52 & 54.53 & 63.76 & 43.4 & 79.65 \\
 & DropKV & & 100 & 98.8 & 89.6 & 97.6 & 86.2 & 98.25 & 96.9 & 77.32 & 49.73 & 64.17 & 43.4 & \underline{82.00} \\
 & \cellcolor{gray!30}AdaDropKV & \cellcolor{gray!30}& \cellcolor{gray!30}100 & \cellcolor{gray!30}100 & \cellcolor{gray!30}93.75 & \cellcolor{gray!30}97.66 & \cellcolor{gray!30}86.8 & \cellcolor{gray!30}99.02 & \cellcolor{gray!30}95.7 & \cellcolor{gray!30}77.66 & \cellcolor{gray!30}54.34 & \cellcolor{gray!30}64.32 & \cellcolor{gray!30}43.4 & \cellcolor{gray!30}\textbf{82.97} \\
 \cmidrule(lr){2-15}
 & CriticalKV & \multirow{4}{*}{10\%} & 100 & 98.8 & 90.4 & 97.4 & 81.8 & 98.5 & 95.55 & 78.64 & 52.73 & 63.3 & 43.4 & 81.87 \\
 & AnDPro & & 100 & 99 & 89.2 & 97.4 & 82.6 & 99 & 95.6 & 78.96 & 56.27 & 63.3 & 43.8 & 82.28 \\
 & DropKV & & 100 & 98.6 & 97.6 & 97.6 & 87.6 & 98.8 & 97.1 & 79.24 & 51.87 & 64.16 & 43.4 & \underline{83.27} \\
 & \cellcolor{gray!30}AdaDropkv &\cellcolor{gray!30}& \cellcolor{gray!30}100 & \cellcolor{gray!30}99.22 & \cellcolor{gray!30}98.44 & \cellcolor{gray!30}97.66 & \cellcolor{gray!30}84.59 & \cellcolor{gray!30}99.41 & \cellcolor{gray!30}95.7 & \cellcolor{gray!30}78.97 & \cellcolor{gray!30}56.73 & \cellcolor{gray!30}63.94 & \cellcolor{gray!30}43.4 & \cellcolor{gray!30}\textbf{83.46} \\
 \cmidrule(r){2-15}
 & CriticalKV & \multirow{4}{*}{30\%} & 100 & 100 & 99.22 & 97.66 & 85.94 & 99.22 & 95.51 & 77.26 & 60.68 & 59.25 & 41.41 & \underline{83.29} \\
 & AnDPro & & 100 & 100 & 97.66 & 97.66 & 84.38 & 99.02 & 93.55 & 76.56 & 58.33 & 60.55 & 41.41 & 82.65 \\
 & DropKV & & 100 & 100 & 99.22 & 97.66 & 85.94 & 99.22 & 96.09 & 77.5 & 59.38 & 60.03 & 41.41 & \underline{83.31} \\
 & \cellcolor{gray!30}AdaDropkv & \cellcolor{gray!30}& \cellcolor{gray!30}100 & \cellcolor{gray!30}100 & \cellcolor{gray!30}99.22 & \cellcolor{gray!30}97.66 & \cellcolor{gray!30}85.94 & \cellcolor{gray!30}99.22 & \cellcolor{gray!30}95.53 & \cellcolor{gray!30}82.03 & \cellcolor{gray!30}60.11 & \cellcolor{gray!30}60.81 & \cellcolor{gray!30}41.41 & \cellcolor{gray!30}\textbf{83.81} \\
 \midrule
 \multirow{15}{*}{64k} & FullKV & 100\% & 100 & 100 & 100 & 100 & 100 & 100 & 99.02 & 98.91 & 85.68 & 71.22 & 50 & 91.35 \\
 \cmidrule(lr){2-15}
 & AdaKV & \multirow{5}{*}{5\%} & 100 & 100 & 90.63 & 100 & 90.63 & 99.61 & 98.44 & 97.18 & 74.74 & 71.22 & 47.66 & 88.19 \\
 & CriticalKV &  & 100 & 100 & 76.2 & 99.2 & 95.4 & 99.55 & 99.15 & 96.24 & 72.73 & 69.3 & 49.2 & 87.00 \\
 & AnDPro & & 100 & 100 & 76.8 & 99.8 & 97.6 & 99.85 & 98.7 & 96.92 & 74.53 & 69.67 & 49.4 & 87.57 \\
 & DropKV & & 100 & 100 & 96.8 & 99.6 & 99 & 99.45 & 98.35 & 96.2 & 71.67 & 69.2 & 50.4 & \underline{89.15} \\
 & \cellcolor{gray!30}AdaDropKV &\cellcolor{gray!30} & \cellcolor{gray!30}100 & \cellcolor{gray!30}100 & \cellcolor{gray!30}98.44 & \cellcolor{gray!30}100 & \cellcolor{gray!30}100 & \cellcolor{gray!30}99.41 & \cellcolor{gray!30}99.22 & \cellcolor{gray!30}97.66 & \cellcolor{gray!30}75.52 & \cellcolor{gray!30}72.01 & \cellcolor{gray!30}48.66 & \cellcolor{gray!30}\textbf{90.08} \\
 \cmidrule(lr){2-15}
 & AdaKV & \multirow{5}{*}{10\%} & 100 & 100 & 100 & 100 & 97.66 & 99.81 & 99.22 & 97.29 & 78.25 & 70.44 & 47.66 & 90.03 \\
 & CriticalKV & & 100 & 100 & 96.09 & 100 & 99.22 & 99.81 & 99.61 & 96.88 & 80.21 & 70.44 & 47.66 & 89.99 \\
 & AnDPro & & 100 & 100 & 96.09 & 100 & 98.44 & 100 & 99.41 & 97.18 & 78.85 & 71.74 & 48.44 & 90.01 \\
 & DropKV & & 100 & 100 & 100 & 100 & 100 & 100 & 99.22 & 97.19 & 77.87 & 72.01 & 46.88 & \underline{90.29} \\
 & \cellcolor{gray!30} AdaDropKV & \cellcolor{gray!30}& \cellcolor{gray!30}100 & \cellcolor{gray!30}100 & \cellcolor{gray!30}100 & \cellcolor{gray!30}100 & \cellcolor{gray!30}100 & \cellcolor{gray!30}100 & \cellcolor{gray!30}99.22 & \cellcolor{gray!30}97.66 & \cellcolor{gray!30}79.17 & \cellcolor{gray!30}72.27 & \cellcolor{gray!30}47.66 & \cellcolor{gray!30}\textbf{90.54} \\
 \cmidrule(lr){2-15}
 & AdaKV & \multirow{5}{*}{30\%} & 100 & 100 & 100 & 100 & 100 & 100 & 98.83 & 98.75 & 86.42 & 70.44 & 46.88 & 91.03 \\
 & CriticalKV & & 100 & 100 & 100 & 100 & 100 & 100 & 98.63 & 97.97 & 85.42 & 71.22 & 46.88 & 90.92 \\
 & AnDPro & & 100 & 100 & 100 & 100 & 100 & 100 & 99.81 & 98.12 & 86.58 & 70.7 & 47.66 & 91.17 \\
 & DropKV & & 100 & 100 & 100 & 100 & 100 & 100 & 99.02 & 98.28 & 86.46 & 71.75 & 48.09 & \underline{91.24} \\
 & \cellcolor{gray!30}AdaDropKV & \cellcolor{gray!30} & \cellcolor{gray!30} 100 & \cellcolor{gray!30} 100 & \cellcolor{gray!30} 100 & \cellcolor{gray!30} 100 & \cellcolor{gray!30} 100 & \cellcolor{gray!30} 100 & \cellcolor{gray!30} 98.83 & \cellcolor{gray!30} 98.91 & \cellcolor{gray!30} 86.46 & \cellcolor{gray!30} 71.7 & \cellcolor{gray!30} 48.44 & \cellcolor{gray!30} \textbf{91.30} \\
 \midrule
 \multicolumn{15}{c}{\textbf{Qwen2.5-7B-Instruct-1M}} \\ 
 \midrule
 \multirow{15}{*}{128k} & FullKV & 100\% & 100 & 100 & 99.6 & 99.8 & 98.2 & 99.9 & 78.9 & 95.6 & 81.13 & 57 & 47 & 87.01 \\
 \cmidrule(lr){2-15}
 & CriticalKV & \multirow{4}{*}{5\%} & 100 & 100 & 22.6 & 99.8 & 88.2 & 99.55 & 77.15 & 96.48 & 64.93 & 56.38 & 46.6  & 77.43 \\
 & AnDPro & & 99.8 & 99.8 & 45.8 & 99.6 & 90 & 99.6 & 77.25 & 96.4 & 66.93 & 57.73 & 46 & 79.90 \\
 & DropKV & & 100 & 99.6 & 81.6 & 99.2 & 97.4 & 99.7 & 77.15 & 96.64 & 67.33 & 57.07 & 45.8 & \underline{83.77} \\
 & \cellcolor{gray!30}AdaDropKV &\cellcolor{gray!30} & \cellcolor{gray!30}100 & \cellcolor{gray!30}100 & \cellcolor{gray!30}81.56 & \cellcolor{gray!30}100 & \cellcolor{gray!30}98.44 & \cellcolor{gray!30}100 & \cellcolor{gray!30}77.76 & \cellcolor{gray!30}96.25 & \cellcolor{gray!30}67.89 & \cellcolor{gray!30}57.34 & \cellcolor{gray!30}49.22 & \cellcolor{gray!30}\textbf{84.41} \\
 \cmidrule(lr){2-15}
 & CriticalKV & \multirow{4}{*}{10\%} & 100 & 100 & 58.2 & 99.6 & 92 & 99.6 & 76.65 & 96.12 & 69.13 & 56.73 & 45.8 & 81.26 \\
 & AnDPro & & 99.8 & 100 & 74 & 99.8 & 94.6 & 99.75 & 76.9 & 96.48 & 71.4 & 57.4 & 45.6 & 83.25 \\
 & DropKV & & 100 & 99.4 & 97 & 99.2 & 97.8 & 99.7 & 77.1 & 96.48 & 69.6 & 56.13 & 46 & \underline{85.31} \\
 & \cellcolor{gray!30}AdaDropKV & \cellcolor{gray!30}& \cellcolor{gray!30}100 & \cellcolor{gray!30}100 & \cellcolor{gray!30}97.38 & \cellcolor{gray!30}100 & \cellcolor{gray!30}97.66 & \cellcolor{gray!30}99.81 & \cellcolor{gray!30}77.78 & \cellcolor{gray!30}96.94 & \cellcolor{gray!30}69.23 & \cellcolor{gray!30}56.12 & \cellcolor{gray!30}47.66 & \cellcolor{gray!30}\textbf{85.69} \\
 \cmidrule(lr){2-15}
 & CriticalKV & \multirow{4}{*}{30\%} & 100 & 100 & 98.44 & 100 & 97.66 & 99.8 & 75.39 & 96.41 & 74.74 & 54.82 & 47.66 & \underline{85.90} \\
 & AnDPro & & 100 & 100 & 89.06 & 100 & 97.66 & 99.61 & 75.97 & 95.78 & 78.39 & 55.6 & 48.44 & 85.50 \\
 & DropKV & & 100 & 100 & 99.22 & 100 & 99.22 & 99.61 & 75.78 & 96.09 & 73.96 & 57.16 & 47.66 & \underline{86.25} \\
 & \cellcolor{gray!30}AdaDropKV & \cellcolor{gray!30}& \cellcolor{gray!30}100 & \cellcolor{gray!30}100 & \cellcolor{gray!30}99.22 & \cellcolor{gray!30}100 & \cellcolor{gray!30}99.22 & \cellcolor{gray!30}99.81 & \cellcolor{gray!30}75.78 & \cellcolor{gray!30}96.94 & \cellcolor{gray!30}74.23 & \cellcolor{gray!30}57.12 & \cellcolor{gray!30}47.66 & \cellcolor{gray!30}\textbf{86.36} \\
 \midrule
 \multirow{15}{*}{64k} & FullKV & 100\% & 100 & 100 & 100 & 100 & 99.22 & 100 & 80.86 & 94.53 & 83.33 & 60.29 & 57.81 & 88.73 \\
 \cmidrule(lr){2-15}
 & AdaKV & \multirow{5}{*}{5\%} & 100 & 100 & 7.03 & 100 & 80.47 & 100 & 78.32 & 94.9 & 66.93 & 57.88 & 55.47 & 76.45 \\
 & CriticalKV & & 100 & 100 & 12.6 & 100 & 88.2 & 99.9 & 77.55 & 94.84 & 66.73 & 55.38 & 52 & 77.02 \\
 & AnDPro & & 100 & 99.8 & 39.2 & 100 & 89.4 & 99.8 & 78.7 & 95.36 & 67 & 55.33 & 52.4 & 79.73 \\
 & DropKV & & 100 & 99.6 & 74.4 & 99.6 & 98.4 & 99.85 & 78.95 & 95.36 & 67.13 & 55.73 & 51.8 & \underline{83.71} \\
 & \cellcolor{gray!30}AdaDropKV & \cellcolor{gray!30}& \cellcolor{gray!30}100 & \cellcolor{gray!30}100 & \cellcolor{gray!30}75.34 & \cellcolor{gray!30}100 & \cellcolor{gray!30}95.31 & \cellcolor{gray!30}100 & \cellcolor{gray!30}77.93 & \cellcolor{gray!30}93.75 & \cellcolor{gray!30}68.23 & \cellcolor{gray!30}58.4 & \cellcolor{gray!30}57.81 & \cellcolor{gray!30}\textbf{84.25} \\
 \cmidrule(lr){2-15}
 & AdaKV & \multirow{5}{*}{10\%} & 100 & 100 & 30.47 & 100 & 89.84 & 100 & 78.52 & 94.37 & 67.19 & 58.2 & 57.03 & 79.60 \\
 & CriticalKV & & 100 & 100 & 42.97 & 100 & 91.41 & 100 & 77.93 & 94.06 & 67.19 & 56.84 & 58.59 & 80.82 \\
 & AnDPro & & 100 & 99.22 & 64.84 & 100 & 94.53 & 99.81 & 78.52 & 95.16 & 67.19 & 57.42 & 56.25 & 82.99 \\
 & DropKV & & 100 & 99.22 & 85.94 & 100 & 98.44 & 100 & 79.1 & 94.69 & 68.49 & 58.72 & 56.25 & \underline{85.53} \\
 & \cellcolor{gray!30}AdaDropKV & \cellcolor{gray!30}& \cellcolor{gray!30}100 & \cellcolor{gray!30}100 & \cellcolor{gray!30}88.16 & \cellcolor{gray!30}100 & \cellcolor{gray!30}97.66 & \cellcolor{gray!30}100 & \cellcolor{gray!30}76.95 & \cellcolor{gray!30}94.9 & \cellcolor{gray!30}69.27 & \cellcolor{gray!30}58.2 & \cellcolor{gray!30}58.59 & \cellcolor{gray!30}\textbf{85.79} \\
 \cmidrule(lr){2-15}
 & AdaKV & \multirow{5}{*}{30\%} & 100 & 100 & 88.28 & 100 & 92.19 & 100 & 77.93 & 94.9 & 69.27 & 58.98 & 57.81 & 85.40 \\
 & CriticalKV & & 100 & 100 & 96.09 & 100 & 97.66 & 99.81 & 78.91 & 94.69 & 74.48 & 57.88 & 57.03 & 86.96 \\
 & AnDPro & & 100 & 100 & 91.41 & 100 & 96.09 & 99.61 & 79.69 & 94.53 & 73.44 & 58.2 & 56.25 & 86.29 \\
 & DropKV & & 100 & 100 & 100 & 100 & 100 & 99.41 & 79.1 & 94.69 & 75.26 & 58.72 & 57.03 & \underline{87.66} \\
 & \cellcolor{gray!30}AdaDropKV & \cellcolor{gray!30}& \cellcolor{gray!30}100 & \cellcolor{gray!30}100 & \cellcolor{gray!30}99.22 & \cellcolor{gray!30}100 & \cellcolor{gray!30}99.22 & \cellcolor{gray!30}100 & \cellcolor{gray!30}79.1 & \cellcolor{gray!30}94.53 & \cellcolor{gray!30}75.48 & \cellcolor{gray!30}58.83 & \cellcolor{gray!30}58.98 & \cellcolor{gray!30}\textbf{87.76} \\
 \bottomrule
 \end{tabular}

 }
\end{table}

\begin{table}[!h]
 \centering
 \caption{LongBench results at 5\%, 10\% and 30\% cache budget. The best result is highlighted in gray, and the second-best is underlined.}
 \label{tab:longbench_adadropkv}
 \resizebox{\textwidth}{!}{
 \begin{tabular}{l|c|cccccccccccccccc|r}
 \toprule
 & & \multicolumn{3}{c}{\textbf{Single Doc. QA}} & \multicolumn{3}{c}{\textbf{Multi Doc. QA}} & \multicolumn{3}{c}{\textbf{Summarization}} & \multicolumn{2}{c}{\textbf{Synthetic}} & \multicolumn{3}{c}{\textbf{Fewshot Learning}} & \multicolumn{2}{c}{\textbf{Code}} & \\ 
 \cmidrule(lr){3-5} \cmidrule(lr){6-8} \cmidrule(lr){9-11} \cmidrule(lr){12-13} \cmidrule(lr){14-16} \cmidrule(lr){17-18}
 Method & Budget & \rotatebox{45}{NrtvQA} & \rotatebox{45}{Qasper} & \rotatebox{45}{MF-en} & \rotatebox{45}{HotpotQA} & \rotatebox{45}{2WikiMQA} & \rotatebox{45}{Musique} & \rotatebox{45}{GovReport} & \rotatebox{45}{QMSum} & \rotatebox{45}{MultiNews} & \rotatebox{45}{PCount} & \rotatebox{45}{PR-en} & \rotatebox{45}{TREC} & \rotatebox{45}{TriviaQA} & \rotatebox{45}{SAMSum} & \rotatebox{45}{Lcc} & \rotatebox{45}{RB-P} & \textbf{Avg.} \\ 
 \midrule
  
 \multicolumn{19}{c}{\textbf{LLaMA-3.1-8B-Instruct}} \\ 
 \midrule
 FullKV & 100\% & 30.51 & 45.59 & 56.34 & 58.22 & 50.02 & 31.75 & 34.3 & 25.33 & 26.76 & 9.5 & 100 & 73.5 & 92.29 & 43.96 & 64.77 & 55.7 & 49.91 \\
 \midrule
 CAKE & \multirow{6}{*}{5\%} & 30.41 & 33.54 & 54.79 & 57.09 & 48.29 & 31.07 & 23.54 & 24.36 & 17.33 & 9.5 & 100 & 37 & 91.05 & 42.54 & 58.58 & 53.7 & 44.55 \\
 AdaKV & & 29.88 & 32.53 & 54.38 & 57.25 & 46.79 & 30.55 & 23.33 & 24.21 & 17.08 & 9.5 & 100 & 42 & 91.8 & 40.99 & 60.62 & 55.28 & 44.76 \\
 CriticalKV & & 30.03 & 31.07 & 54.55 & 57.33 & 47.15 & 30.09 & 23.16 & 23.92 & 18.46 & 9.5 & 100 & 41 & 92.03 & 42.14 & 59.25 & 54.54 & 44.64 \\
 AnDPro & & 30.97 & 32.06 & 54.44 & 57.62 & 46.44 & 31.53 & 22.92 & 24.25 & 17.98 & 9.25 & 100 & 38 & 91.65 & 41.86 & 58.58 & 54.9 & 44.53 \\
 DropKV & & 30.48 & 32.49 & 55.34 & 57.29 & 46.84 & 31.25 & 23.71 & 24.94 & 18.7 & 9.5 & 100 & 48.5 & 92.27 & 43.06 & 60.29 & 53.47 & \underline{45.51} \\
 \rowcolor{gray!30} AdaDropKV & & 30.57 & 38.4 & 55.64 & 57.84 & 48.03 & 30.71 & 24.73 & 24.67 & 19.05 & 9.5 & 100 & 52.5 & 92.21 & 43.01 & 62.02 & 54.35 & \textbf{46.45} \\
 \midrule
 CAKE & \multirow{6}{*}{10\%} & 29.6 & 40.32 & 56.08 & 57.69 & 49.45 & 31.18 & 26.11 & 24.56 & 20.03 & 9.5 & 100 & 50 & 92.53 & 43.72 & 62.05 & 54.98 & 46.74 \\
 AdaKV & & 30.63 & 38.26 & 55.15 & 57.21 & 48.11 & 32.06 & 25.73 & 24.42 & 20.34 & 9.2 & 100 & 58 & 91.57 & 42.33 & 63.03 & 55.08 & 46.95 \\
 CriticalKV & & 30.39 & 38.44 & 55.87 & 57.22 & 48.46 & 31.85 & 26.29 & 24.9 & 20.96 & 9.25 & 100 & 55 & 91.6 & 42.94 & 63.15 & 54.42 & 46.92 \\
 AnDPro & & 30.9 & 37.39 & 56.18 & 57.9 & 48.13 & 32.07 & 25.46 & 24.38 & 20.48 & 9.25 & 100 & 56 & 91.97 & 43.77 & 61.74 & 55.64 & 46.95 \\
 DropKV & & 31.34 & 40.66 & 57.3 & 57.9 & 48.49 & 31.65 & 26.77 & 24.37 & 21 & 9.5 & 100 & 59 & 91.75 & 43.61 & 62.29 & 54.38 & \underline{47.50} \\
 \rowcolor{gray!30} AdaDropKV & & 31.25 & 43.24 & 56.75 & 58.03 & 49.63 & 32.28 & 27.39 & 24.95 & 21.75 & 9.5 & 100 & 66 & 91.72 & 43.69 & 63.49 & 55.23 & \textbf{48.43} \\
 \midrule
 CAKE & \multirow{6}{*}{30\%} & 29.65 & 44.72 & 56.87 & 57.98 & 49.5 & 32.1 & 30.29 & 24.99 & 23.69 & 9.5 & 100 & 66.5 & 92.26 & 43.44 & 64.4 & 55.21 & 48.82 \\
 AdaKV & & 30.16 & 45.14 & 56.91 & 58.09 & 50.18 & 31.34 & 30.34 & 24.92 & 24.06 & 9.5 & 100 & 68.5 & 91.58 & 43.79 & 64.65 & 55.63 & 49.05 \\
 CriticalKV & & 30.61 & 45.54 & 56.57 & 57.93 & 49.89 & 32.04 & 30.92 & 24.95 & 24.12 & 9.5 & 100 & 68.5 & 91.87 & 42.83 & 64.63 & 55.12 & 49.06 \\
 AnDPro & & 30.78 & 44.93 & 56.67 & 57.95 & 49.58 & 31.95 & 30.56 & 25.39 & 23.72 & 9.5 & 100 & 67.5 & 92.33 & 43.51 & 64.62 & 55.95 & 49.06 \\
 DropKV & & 30.29 & 45.22 & 57.22 & 58.02 & 50.06 & 31.42 & 31.07 & 24.93 & 24.13 & 9.5 & 100 & 68.5 & 91.73 & 44.18 & 64.46 & 55.57 & \underline{49.14} \\
 \rowcolor{gray!30} AdaDropKV & & 30.59 & 46.1 & 56.82 & 58.17 & 50.39 & 31.83 & 31.08 & 25.34 & 24.39 & 9.4 & 100 & 71.5 & 91.85 & 44.05 & 64.09 & 55.86 & \textbf{49.47} \\
 \midrule
 \multicolumn{19}{c}{\textbf{Qwen2.5-7B-Instruct-1M}} \\ 
 \midrule
 FullKV &100\% & 28.88 & 48.39 & 51.57 & 60.27 & 53.19 & 35.63 & 33.33 & 23.22 & 24.53 & 6.5 & 100 & 66.5 & 85.28 & 45.29 & 45.61 & 35.31 & 46.47\\
 \midrule
 CAKE & \multirow{6}{*}{5\%} & 28.75 & 38.51 & 49.22 & 59.59 & 52.14 & 33.85 & 24.12 & 22.59 & 14.08 & 7 & 100 & 36 & 84.71 & 44.35 & 37.53 & 35.06 & 41.72 \\
 AdaKV & & 28.29 & 36.09 & 47.33 & 58.78 & 51.62 & 34.09 & 23.55 & 21.72 & 13.89 & 6.5 & 100 & 36.5 & 85.18 & 44.07 & 39.21 & 35.22 & 41.38 \\
 CriticalKV & & 28.31 & 38.22 & 48.83 & 57.63 & 51.99 & 34.11 & 24.4 & 22.35 & 14.83 & 7 & 100 & 36.5 & 84.61 & 43.83 & 38.59 & 34.73 & 41.52\\
 AnDPro & & 28.61 & 36.11 & 49.88 & 59.44 & 50.87 & 34.56 & 23.61 & 22.24 & 14.22 & 6.5 & 100 & 37.5 & 84.71 & 44.21 & 38.08 & 34.76 & 41.58 \\
 DropKV & & 28.28 & 39.09 & 50.54 & 58.87 & 51.6 & 33.93 & 25.25 & 21.42 & 15.36 & 7 & 100 & 48.5 & 85.06 & 44.85 & 38.95 & 33.88 & \underline{42.66} \\
 \rowcolor{gray!30} AdaDropKV & & 28.09 & 42.39 & 49.41 & 59.4 & 52.52 & 35.04 & 26.12 & 22.18 & 15.71 & 6.5 & 100 & 48 & 84.25 & 44.14 & 39.85 & 34.78 & \textbf{43.02} \\
 \midrule
 CAKE & \multirow{6}{*}{10\%} & 28.85 & 41.88 & 50.94 & 60.09 & 52.09 & 34.75 & 27.41 & 22.97 & 17.23 & 7 & 100 & 48 & 85.03 & 44.31 & 43.09 & 35.35 & 43.69 \\
 AdaKV & & 27.55 & 41.62 & 50.55 & 59.72 & 52.91 & 34.78 & 26.53 & 22.72 & 17.28 & 6.5 & 100 & 45.67 & 84.83 & 45.21 & 42.45 & 35.7 & 43.38 \\ 
 CriticalKV & & 27.93 & 42.81 & 50.71 & 60.55 & 52.34 & 35.1 & 27.67 & 22.69 & 17.52 & 7 & 100 & 48.5 & 85.39 & 44.22 & 42.46 & 35.64 & 43.78 \\
 AnDPro & & 27.98 & 42.07 & 50.68 & 59.84 & 51.75 & 36.02 & 26.97 & 22.86 & 16.56 & 7 & 100 & 53 & 85.34 & 45.1 & 42.65 & 35.23 & 43.94 \\
 DropKV & & 28.48 & 43.43 & 51 & 60.08 & 52.84 & 34.46 & 28.23 & 22.15 & 18.14 & 6.5 & 100 & 59.5 & 85.15 & 44.62 & 43.01 & 35.73 & \underline{44.58} \\
 \rowcolor{gray!30} AdaDropKV & & 27.91 & 44.77 & 50.98 & 60.01 & 52.82 & 34.6 & 28.61 & 22.63 & 18.33 & 6.5 & 100 & 60.17 & 85.2 & 44.62 & 44.13 & 35.44 & \textbf{44.80} \\ 
 \midrule
 CAKE & \multirow{6}{*}{30\%} & 28.81 & 46.36 & 51.38 & 59.78 & 53.34 & 35.62 & 31.22 & 23.18 & 20.95 & 7.5 & 100 & 63 & 84.85 & 45.48 & 45.74 & 35.49 & 45.79\\
 AdaKV & & 28.35 & 45.87 & 50.76 & 59.89 & 53.49 & 35.66 & 30.77 & 23.02 & 20.69 & 7.5 & 100 & 63 & 85.11 & 44.69 & 45.68 & 35.48 & 45.62 \\ 
 CriticalKV & & 28.52 & 46.75 & 51.72 & 59.95 & 52.9 & 35.65 & 31.37 & 22.84 & 20.94 & 6.5 & 100 & 64.5 & 84.99 & 45.26 & 45.32 & 35.9 & 45.82 \\
 AnDPro & & 27.77 & 46.58 & 50.75 & 60.3 & 53.41 & 35.55 & 31.37 & 23.14 & 20.6 & 7 & 100 & 66 & 84.99 & 44.92 & 45.4 & 35.84 & 45.85 \\
 DropKV & & 28.28 & 47.4 & 51.18 & 60.13 & 52.84 & 35.12 & 31.55 & 22.86 & 21.55 & 7 & 100 & 65 & 84.99 & 45.13 & 45.52 & 35.99 & \underline{45.91} \\
 \rowcolor{gray!30} AdaDropKV & & 28.61 & 47.03 & 51 & 60.1 & 53.24 & 35.09 & 31.35 & 23.15 & 21.22 & 7 & 100 & 66.5 & 85.21 & 44.97 & 45.97 & 35.62 & \textbf{46.00} \\ 
 \bottomrule
 \end{tabular}

 }
\end{table}

\subsection{Serving Throughput vs.\ Dynamic Cache}\label{app:serving}

\textbf{Serving throughput vs.\ dynamic cache.} As shown in Table~\ref{tab:serving-throughput}, with DropKV the maximum batch ranges from $1.75{\times}$ (Qwen, $n{=}128\mathrm{K}$) to $6.87{\times}$ (Mistral, $n{=}128\mathrm{K}$) of the dynamic baseline, with Triton and PyTorch DropKV realizing comparable capacity on Llama, a clear Triton lead on Mistral at every $n$ (peaking at $B{=}3$ vs.\ $B{=}2$ and $70.8{\to}86.5$\,tok/s at $n{=}128\mathrm{K}$), and Triton matched-batch-neutral on Qwen only at $n{=}32\mathrm{K}$ ($\delta{=}{-}0.1\%$) and positive at $n{=}4\mathrm{K}$ ($+1.9\%$) and $n{=}128\mathrm{K}$ ($+3.3\%$); Qwen's smaller $H_{\mathrm{kv}}{=}4$ halves the kernel grid $B{\cdot}H_{\mathrm{kv}}$, dominant at intermediate $n$.

\begin{table}[ht]
\centering
\caption{Serving throughput at each method's maximum batch at $5\%$ budget on A100-80GB GPU. $B$ is the largest batch that fits for the given $(n,\text{method})$; aggregate decode throughput is in tokens/s. $\times_{\mathrm{dyn}}$ is Triton's decode throughput divided by dynamic's.}
\label{tab:serving-throughput}
\begin{tabular}{l|rrrrrrrr}
\toprule
 & & \multicolumn{2}{c}{Dynamic} & \multicolumn{2}{c}{DropKV-PyTorch} & \multicolumn{2}{c}{DropKV-Triton} & \\
\cmidrule(lr){3-4}\cmidrule(lr){5-6}\cmidrule(lr){7-8}
Model & $n$ & $B$ & tok/s & $B$ & tok/s & $B$ & tok/s & $\times_{\mathrm{dyn}}$ \\
\midrule
\multirow{3}{*}{Llama-3.1-8B}     &   4k & 38 &   395 & 56 &  2111 & 56 &  2066 & $5.23$ \\
                 &  32k &  4 &    50 &  7 &   260 &  6 &   223 & $4.46$ \\
                 & 128k &  1 &    12 &  1 &    38 &  1 &    38 & $3.02$ \\
\midrule
\multirow{3}{*}{Mistral-7B-512k}  &   4k & 56 &   411 & 88 &  2805 & 99 &  2938 & $7.15$ \\
                 &  32k &  7 &    55 & 11 &   413 & 12 &   452 & $8.22$ \\
                 & 128k &  1 &    13 &  2 &    71 &  3 &    86 & $6.87$ \\
\midrule
\multirow{3}{*}{Qwen2.5-7B-1M}    &   4k & 42 &   793 & 52 &  2169 & 51 &  2141 & $2.70$ \\
                 &  32k &  5 &   100 &  6 &   246 &  6 &   246 & $2.46$ \\
                 & 128k &  1 &    23 &  1 &    41 &  1 &    41 & $1.75$ \\
\bottomrule
\end{tabular}
\end{table}

\section{Scoring Functions of Prior Eviction Methods}\label{app:prior-objectives}

The eviction baselines compared in Table~\ref{tab:ruler_128k_0.05} can all be cast as surrogates of the same NP-hard problem $\min_{\J}\|\Delta(\J)\|_2$, but they differ in \emph{what} surrogate they construct: (a) \emph{heuristic retention scores} that depend only on cache-side proxies (position, key vectors, or attention magnitudes) and do not model the value-space output perturbation, and (b) \emph{output-perturbation surrogates} that explicitly target $\|\Delta(\J)\|_2$ via different relaxations of the combinatorial objective. Throughout this section, $\Q$ denotes the SnapKV-style observation window of $w$ recent queries, $p_{j}^{(t)}$ the softmax attention weight from query $\q^{(t)}\in\Q$ to token $j$, $\a^{(t)} = \V^{\top}\p^{(t)}$ the corresponding pre-eviction attention output, $\W_{O}$ the per-head output projection, and $\bar{p}_{j} = \tfrac{1}{|\Q|}\sum_{t\in\Q}p_{j}^{(t)}$. All baseline scores below are \emph{retention} scores (large $=$ keep); DropKV's score is an \emph{eviction cost} (small $=$ evict). Table~\ref{tab:prior-objectives} summarises the per-token expressions in one line per method.

\begin{table}[h]
\centering
\renewcommand{\arraystretch}{1.3}
\caption{Per-token scoring expressions of the baselines in Table~\ref{tab:ruler_128k_0.05}. ``Type'' classifies the surrogate strategy: \textbf{H:} heuristic retention (no value-space modelling), \textbf{B:} explicit upper bound on $\|\Delta(\J)\|_1$, \textbf{R:} convex (LASSO-style \cite{tibshirani1996regression}) relaxation of the combinatorial perturbation objective, \textbf{E:} exact singleton perturbation followed by per-token decoupling. ``Direction'' indicates whether the score is read as ``keep large'' (retention score) or ``evict small'' (eviction cost). Notation: $\bar{p}_{j} = \tfrac{1}{|\Q|}\sum_{t\in\Q}p_{j}^{(t)}$; $\mu(\K) = \tfrac{1}{n}\sum_{i}\k_{i}$; $s$, $w$ are the StreamingLLM sink and recent-window sizes; $\kappa$ is the SnapKV pooling kernel size; $\gamma$ is CAKE's variance weight; and $k^{l}$ is PyramidKV's per-layer budget.}
\label{tab:prior-objectives}
\resizebox{\textwidth}{!}{%
\begin{tabular}{l c l l c}
\toprule
\textbf{Method} & \textbf{Type} & \textbf{Per-token score} & \textbf{Direction} & \textbf{Reference}\\
\midrule
StreamingLLM~\cite{xiao2023efficient} & H & $\1[j\le s] + \1[j > n - w]$ & keep large (mask) & \S3.2 \\
KeyDiff~\cite{park2025keydiff} & H & $-\,\mathrm{CosSim}\!\bigl(\k_{j},\,\mu(\K)\bigr)$ & keep large & Eq.~8 \\
SnapKV~\cite{li2024snapkv} & H & $\mathrm{MaxPool1d}_{\kappa}\!\bigl(\sum_{t\in \Q} p_{j}^{(t)}\bigr)$ & keep large & Eq.~2--3 \\
PyramidKV~\cite{cai2024pyramidkv} & H & SnapKV per-token; layer-asym.\ budget $k^{l}$ & keep large & Eq.~1, 3 \\
CAKE~\cite{qin2025cake} & H & $\mathrm{Mean}(\A[-w{:},j]) + \gamma\,\mathrm{Var}(\A[-w{:},j])$ & keep large & Eq.~11 \\
CriticalKV~\cite{feng2025identify} & B & $(\bar{p}_{j} + \epsilon)\,\|(\V\W_{O})_{j,:}\|_{1}$ & keep large & Alg.~1; Thm.~3.5 \\
AnDPro~\cite{geng2025accurate} & R & $\sum_{t\in \Q} p_{j}^{(t)}\,\inner{\a^{(t)},\,\v_{j}}$ & keep large & Eq.~15 \\
\rowcolor{gray!30}
\textbf{DropKV} (ours) & E & $\sum_{t\in \Q} \bigl(p_{j}^{(t)}/(1-p_{j}^{(t)})\bigr)^{\!2}\,\|\a^{(t)} - \v_{j}\|_{2}^{2}$ & evict small & Lem.~\ref{thm:exact} \\
\bottomrule
\end{tabular}%
}
\end{table}

\paragraph{Value-aware surrogates.} CriticalKV minimizes an explicit upper bound $\widehat{\theta}$ on the $\ell_{1}$ analog of $\|\Delta(\J)\|_2$ after the per-head output projection $\W_{O}$ (\cite{feng2025identify}, Theorems~3.3 and~3.5); the listed score is the stage-2 criterion of a two-stage selection in which a fraction $\alpha$ of the cache is filled by raw $\bar{p}_{j}$ first and the remainder by the composite score, with $\alpha\!=\!0.5$ in the paper's experiments (the listed default in Algorithm~1 is $\alpha\!=\!0.25$). AnDPro starts from the value-space reconstruction problem $\min_{\z\in\{0,1\}^{n}}\tfrac{1}{2}\|\y - \tfrac{\sum_{i}z_{i}a_{i}\v_{i}}{\sum_{i}z_{i}a_{i}}\|_{2}^{2}$, relaxes the binary mask to $\beta_{i}\!\in\![0, 1/s]$ with a LASSO penalty (Problem~P1 of \cite{geng2025accurate}, \S3.2), and uses the KKT screening condition with anchor $\theta\!=\!\y^{(t)}$ to derive the per-token retention score $p_{j}^{(t)}\inner{\a^{(t)},\v_{j}}$ (Eq.~13--15). DropKV instead uses the \emph{exact} per-token eviction perturbation $\|\Delta(\{j\})\|_{2} = \tfrac{p_{j}}{1-p_{j}}\|\a - \v_{j}\|_{2}$ from Lemma~\ref{thm:exact} as the per-token cost and decouples the multi-token objective into a window-summed sum of squared tokens; the decoupling is justified by the cone condition (Theorem~\ref{thm:optimal}), which we verify empirically at deployment scale (Appendix~\ref{app:assumption-verification}).

\paragraph{SnapKV as the attention-only baseline.} Since all methods use the same observation window $w{=}8$ and pooling kernel $\kappa{=}11$ (\S\ref{sec:experiments}), the SnapKV vs.\ DropKV gap in Table~\ref{tab:ruler_128k_0.05} directly isolates the contribution of the value-residual factor $\|\a^{(t)}-\v_j\|_2$ over attention-weighted scoring: $+2.43$, $+7.68$, and $+7.09$ average RULER points on Llama-3.1-8B, Mistral-7B-512K, and Qwen2.5-7B-1M at $n{=}128\mathrm{K}$, $5\%$ budget. The two scores differ only in pool order (SnapKV pools after summing attention; DropKV pools after summing the squared score), which is negligible at the small attention weights typical of the eviction pool.

\paragraph{Pre- vs.\ post-$\W_O$ scoring.} DropKV scores tokens by the pre-projection perturbation $\|\a^{(t)}-\v_j\|_2$ rather than the residual-stream perturbation $\|\W_O^{(h)}(\a^{(t)}-\v_j)\|_2$ that the rest of the network actually receives. By submultiplicativity, $\|\W_O^{(h)}\delta\|_2 \le \|\W_O^{(h)}\|_{\mathrm{op}}\|\delta\|_2$; when the head projection is well-conditioned on the perturbation directions, the pre- and post-projection objectives differ by at most the corresponding conditioning factor, leaving the form of Theorem~\ref{thm:optimal} unchanged while changing only this constant. Empirically, the closest baseline that explicitly uses $\W_O$-projected value representations is CriticalKV~\cite{feng2025identify}, whose stage-2 score ranks tokens by attention weight times the $L_1$ norm of $\W_O$-projected values; while this comparison also changes the norm, attention weighting, and selection procedure, DropKV --- which omits the projection --- outperforms CriticalKV at $n{=}128\mathrm{K}$ and $5\%$ budget by $+3.28$, $+7.34$, and $+6.34$ average RULER points on Llama-3.1-8B, Mistral-7B-512K, and Qwen2.5-7B-1M, respectively (Table~\ref{tab:ruler_128k_0.05}), suggesting explicit output projection is not necessary at the score level.

\paragraph{DropKV vs.\ AnDPro at $\v_{j} \approx \a$.} DropKV and AnDPro target the same value-space output but their per-token surrogates disagree in sign within a structurally important regime. By Lemma~\ref{thm:exact}, removing a single token $j$ shifts the output by $\Delta(\{j\}) = \tfrac{p_{j}}{1-p_{j}}(\a - \v_{j})$, so a token with $\v_{j} \approx \a$ produces a near-zero shift and is therefore safely evictable. DropKV correctly assigns it a near-zero eviction cost. AnDPro's KKT-derived score $p_{j}\inner{\a,\v_{j}}$, in contrast, is large precisely when $\v_{j}$ aligns with $\a$, so AnDPro \emph{retains} the same token. The two surrogates approximate the same combinatorial objective $\|\Delta(\J)\|_2$ but disagree pointwise at strong $\v_{j}$-$\a$ alignment, which is the regime projection-based scoring is designed to reward. On RULER at the $5\%$ budget, DropKV improves average accuracy over AnDPro by $+2.35$, $+3.87$, and $+4.03$ points on Llama-3.1-8B-Instruct, Qwen2.5-7B-1M, and MegaBeam-Mistral-7B-512K respectively (Table~\ref{tab:ruler_128k_0.05}), consistent with the per-token perturbation surrogate being the more faithful one.

\paragraph{Linear vs.\ quadratic attention weighting.} CriticalKV and AnDPro weight scores by $p_{j}$ linearly; DropKV uses $(p_{j}/(1-p_{j}))^{2}$. At deployment-scale eviction-pool weights $p_{j} \sim 7.6\!\times\!10^{-6}$ (uniform $1/n$ at $n = 128\mathrm{K}$), the $1/(1-p_{j})$ correction contributes at most $\approx\!1.5\!\times\!10^{-5}$ relative deviation; the substantive gap between DropKV and the prior value-aware surrogates is therefore between linear and quadratic in $p_{j}$, which polarizes DropKV's eviction cost toward the few less-tiny attention weights when they are present in the candidate pool $\L$.

\section{Theoretical Analysis}\label{sec:proof}

\subsection{NP-hardness of Exact Eviction Minimization}\label{app:nphard}

Even under the small-weights regime in Assumption~\ref{ass:small-p}, exact
minimization of $F(\J)=\|\Delta(\J)\|_2$ is NP-hard. The obstacle is the
numerator $\|\sum_{j\in\J} p_j\r_j\|_2$, which is the norm of a sum couples all elements of $\J$, so the objective is not decomposable over individual
tokens.

To see this, consider the scalar case $d=1$. Given an instance of
\textbf{Balanced Partition} \cite{hartmanis1982computers}, namely $2m$ positive integers
$a_1,\dots,a_{2m}$ with total sum $A$, decide whether they can be split into
two equal-size groups of equal sum. Set $n=2m$, $k=m$, $p_j=1/n$ for all $j\in[n]$,
and $\v_j=a_j $, $\a=A/n$, so that
$\r_j=\a-\v_j=A/n-a_j$. Then, for any feasible $|\J|=k$,
\[
\sum_{j\in\J}p_j\r_j
=\frac{1}{n}\sum_{j\in\J}(A/n-a_j)
=\frac{1}{n}\left(\frac{kA}{n}-\sum_{j\in\J}a_j\right)
=\frac{1}{n}\left(\frac{A}{2}-\sum_{j\in\J}a_j\right).
\]
This term vanishes if and only if $\sum_{j\in\J}a_j=A/2$, i.e., if and
only if the balanced partition exists. Moreover, $P_\J=k/n=1/2$ is constant
over all feasible $\J$, so the denominator $1-P_\J$ does not change the
decision. Thus deciding whether $\min_{|\J|=k}F(\J)=0$ is at least as hard as
\textbf{Balanced Partition}, which is NP-complete \cite{hartmanis1982computers}. The construction also
satisfies Assumption~\ref{ass:small-p}: taking $\varepsilon=1/n$ gives
$k\varepsilon=1/2<1$.

\subsection{Main Proofs}
\begin{customlemma}{1}[Perturbation]
    Suppose a set of tokens indexed by $\mathcal{J} \subset [n]$ is evicted from the full KV cache, then the perturbation to the original attention output $\a$ associated with the query $\q$ is
    $$
    \Delta(\mathcal{J}) = \frac{\sum_{i \in \mathcal{J}} p_i (\a-\v_i)}{1 - P_{\mathcal{J}}},
    $$
    where $P_\mathcal{J} := \sum_{i \in \mathcal{J}} p_i$ denotes the total attention weights associated with the evicted tokens. In particular, for $\mathcal{J}=\{j\}$, we have the perturbation
    $\Delta(\{j\}) = \frac{p_j}{1-p_j}\,(\a-\v_j)$ for per-token eviction.
\end{customlemma}

\begin{proof}[Proof of Lemma~\ref{thm:exact}]
    Let $\a_{-\mathcal{J}}$ be the attention output computed after removing the tokens in $\mathcal{J}$ from the KV cache. Because the softmax probabilities must be re-normalized over the remaining tokens. Specifically, $\forall \, i\notin \mathcal{J}$, it is easy to show that the associated re-normalized probability is given by $p^\prime_i = \frac{p_i}{\sum_{j \notin \mathcal{J}} p_j}$. 
    
    Therefore, this updated attention output is given by
  $$\a_{-\mathcal{J}} = \sum_{i\notin \mathcal{J}} p_i^\prime \v_i = \frac{\sum_{i \notin \mathcal{J}} p_i \v_i}{\sum_{j \notin \mathcal{J}} p_j} = \frac{\a-\sum_{i\in \mathcal{J}} p_i \v_i}{1- \sum_{i\in \mathcal{J}} p_i} = \frac{\a-\sum_{i\in \mathcal{J}} p_i \v_i}{1-P_\mathcal{J}},
  $$
  where we used $\sum_{i=1}^n p_i = 1$ and $\a = \sum_{i=1}^n p_i\v_i$.
Thus,
\[
\Delta(\mathcal{J}) = \a_{-\J}-\a
= \frac{\a-\sum_{i\in \mathcal{J}} p_i v_i}{1-P_\mathcal{J}}-\a
= \frac{P_\mathcal{J} \a -\sum_{i\in \mathcal{J}} p_i \v_i}{1-P_\mathcal{J}}
= \frac{\sum_{i\in \mathcal{J}} p_i(\a-\v_i)}{1-P_\mathcal{J}}.
\]
\end{proof}

\begin{customthm}{1}[Near-Optimality]
Let $\J^*=\arg\min\{F(\J) := \|\Delta(\J)\|_2:\J\subseteq \L,\ |\J|=k\}$ be the optimal eviction set and
$\hat{\J}=\arg\mathrm{topk}_{j\in\L}^{-} \left(\|\Delta(\{j\})\|_2, k\right)$ the proxy.
Under Assumptions~\ref{ass:small-p} and \ref{ass:cone}, we have
\[
F(\hat{\J})
\le \frac{1}{c(1-\varepsilon)(1-k\varepsilon)}\,F(\J^*).
\]
\end{customthm}

\begin{proof}[Proof of Theorem~\ref{thm:optimal}]
For any $\J\subseteq \L$ with $|\J|=k$, we first define an auxiliary objective function:
$$
G(\J): = \sum_{j\in\J} \|\Delta(\{j\})\|_2 = \sum_{j\in \J}\frac{p_j}{1-p_j}\|\r_j\|_2. 
$$
Then, by the definition of $\hat{\J}:=\arg\mathrm{topk}^{-}_{j\in\L}(\|\Delta(\{j\})\|_2, k)$, it is easy to see that $G(\J)\geq G(\hat{\J})$, $\forall \, \J\subseteq\L$ with $|\J|=k$. That is, $\hat{\J}$ minimizes $G$. Moreover, using the triangle inequality, we have
\[
 G(\J) = \sum_{j\in \J}\frac{p_j}{1-p_j}\|\r_j\|_2  \geq \sum_{j\in \J} p_j\|\r_j\|_2 \geq \Big\|\sum_{j\in \J} p_j \r_j\Big\|_2.
\]
Hence, by Assumption~\ref{ass:small-p}, $P_\J \leq k\varepsilon$, we have
\begin{equation}\label{eq:ftog}
 F(\J)=\frac{\big\|\sum_{j\in \J} p_j \r_j\big\|_2}{1-P_\J}
\le \frac{G(\J)}{1-P_\J}
\le \frac{G(\J)}{1-k\varepsilon}.   
\end{equation}
On the other hand, since $\u^\top \z\leq \|\u\|_2\|\z\|_2 = \|\z\|_2$ for all $\z\in\R^d$, and by Assumption~\ref{ass:cone}, $\u^\top \r_j \;\ge\; c \|\r_j\|_2$ for all $j\in\J$, we have
\begin{equation*}
   \Big\|\sum_{j\in \J} p_j \r_j\Big\|_2
\ge \u^\top\sum_{j\in \J}p_j\r_j
\ge c\sum_{j\in \J} p_j\|\r_j\|_2. 
\end{equation*}
Then, since $p_j\le\varepsilon$, we have
\begin{equation*}
    G(\J)=\sum_{j\in \J}\frac{p_j}{1-p_j}\|\r_j\|_2 \le \frac{1}{1-\varepsilon}\sum_{j\in \J} p_j\|\r_j\|_2\leq \frac{1}{c(1-\varepsilon)}\Big\|\sum_{j\in \J} p_j \r_j\Big\|_2.
\end{equation*}
Therefore,
\begin{equation}\label{eq:gtof}
 F(\J) =\frac{\big\|\sum_{j\in \J} p_j \r_j\big\|_2}{1-P_\J} \ge\Big\|\sum_{j\in \J} p_j \r_j\Big\|_2 \ge c(1-\varepsilon) G(\J).   
\end{equation}
Note that $G(\hat{\J})\le G(\J^*)$, we have
$$
F(\J^*)\geq c(1-\varepsilon)G(\J^*)\geq c(1-\varepsilon)G(\hat{\J})\geq c(1-\varepsilon)(1-k\varepsilon)F(\hat{\J}),
$$
where we used (\ref{eq:gtof}) in the first inequality and (\ref{eq:ftog}) in the last one. This completes the proof.
\end{proof}

\paragraph{Extension to the multi-query score.}
Theorem~\ref{thm:optimal} concerns a single query, whereas DropKV sums squared perturbations over the observation window $\Q=\{\q^{(t)}\}_{t=1}^w$ in~\eqref{eq:DropKV}. For each $t\in[w]$, let $p^{(t)}_j$, $\a^{(t)}$, $\r^{(t)}_j:=\a^{(t)}-\v_j$, and $\Delta^{(t)}(\J)$ denote the quantities of Lemma~\ref{thm:exact} for query $\q^{(t)}$. Let $\L$ be a candidate pool visible to every query in $\Q$. Define
$$
F_\Sigma(\J):=\sum_{t=1}^w\big\|\Delta^{(t)}(\J)\big\|_2,
\qquad
\hat{\J}_w:=\arg\mathrm{topk}^{-}_{j\in\L}\left(\mathrm{score}_j,k\right).
$$

\begin{prop}[Multi-Query Near-Optimality]\label{prop:multiquery}
Suppose Assumptions~\ref{ass:small-p} and~\ref{ass:cone} hold for every query $\q^{(t)}$, $t\in[w]$, with common constants $\varepsilon$ and $c$ satisfying $k\varepsilon<1$. The cone directions $\u_t$ may differ across queries. Then
$$
F_\Sigma(\hat{\J}_w)\;\le\;\frac{\sqrt{w}}{c(1-\varepsilon)(1-k\varepsilon)}\,\min_{\J\subseteq\L,\ |\J|=k}F_\Sigma(\J).
$$
\end{prop}

\begin{proof}[Proof of Proposition~\ref{prop:multiquery}]
Let $g^{(t)}_j:=\|\Delta^{(t)}(\{j\})\|_2$, $\gamma_j:=(g^{(1)}_j,\dots,g^{(w)}_j)\in\R^w$, and $G_\Sigma(\J):=\sum_{j\in\J}\|\gamma_j\|_1 = \sum_{t=1}^w \sum_{j\in\J} \|\Delta^{(t)}(\{j\})\|_2$ with $|\J| = k$. Apply~\eqref{eq:ftog} and~\eqref{eq:gtof} to each query $t$, using direction $\u_t$, and sum over $t$. This gives, for every $\J\subseteq\L$ with $|\J|=k$,
\begin{equation}\label{eq:mq-sandwich}
c(1-\varepsilon)\,G_\Sigma(\J)\;\le\;F_\Sigma(\J)\;\le\;\frac{G_\Sigma(\J)}{1-k\varepsilon}.
\end{equation}
Since $\mathrm{score}_j=\|\gamma_j\|_2^2= \sum_{t=1}^w \|\Delta^{(t)}(\{j\})\|_2^2$, the set $\hat{\J}_w$ minimizes $\sum_{j\in\J}\|\gamma_j\|_2$ over all $k$-subsets of $\L$. Combining this with $\|\gamma\|_2\le\|\gamma\|_1\le\sqrt{w}\,\|\gamma\|_2$ gives
$$
G_\Sigma(\hat{\J}_w)\le\sqrt{w}\sum_{j\in\hat{\J}_w}\|\gamma_j\|_2\le\sqrt{w}\sum_{j\in\J}\|\gamma_j\|_2\le\sqrt{w}\,G_\Sigma(\J).
$$
Take $\J=\J^*_\Sigma\in\arg\min_{\J\subseteq\L,|\J|=k}F_\Sigma(\J)$ and combine with~\eqref{eq:mq-sandwich}:
$$
F_\Sigma(\hat{\J}_w)\le\frac{G_\Sigma(\hat{\J}_w)}{1-k\varepsilon}
\le\frac{\sqrt{w}\,G_\Sigma(\J^*_\Sigma)}{1-k\varepsilon}
\le\frac{\sqrt{w}}{c(1-\varepsilon)(1-k\varepsilon)}\,F_\Sigma(\J^*_\Sigma).
$$
\end{proof}

For $w=1$, Proposition~\ref{prop:multiquery} reduces to Theorem~\ref{thm:optimal}. Under GQA, scores are averaged over the $\mathfrak{g}$ query heads that share a KV head. The same argument applies to the resulting $\mathfrak{g}w$ (query head, window position) pairs and gives the factor $\sqrt{\mathfrak{g}w}$. Replacing Assumption~\ref{ass:cone} by Assumption~\ref{ass:approx_cone} for each query adds $(1+c)\sum_t\mu^{(t)}$ to $F_\Sigma(\J^*_\Sigma)$ inside the bound, as in Theorem~\ref{thm:robust_optimality}. The analysis concerns $\mathrm{score}_j$ before the max-pooling step of Algorithm~\ref{alg:dropkv}.

\subsection{Numerical Verification of Theorem Assumptions}
\label{app:assumption-verification}

This appendix reports two complementary empirical analyses of Theorem~\ref{thm:optimal}: (i) verification of the two assumptions and the resulting worst-case bound factor on real $n{=}128\mathrm{K}$ prefills, and (ii) the proxy's realized approximation ratio measured by brute-force enumeration on smaller-scale ($n{=}2\mathrm{K}$) prefills, where the full per-query attention vector fits in memory under the eager backend. Both analyses cover Llama-3.1-8B-Instruct, MegaBeam-Mistral-7B-512K, and Qwen2.5-7B-Instruct-1M; for each model, we sample (layer, head, query-position) triples and extract the attention weights $\p$ and residuals $\r_j:=\a-\v_j$. The candidate pool $\L$ is the low-attention eviction pool, excluding the final observation-window tokens.

\paragraph{Assumption~\ref{ass:small-p}: small eviction probabilities.}
For each triple, let $\varepsilon=\max_{j\in\L}p_j$. Table~\ref{tab:assumption-small-p} shows that the typical low-attention candidate pool is in the small-weight regime: the median $\varepsilon$ is $2.6$--$4.3\times10^{-6}$ and the 95th percentile is at most ${\approx}1.2\times10^{-4}$ across all three models. The maximum values are larger because a small number of early-layer/head triples contain isolated high-probability tokens; therefore we also report the finite-bound coverage below instead of claiming the union-bound denominator $1-k\varepsilon$ is positive uniformly over all triples.

\begin{table}[h]
\centering
\small
\caption{Verification of Assumption~\ref{ass:small-p} at $n=128\mathrm{K}$: distribution of $\varepsilon=\max_{j\in\L}p_j$ across sampled triples.}
\label{tab:assumption-small-p}
\begin{tabular}{lrrrrr}
\hline
Model & triples & median $\varepsilon$ & p95 $\varepsilon$ & max $\varepsilon$ & $\varepsilon\le 0.005$ \\
\hline
Llama-3.1-8B-Instruct      & $32\mathrm{K}$ & $3.1{\times}10^{-6}$ & $1.1{\times}10^{-4}$ & $0.146$ & $97.89\%$ \\
MegaBeam-Mistral-7B-512K   & $32\mathrm{K}$ & $2.6{\times}10^{-6}$ & $9.7{\times}10^{-5}$ & $0.222$ & $97.19\%$ \\
Qwen2.5-7B-Instruct-1M     & $25{,}088$ & $4.3{\times}10^{-6}$ & $1.2{\times}10^{-4}$ & $0.176$ & $97.52\%$ \\
\hline
\end{tabular}
\end{table}

\paragraph{Assumption~\ref{ass:cone}: cone condition.}
For each triple, we compute
\[
c^\star=\max_{\|\u\|=1}\min_{i\in\L}\frac{\u^\top\r_i}{\|\r_i\|_2},
\]
so Assumption~\ref{ass:cone} holds exactly when $c^\star>0$. Table~\ref{tab:assumption-cone-dist} reports the distribution of $c^\star$, and Table~\ref{tab:assumption-cone-threshold} reports the fraction of triples exceeding several positive thresholds. The strict cone condition holds on $91.21\%$, $91.12\%$, and $76.45\%$ of triples for Llama, Mistral, and Qwen2.5, respectively.

\begin{table}[h]
\centering
\small
\caption{Cone constant $c^\star$ distribution across sampled triples at $n=128\mathrm{K}$.}
\label{tab:assumption-cone-dist}
\begin{tabular}{lrrrrrr}
\hline
Model & min & p05 & \textbf{p50} & mean & p95 & max \\
\hline
Llama-3.1-8B-Instruct             & $-0.325$ & $-0.148$ & $\mathbf{+0.107}$ & $+0.120$ & $+0.389$ & $+0.882$ \\
MegaBeam-Mistral-7B-512K          & $-0.334$ & $-0.131$ & $\mathbf{+0.094}$ & $+0.110$ & $+0.368$ & $+0.849$ \\
Qwen2.5-7B-Instruct-1M            & $-0.567$ & $-0.196$ & $\mathbf{+0.074}$ & $+0.088$ & $+0.423$ & $+0.845$ \\
\hline
\end{tabular}
\end{table}

\begin{table}[h]
\centering
\small
\caption{Fraction of sampled triples satisfying $c^\star>\tau$.}
\label{tab:assumption-cone-threshold}
\begin{tabular}{lrrrrrr}
\hline
Model & $\tau=0.00$ & $0.05$ & $0.10$ & $0.20$ & $0.30$ & $0.50$ \\
\hline
Llama-3.1-8B-Instruct             & $91.21\%$ & $74.16\%$ & $53.01\%$ & $21.43\%$ & $8.93\%$ & $1.29\%$ \\
MegaBeam-Mistral-7B-512K          & $91.12\%$ & $72.37\%$ & $46.85\%$ & $17.64\%$ & $8.49\%$ & $0.99\%$ \\
Qwen2.5-7B-Instruct-1M            & $76.45\%$ & $57.29\%$ & $42.99\%$ & $22.50\%$ & $12.33\%$ & $2.14\%$ \\
\hline
\end{tabular}
\end{table}

\paragraph{Worst-case bound factor.}
Combining the measured $(c^\star,\varepsilon)$ values with the union-bound denominator $1-k\varepsilon$ from Theorem~\ref{thm:optimal} yields the conservative multiplicative factor in Table~\ref{tab:assumption-bound-factor}. The factor is finite only when $c^\star>0$ \emph{and} $1-k\varepsilon>0$; over this finite subset, the median is $10.4$, $11.8$, and $11.3$ across the three models, with heavy right tails (p95 between $115$ and $198$).

\begin{table}[h]
\centering
\small
\caption{Empirical distribution of the conservative multiplicative factor in Theorem~\ref{thm:optimal}, computed as $1/[c^\star(1-\varepsilon)(1-k\varepsilon)]$ from the measured assumption constants. ``finite'' is the fraction of sampled triples with $c^\star>0$ and $1-k\varepsilon>0$; p50/p95/max are computed over this finite subset.}
\label{tab:assumption-bound-factor}
\begin{tabular}{lrrrr}
\hline
Model & finite & p50 & p95 & max \\
\hline
Llama-3.1-8B-Instruct       & $89.13\%$ & $10.42$ & $122$ & $8.2{\times}10^{4}$ \\
\hline
MegaBeam-Mistral-7B-512K    & $87.53\%$ & $11.83$ & $115$ & $1.5{\times}10^{6}$ \\
\hline
Qwen2.5-7B-Instruct-1M      & $73.43\%$ & $11.32$ & $198$ & $2.2{\times}10^{5}$ \\
\hline
\end{tabular}
\end{table}

\paragraph{Realized approximation ratio.}
Beyond the worst-case bound characterized above, we directly measure the proxy's actual approximation ratio. For each sampled (layer, head, query-position) triple, we draw $n_{\mathrm{small}}{=}20$ keys from the candidate pool $\mathcal{L}$ and brute-force enumerate all $\binom{20}{k}$ subsets ($k\in\{10, 18\}$) to find the joint optimum $F(\mathcal{J}^*) := \min_{\mathcal{J},|\mathcal{J}|=k}\|\sum_{i\in\mathcal{J}}p_i(\a-\v_i)\|_2/(1-P_\mathcal{J})$. Four stratifications target progressively harder regimes: (i) uniform random sampling from $\mathcal{L}$; (ii) the \emph{low-attention tail}, the bottom-$n_{\mathrm{small}}$ entries by $p_j$; (iii) the \emph{near-proxy-threshold} band, entries closest to the median DropKV score within $\mathcal{L}$, where the proxy is most likely to misrank; and (iv) the \emph{rank-disagreement} band, entries where the attention-only and DropKV rankings diverge most, the regime where the value-residual factor matters most.

Across $1{,}200$ records per model on Llama-3.1-8B-Instruct, MegaBeam-Mistral-7B-512K, and Qwen2.5-7B-Instruct-1M, the DropKV proxy ratio $F(\hat{\mathcal{J}})/F(\mathcal{J}^*)$ has median $\le 1.16$ and 95th percentile $\le 1.43$ in every (model$\,\times\,$stratification$\,\times\,k$) cell, with the worst single record at $1.82$ (Table~\ref{tab:exact-optimal-gap}). Compared with the median bound factor of $10$--$12$ in Table~\ref{tab:assumption-bound-factor}, the empirical realized ratio is roughly an order of magnitude tighter; the looseness of Theorem~\ref{thm:optimal}'s constant therefore reflects worst-case mathematical conservatism rather than slack in the algorithm itself. DropKV and raw-attention selection behave nearly identically on uniform-random and low-attention-tail strata (medians within $0.02$), but diverge on the rank-disagreement stratification (DropKV median $1.03$ vs.\ attention-only median $1.09$--$1.15$; 95th-percentile $1.17$--$1.26$ vs.\ $1.49$--$1.58$ across the three models), confirming that the value-residual factor $\|\a-\v_j\|_2$ is decisive precisely where attention rank alone is ambiguous. Random selection is consistently dominated, with median $1.16$--$2.93$ at $k{=}10$.

\begin{table}[h]
\centering
\small
\caption{Empirical approximation gap: DropKV proxy ratio $F(\hat{\mathcal{J}})/F(\mathcal{J}^*)$ at $n_{\mathrm{small}}{=}20$, broken down by stratification and $k$. Median (p50) and 95th percentile (p95) computed over $150$ triples per (model, stratification, $k$) cell. Compare against the worst-case bound (median $10$--$12$) in Table~\ref{tab:assumption-bound-factor}.}
\label{tab:exact-optimal-gap}
\begin{tabular}{l|cc|cc|cc}
\hline
                       & \multicolumn{2}{c|}{Llama-3.1-8B} & \multicolumn{2}{c|}{Mistral-7B-512K} & \multicolumn{2}{c}{Qwen2.5-7B-1M} \\
Stratification         & p50 & p95 & p50 & p95 & p50 & p95 \\
\hline
\multicolumn{7}{l}{$k=10$ (50\% budget)} \\
random                 & $1.000$ & $1.095$ & $1.000$ & $1.107$ & $1.000$ & $1.107$ \\
low-attention tail     & $1.036$ & $1.164$ & $1.029$ & $1.130$ & $1.035$ & $1.136$ \\
near-threshold         & $1.140$ & $1.278$ & $1.151$ & $1.344$ & $1.160$ & $1.430$ \\
rank-disagreement      & $1.034$ & $1.169$ & $1.029$ & $1.185$ & $1.034$ & $1.262$ \\
\hline
\multicolumn{7}{l}{$k=18$ (10\% budget, matches the deployment budget in \S\ref{sec:experiments})} \\
random                 & $1.000$ & $1.099$ & $1.000$ & $1.093$ & $1.000$ & $1.081$ \\
low-attention tail     & $1.012$ & $1.066$ & $1.012$ & $1.079$ & $1.011$ & $1.073$ \\
near-threshold         & $1.032$ & $1.084$ & $1.032$ & $1.109$ & $1.038$ & $1.109$ \\
rank-disagreement      & $1.007$ & $1.104$ & $1.000$ & $1.085$ & $1.003$ & $1.094$ \\
\hline
\end{tabular}
\end{table}

The verification above is limited to Assumptions~\ref{ass:small-p} and~\ref{ass:cone}, which are the two assumptions used by Theorem~\ref{thm:optimal}. Assumption~\ref{ass:approx_cone} below is a relaxed condition for the additive robustness result; verifying it directly would require estimating the out-of-cone weighted mass parameter $\mu$ for each triple, which we leave outside this numerical check.

\subsection{Further Robustness Analysis}\label{subsec:robust_analysis}
When the candidate pool $\L$ does not strictly satisfy the cone condition as in Assumption~\ref{ass:cone}, we establish near-optimality bound up to an additive error to account for the outlier value residuals $\r_i:=\a-\v_i$. Specifically, we make the following relaxed assumption instead:
\begin{assumption}[Approximate Cone Condition]\label{ass:approx_cone}
    There exist constants $\mu\geq0$, $c\in(0,1]$, and a unit vector $\u\in\R^d$, such that
    $$
    \sum_{j\in\L\setminus\L_{\u,c}} p_j\|\r_j\|_2 \leq \mu,
    $$
    where $\L_{\u,c}:= \{l\in\L: \u^\top \r_l\ge c \|\r_l\|_2\}$ denotes the set of indices whose residuals $\r_l$ reside within a primary alignment cone centered at $\u$.
\end{assumption}

This relaxed assumption imposes that the weighted mass of outlier residuals is upper-bounded by $\mu$. Therefore, the strict cone condition, Assumption~\ref{ass:cone}, is a special case of Assumption~\ref{ass:approx_cone} with $\L_{\u, c} = \L$ and $\mu=0$; See Fig.~\ref{fig:cone_comparison} for an illustration.

\begin{figure}[ht]
\centering
\begin{subfigure}[b]{0.45\textwidth}
\centering
\begin{tikzpicture}[scale=1, >=stealth]
  \fill[blue!8] (0,0) -- (25:4) arc[start angle=25, end angle=-25, radius=4] -- cycle;
  \draw[blue!55, dashed, thick] (0,0) -- (25:4);
  \draw[blue!55, dashed, thick] (0,0) -- (-25:4);
  \draw[->, red!80!black, very thick] (0,0) -- (2.0,0) node[right, font=\small] {$\u$};
  \draw[->, thick] (0,0) -- (5:2.9) node[right, font=\footnotesize] {$\r_1$};
  \draw[->, thick] (0,0) -- (-12:3.3) node[right, font=\footnotesize] {$\r_2$};
  \draw[->, thick] (0,0) -- (23:2.3) node[right, font=\footnotesize] {$\r_3$};
  \draw[->, thick] (0,0) -- (-3:3.6) node[right, font=\footnotesize] {$\r_4$};
  \draw[->, thick] (0,0) -- (-22:2.7) node[below right, font=\footnotesize] {$\r_5$};
  \draw[<->, blue, thick] (1.0,0) arc[start angle=0, end angle=25, radius=1.0];
  \node[font=\scriptsize, blue, anchor=west] at (1.5, 0.55) {$\arccos(c)$};
  \fill[black] (0,0) circle (0.045);
  \node[below left, font=\footnotesize] at (0,0) {$O$};
  \node[red!80!black, very thick, font=\small] at (3.5,0.6) {$\L$};
\end{tikzpicture}
\end{subfigure}
\begin{subfigure}[b]{0.45\textwidth}
\centering
\begin{tikzpicture}[scale=1, >=stealth]
  \fill[blue!8] (0,0) -- (25:4) arc[start angle=25, end angle=-25, radius=4] -- cycle;
  \draw[blue!55, dashed, thick] (0,0) -- (25:4);
  \draw[blue!55, dashed, thick] (0,0) -- (-25:4);
  \draw[->, red!80!black, very thick] (0,0) -- (2.0,0) node[right, font=\small] {$\u$};
  \draw[->, thick] (0,0) -- (5:2.9) node[right, font=\footnotesize] {$\r_1$};
  \draw[->, thick] (0,0) -- (-12:3.3) node[right, font=\footnotesize] {$\r_2$};
  \draw[->, thick] (0,0) -- (23:2.3) node[right, font=\footnotesize] {$\r_3$};
  \draw[->, thick] (0,0) -- (-183:1.8) node[above, font=\footnotesize] {$\r_4$};
  \draw[->, thick] (0,0) -- (-22:2.7) node[below right, font=\footnotesize] {$\r_5$};
  \draw[<->, blue, thick] (1.0,0) arc[start angle=0, end angle=25, radius=1.0];
  \node[font=\scriptsize, blue, anchor=west] at (1.5, 0.55) {$\arccos(c)$};
  \fill[black] (0,0) circle (0.045);
  \node[below left, font=\footnotesize] at (0,0) {$O$};
  \node[red!80!black, very thick, font=\small] at (3.5,0.6) {$\L_{\u,c}$};
\end{tikzpicture}

\end{subfigure}
\caption{\textbf{Left:} Strict Cone Condition (Assumption~\ref{ass:cone}). \textbf{Right:} Approximate Cone Condition (Assumption~\ref{ass:approx_cone}): $\L_{\u,c} = \{1,2,3,5\}$ and $p_4\|\r_4\|_2\leq \mu$.}\label{fig:cone_comparison}
\end{figure}
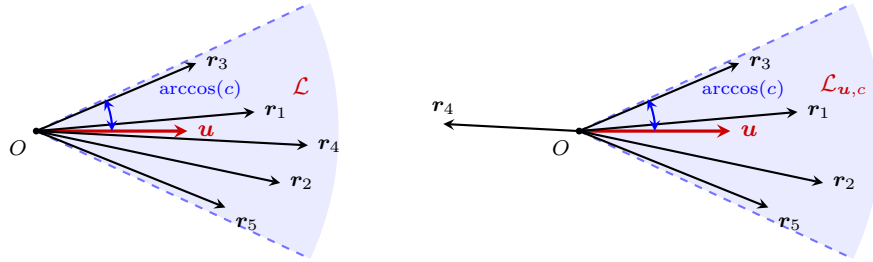

\begin{theorem}[Robust Near-Optimality]\label{thm:robust_optimality}
Let $\J^*=\arg\min\{F(\J) := \|\Delta(\J)\|_2:\J\subseteq \L,\ |\J|=k\}$ be the optimal eviction set and
$\hat{\J}=\arg\mathrm{topk}_{j\in\L}^{-} \left(\|\Delta(\{j\})\|_2, k\right)$ the proxy.
Under Assumptions~\ref{ass:small-p} and \ref{ass:approx_cone}, we have
$$
F(\widehat{\J})
\;\le\;
\frac{1}{c(1-\varepsilon)(1-k\varepsilon)}\, \Big(F(\J^\star) + (1+c)\mu\Big).
$$
\end{theorem}

\begin{proof}
For any $\J\subseteq\L$ with $|\J|=k$, following the proof of Theorem~\ref{thm:optimal}, we define 
$$
G(\J): = \sum_{j\in\J} \|\Delta(\{j\})\|_2 = \sum_{j\in \J}\frac{p_j}{1-p_j}\|\r_j\|_2.
$$
Then, $G(\J)\geq G(\hat{\J})$, $\forall \, \J\subseteq\L$ with $|\J|=k$. Moreover, we have $F(\J)\;\le\;\frac{G(\J)}{1-k\varepsilon}.$

On the other hand, we have
\begin{align*}
\|\sum_{j\in\J} p_j \r_j\|_2 \geq  \u^\top \sum_{j\in\J} p_j \r_j = & \; \u^\top \sum_{j\in\J\cap \L_{\u,c}} p_j \r_j + \u^\top \sum_{j\in\J\setminus \L_{\u,c}} p_j \r_j \\
\ge & \; c\sum_{j\in\J\cap \L_{\u,c}} p_j\|\r_j\|_2 -  \sum_{j\in\J\setminus \L_{\u,c}} p_j  \|\r_j \|_2,
\end{align*}
using $\u^\top \r_j\ge c\|\r_j\|_2$ on $\L_{\u,c}$ and $\|\z\|_2\ge\u^\top \z\ge -\|\z\|_2$ for all $\z\in\R^d$.

Rearranging and adding the term $c\sum_{j\in\J\setminus \L_{\u,c}} p_j  \|\r_j \|_2$ on both sides gives
$$
c\sum_{j\in\J} p_j\|\r_j\|_2
\le \Big\|\sum_{j\in\J} p_j \r_j\Big\|_2 + (1+c)\sum_{j\in\J\setminus \L_{\u,c}} p_j\|\r_j\|_2.
$$
Since $p_j\le\varepsilon$ implies $G(\J)\le \frac{1}{1-\varepsilon}\sum_{j\in\J}p_j\|\r_j\|_2$, we obtain
\begin{align}\label{eq1}
    G(\J) & \le \frac{1}{c(1-\varepsilon)}\Big(\Big\|\sum_{j\in\J} p_j \r_j\Big\|_2 + (1+c)\sum_{j\in\J\setminus \L_{\u,c}} p_j\|\r_j\|_2\Big) \notag \\
    & \le \frac{1}{c(1-\varepsilon)}\Big(F(\J) + (1+c)\sum_{j\in\J\setminus \L_{\u,c}} p_j\|\r_j\|_2\Big) \notag \\
    & \leq \frac{1}{c(1-\varepsilon)}\Big(F(\J) + (1+c) \mu\Big)
\end{align}
where we used $F(\J)=\|\sum_{j\in\J}p_j\r_j\|_2/(1-P_\J)\ge \|\sum_{j\in\J}p_j\r_j\|_2$ and Assumption~\ref{ass:approx_cone}:
$$
\sum_{j\in\J\setminus \L_{\u,c}} p_j\|\r_j\|_2 \leq \sum_{j\in\L\setminus \L_{\u,c}} p_j\|\r_j\|_2\leq \mu.
$$
Applying \eqref{eq1} to $\J=\J^\star$, we have
$$
G(\J^*)\leq \frac{1}{c(1-\varepsilon)}\Big(F(\J^*) + (1+c)\mu\Big)
$$
Finally, the facts $G(\widehat{\J})\le G(\J^\star)$ and $F(\widehat{\J})\le \frac{G(\widehat{\J})}{1-k\varepsilon}$ yield the claim.
\end{proof}

\section{Triton Implementation Details}\label{sec:triton-pseudocode}

This section describes the Triton implementation sketched in
\S\ref{subsec:triton}. For clarity, the notation is written for one
attention head; the actual kernels are launched independently over the
batch-head axis. The notation is head-local; model-specific head layout
is handled before this scoring interface. Queries have shape
$\boldsymbol{Q}\in\R^{B\times H\times W\times D}$, keys and values
have shape $\K,\V\in\R^{B\times H\times L\times D}$, and the output
score tensor has shape $\mathrm{score}\in\R^{B\times H\times L}$. Let
$s:=1/\sqrt{D}$, let $\BLOCKN$ denote the key/value tile size, and let
$\BLOCKK$ denote the inner dimension tile used for dot products.

\paragraph{Auxiliary forward and split-K merge.}
Kernel~A is a standard FlashAttention-style forward pass over the $W$
window queries. For each query $\q^{(t)}$, it computes the attention
output $\a^{(t)}$ and retains the row-wise log-sum-exp
\[
L^{(t)}=\log\sum_{j=1}^{L}\exp\!\left(s\,\k_j^\top \q^{(t)}\right),
\]
which allows Kernel~B to recover $p_j^{(t)}$ without materializing the
attention matrix. When split-K is enabled, Kernel~A writes partial
online-softmax states $(m_r,\ell_r,\acc_{y,r})$ for split $r$, and the
standard merge Kernel~C computes
\[
m^\star=\max_r m_r,\qquad
\ell^\star=\sum_r e^{m_r-m^\star}\ell_r,\qquad
\acc_y^\star=\sum_r e^{m_r-m^\star}\acc_{y,r},
\]
then returns $\a=\acc_y^\star/\ell^\star$ and
$L=m^\star+\log\ell^\star$. These two auxiliary kernels are unchanged
from the usual online-softmax computation except that $L^{(t)}$ is
kept for the scoring pass.

\begin{algorithm}[t]
\caption{Kernel B: DropKV score accumulation}\label{alg:triton-score}
\begin{algorithmic}[1]
\REQUIRE $\q^{(1:W)}\in\R^{W\times D}$, $\K,\V\in\R^{L\times D}$, outputs $\a^{(1:W)}$, Log-Sum-Exp $L^{(1:W)}$
\REQUIRE key tile $\mathcal{T}=\{j_0,\dots,j_0+\BLOCKN-1\}$, scale $s$
\ENSURE $\mathrm{score}_j$ for $j\in\mathcal{T}$
\STATE Load $\k_j,\v_j$ for $j\in\mathcal{T}$; precompute $\|\v_j\|_2^2$
\STATE $\mathrm{score}\gets 0^{\BLOCKN}$
\FOR{$t=1,\dots,W$}
  \STATE $\logits_j\gets s\,\k_j^\top\q^{(t)}$ for all $j\in\mathcal{T}$; apply causal and padding masks
  \STATE $p_j\gets\exp(\logits_j-L^{(t)})$
  \STATE $\tilde p_j\gets\mathrm{cast}_{\texttt{BF16}}(p_j)$ \hfill\COMMENT{matches the \texttt{BF16} reference path}
  \STATE $\omega_j\gets\big(\tilde p_j/(1-\tilde p_j+\epsilon)\big)^2$
  \STATE $d_j^2\gets\|\a^{(t)}\|_2^2+\|\v_j\|_2^2-2\langle\a^{(t)},\v_j\rangle$
  \STATE $\mathrm{score}_j\gets\mathrm{score}_j+\omega_j d_j^2$ for all $j\in\mathcal{T}$
\ENDFOR
\STATE Write $\mathrm{score}_j$ for $j\in\mathcal{T}$
\end{algorithmic}
\end{algorithm}

\paragraph{Launch and numerical guards.}
Kernel~B is launched over $(B\cdot H,\lceil L/\BLOCKN\rceil)$, so each
program owns one contiguous key tile for one batch-head pair. Causal and
padding masks are applied to logits before either the log-sum-exp or the
score is formed; padded query rows may be zero-filled only after being
excluded by the mask. The $\epsilon$ in $(1-p+\epsilon)$ avoids division
near $p\to 1$. The \texttt{BF16} cast of $p$ before squaring is used to match the
\texttt{BF16} PyTorch reference path; without it, precision loss in $(1-p)$ can
saturate the denominator differently and flip top-$k$ decisions on
peak-attention tokens. We verify the Triton and PyTorch paths with a
precision-matching test suite.

\subsection{GPU Profiling of the Triton Scoring Kernels}
\label{sec:gpu-profiling}

We further profile the scoring path with Nsight Compute (\texttt{ncu},
\texttt{--set full}) on an A100-80GB GPU using the Llama-3.1-8B shape
($B{=}1$, $H_q{=}32$, $H_{kv}{=}8$, $D{=}128$, $W{=}8$) in \texttt{BF16}.
The goal of this profiling is to explain the implementation advantage
behind Fig.~\ref{fig:scoring}: DropKV's Triton scorer is not merely a
faster implementation of the same dense computation, but removes the
dominant materialized attention and repeated-KV traffic from the scoring
path.

\paragraph{Full-call advantage over PyTorch.}
Table~\ref{tab:app-ncu-aggregate} compares the matched PyTorch scoring
reference with the Triton scoring pipeline at $n{=}16\mathrm{K}$. Triton is
$15.7{\times}$ faster in this isolated \texttt{ncu} run, launches fewer
kernels, executes $41{\times}$ fewer instructions, and moves $26{\times}$
less HBM traffic. The PyTorch reference has higher aggregate SM throughput
and active-warps only because it creates far more work by materializing
dense $W{\times}L$ attention tensors and repeated K/V tensors. In contrast,
the Triton path streams K/V tiles and accumulates scores directly, which is
the desired behavior for long-context KV scoring.

\begin{table}[t]
\centering
\caption{{Nsight Compute full-call profiling at $n{=}16\mathrm{K}$.}
Both columns aggregate all launches inside the scoring call: time/bytes/
instructions are summed, percentage metrics are time-weighted, and register
or shared-memory footprints are maxima over launches.}
\label{tab:app-ncu-aggregate}
\small
\begin{tabular}{lrr}
\toprule
Metric & PyTorch scorer & Triton scorer \\
\midrule
Runtime (ms) & 3.63 & \textbf{0.23} \\
Speedup & 1.0$\times$ & \textbf{15.7$\times$} \\
Kernels launched & 34 & \textbf{8} \\
\midrule
HBM read (MB) & 2{,}175 & \textbf{138} \\
HBM write (MB) & 1{,}545 & \textbf{5} \\
Mem\% of peak HBM & 50.3 & 30.4 \\
\midrule
SM throughput (\%) & 48.4 & 18.1 \\
Active warps (\%) & 77.3 & 12.8 \\
HMMA pipe (\%) & 4.6 & \textbf{10.1} \\
FMA pipe (\%) & 27.4 & 15.3 \\
Instructions (M) & 851 & \textbf{20.5} \\
\midrule
Regs/thread (max) & 150 & 154 \\
SMEM/block, KB (max) & 99.3 & 37.9 \\
\bottomrule
\end{tabular}
\end{table}

\paragraph{Per-kernel behavior.}
The Triton scoring call consists of three named kernels: Kernel~A computes
the auxiliary FlashAttention-style forward/log-sum-exp, Kernel~B accumulates
DropKV scores, and Kernel~C merges split-K partial states. The named
A/B/C launches in Table~\ref{tab:app-ncu-kernels} sum to $212.0\,\mu$s,
while the aggregate Triton row in Table~\ref{tab:app-ncu-aggregate} is
$231.5\,\mu$s over all eight launches; the remaining auxiliary launches
therefore add only $19.5\,\mu$s. This confirms that the two \texttt{ncu}
views are internally consistent.

At long context, Kernels~A and~B are memory-throughput limited: at
$n{=}16\mathrm{K}$, Kernel~A reaches 49.0\% of peak HBM bandwidth versus
22.7\% SM throughput, while Kernel~B reaches 54.0\% peak HBM bandwidth
versus 39.6\% SM throughput. Kernel~B alone therefore sustains roughly
1.10\,TB/s on A100 HBM, despite the low aggregate active-warps in
Table~\ref{tab:app-ncu-aggregate}. The low active-warps are not evidence of
wasted dense computation; they reflect resource ceilings from registers and
shared memory, while the useful long-context kernels are already limited
primarily by streaming K/V traffic.

\begin{table}[t]
\centering
\caption{{Per-kernel \texttt{ncu} measurements for the Triton scoring
pipeline at $n{=}16\mathrm{K}$.} SM\% and Mem\% are fractions of peak sustained
SM issue and peak HBM bandwidth. Active warps is measured achieved occupancy;
Ceiling is the resource-limited occupancy ceiling.}
\label{tab:app-ncu-kernels}
\small
\resizebox{\linewidth}{!}{%
\begin{tabular}{lrrrrrr}
\toprule
Kernel & Time ($\mu$s) & SM\% & Mem\% & Active warps & Ceiling & Main limiter \\
\midrule
A: forward/LSE & 69.82 & 22.7 & \textbf{49.0} & 11.6\% & 18.75\% & Registers \\
B: score accumulation & 63.52 & 39.6 & \textbf{54.0} & 22.6\% & 25.00\% & SMEM \\
C: split-K merge & 78.66 & 0.5 & 2.3 & 6.3\% & 43.75\% & Reduction overhead \\
\bottomrule
\end{tabular}
}
\end{table}

\paragraph{Why this supports the three-kernel design.}
The profiling identifies the remaining bottleneck as structural rather than
an avoidable PyTorch-style materialization overhead. Kernel~A uses 154
registers per thread and is capped at 12 resident warps/SM; Kernel~B uses
37.9\,KB shared memory per block and is capped at 16 resident warps/SM.
A persistent fused A+B kernel would collapse Kernel~B's parallel grid from
$(B\cdot H_{kv},\lceil L/\BLOCKN\rceil)$ to $(B\cdot H_{kv},)$ after the
forward phase, leaving only eight CTAs for Llama-3.1-8B at $B{=}1$ and
substantially reducing memory-level parallelism. Thus the current split
keeps enough tile-level parallelism for the memory-bound scoring kernel,
while still avoiding the PyTorch reference's dense attention matrix and
repeated-KV tensors. This is the central implementation advantage of DropKV:
it turns an otherwise materialization-heavy scoring formula into a compact,
streaming, bandwidth-efficient Triton pipeline.

\section{Details of RULER and LongBench Benchmarks}
\label{sec:ruler_longbench_info}

{\bf RULER: }
Below is a brief overview of the different tasks in RULER.

\begin{itemize}
	\itemsep0em
	
	\item \textbf{Single NIAH (S-NIAH): } A basic information retrieval task. In this scenario, a keyword sentence, referred to as the ``needle'', is embedded within a lengthy text, called the ``haystack''. The goal is to retrieve the ``needle'' from the context. The ``haystack'' may consist of repetitive, noisy sentences or essays by Paul Graham~\cite{Needleinahaystack}.

	\item \textbf{Multi-keys NIAH (MK-NIAH): } Multiple ``needles'' are inserted into the ``haystack'', but the task requires retrieving only one of them.

	\item \textbf{Multi-values NIAH (MV-NIAH): } Multiple ``needles'' in the form of key-value pairs are placed within the ``haystack''. The task is to retrieve all values associated with a given key.

	\item \textbf{Multi-queries NIAH (MQ-NIAH): } Multiple ``needle'' in the form of key-value pairs are inserted into the ``haystack''. All ``needles'' corresponding to the keys specified in the queries need to be retrieved.

	\item \textbf{Variable Tracking (VT): } A series of variable binding statements (e.g., $X2 = X1, X3 = X2, ...$) are inserted at various positions within a long text to simulate the recognition of entities and relationships. The task objective is to return all variable names that point to the same value $V$.
	
	\item \textbf{Common words extraction task (CWE): } The input sequence is constructed by sampling words from a discrete uniform distribution within a pre-defined (synthetic) word list. Words that appear most frequently are termed ``common words''. The model is required to identify all common words. The number of common words remains constant while the number of uncommon words increases with the length of the sequence.

	\item \textbf{Frequent words extraction task (FWE): } The input sequence is constructed by sampling words from a Zeta distribution within a pre-defined (synthetic) word list. The model need to return the top-K frequent words in the context.
	
	\item \textbf{Single-Hop QA (S-QA)} Answer questions based on passages. The context is extended to simulate long-context input by adding distracting information(golden paragraphs). And answers come from a single piece of evidence.
	
	\item \textbf{Multi-Hop QA (M-QA)} Answer questions based on passages. For Multi-hop QA, which requires integrating information from multiple sources.
	
\end{itemize}

{\bf LongBench}
spans multiple task domains with an average length of 6711, including single-document QA~\cite{dasigi2021dataset,kovcisky2018narrativeqa}, multi-document QA~\cite{ho-etal-2020-constructing,trivedi2022musique,yang2018hotpotqa}, summarization~\cite{fabbri2019multi,huang2021efficient,zhong2021qmsum}, few-shot learning~\cite{gliwa2019samsum,joshi2017triviaqalargescaledistantly,li2002learning}, synthetic tasks~\cite{bai2024longbench}, and code generation~\cite{guo2023longcoderlongrangepretrainedlanguage,liu2023repobenchbenchmarkingrepositorylevelcode}.

\begin{itemize}
	\item \textbf{Single-Doc QA:} Single-document QA requires models to obtain the answer from a single piece of source. The test samples are derived from diverse datasets including NarrativeQA ~\citep{kovcisky2018narrativeqa}, qasper ~\citep{dasigi2021dataset}, and MultiFieldQA~\citep{liu2023lost}, encompassing a range of documents such as legal files, governmental reports, encyclopedia, and academic papers.

	\item \textbf{Multi-Doc QA:} Multi-document QA requires models to extract and combine information from several documents to obtain the answer. The test samples are built from three Wikipedia-based multi-hop QA datasets: HotpotQA ~\citep{yang2018hotpotqa}, 2WikiMultihopQA ~\citep{ho2020constructing} and MuSiQue ~\citep{trivedi2022musique}. 
	
	\item \textbf{Summarization:} Summarization task requires a comprehensive understanding of the context. The test samples are drawn from the GovReport dataset ~\citep{huang2021efficient}, and QMSum ~\citep{zhong2021qmsum}. Additionally, the MultiNews dataset originates from a multi-document summarization corpus detailed in ~\citep{fabbri2019multi}
	
	\item \textbf{Few-shot Learning:} Few-shot Learning task have integrated classification, summarization, and reading comprehension tasks to maintain task diversity. For classification, the TREC dataset \citep{li2002learning} is included. The SAMSum dataset \citep{gliwa2019samsum}, which comprises messenger-style conversations with human-annotated summaries, has been incorporated for summarization tasks. Additionally, TriviaQA~\citep{joshi2017triviaqalargescaledistantly}, a dataset of question-answer pairs, is utilized for reading comprehension tasks.
		
	\item \textbf{Synthetic Task:} Synthetic Tasks are designed to assess the model's ability in handling specific scenarios and patterns. Two synthetic tasks, PassageRetrieval-en and PassageCount, are included in LongBench. PassageRetrieval-en is derived from English Wikipedia. 30 passages is randomly selected and use GPT-3.5-Turbo to summarize one of them. The task challenges the model to identify the original paragraph that matches the generated summary. PassageCount is designed to be more challenging, this task requires the model to leverage the entire context to solve it. Several passages are randomly pickedfrom English Wikipedia, randomly duplicate each paragraph multiple times, and then shuffle the paragraphs. The task is to determine the number of unique passages within the provided set.
	
	\item \textbf{Code Completion:} Code Completion task is to assist users by completing code based on previous code input and context. The LCC dataset, derived from the original Long Code Completion dataset ~\citep{guo2023longcoderlongrangepretrainedlanguage}. The RepoBench-P dataset ~\citep{liu2023repobenchbenchmarkingrepositorylevelcode} is designed to aggregating information from code across files. Model is required to predict the next line of code.

\end{itemize}

\section{Broader Impact}
\label{sec:impact}
By significantly reducing the memory footprint and HBM traffic during inference, DropKV reduces the total energy consumption required to serve long-context models. This contributes to the ``Green AI'' initiative by lowering the carbon footprint of large-scale AI deployments. While efficiency is a technical merit, making long-context LLMs cheaper and faster to run also lowers the barrier for malicious actors to generate high-volume, contextually dense disinformation or to conduct automated surveillance on large corpora of private communications.

\end{document}